\documentclass{article}
\PassOptionsToPackage{numbers,sort}{natbib}

\usepackage[main,final]{neurips_2026}

\usepackage[utf8]{inputenc}
\usepackage[T1]{fontenc}

\usepackage{microtype}
\usepackage{graphicx}
\usepackage{subcaption}
\usepackage{booktabs}
\usepackage{hyperref}
\usepackage{url}
\usepackage{amsfonts}
\usepackage{nicefrac}
\usepackage{xcolor}
\usepackage{amsmath}
\usepackage{amssymb}
\usepackage{mathtools}
\usepackage{amsthm}
\usepackage{float}
\usepackage{caption}
\usepackage{enumitem}
\setlist[itemize]{itemsep=0.5pt, topsep=1pt, parsep=0pt, partopsep=0pt}
\setlist[enumerate]{itemsep=1pt, topsep=1pt, parsep=0pt, partopsep=0pt}

\newtheorem{result}{Result}
\newtheorem{theorem}[result]{Theorem}
\newtheorem{proposition}[result]{Proposition}
\newtheorem{lemma}[result]{Lemma}
\newtheorem{corollary}[result]{Corollary}
\newtheorem{conjecture}{Conjecture}
\theoremstyle{definition}
\newtheorem{definition}{Definition}[section]

\theoremstyle{remark}
\newtheorem{remark}{Remark}[section]

\newcommand{\R}{\mathbb{R}}
\newcommand{\GL}{\mathrm{GL}}
\newcommand{\Id}{\mathrm{Id}}
\newcommand{\Aff}{\mathrm{Aff}}
\newcommand{\Gauss}{\mathcal{N}}
\newcommand{\NormalDensity}{\varphi}
\newcommand{\NormalDensityParam}[1]{\NormalDensity_{#1}}
\newcommand{\pushfwd}[2]{{#1}_{\sharp}\,#2}
\newcommand{\RVobs}{X}
\newcommand{\deq}{\sim}

\newcommand{\ambientdim}{d}
\newcommand{\mapgeneric}{h}
\newcommand{\branchgeneric}{P}
\newcommand{\indexgeneric}{I}
\newcommand{\lingeneric}{A}
\newcommand{\biasgeneric}{b}
\newcommand{\GMMgeneric}{Y}
\newcommand{\numcompgeneric}{K}
\newcommand{\wgeneric}{w}
\newcommand{\mugeneric}{\mu}
\newcommand{\Sigmageneric}{\Sigma}
\newcommand{\paramImg}[2]{\theta_{#1,#2}}
\newcommand{\mapf}{f}
\newcommand{\If}{I_{f}}
\newcommand{\branchf}[1]{P_{#1}}
\newcommand{\linf}[1]{A_{#1}}
\newcommand{\biasf}[1]{b_{#1}}
\newcommand{\latentZ}{Z}
\newcommand{\numcomptrue}{K}
\newcommand{\mutrue}[1]{\mu_{#1}}
\newcommand{\Sigmatrue}[1]{\Sigma_{#1}}
\newcommand{\wtrue}[1]{w_{#1}}
\newcommand{\paramf}[2]{\theta^{(\mapf)}_{#1,#2}}
\newcommand{\mapg}{g}
\newcommand{\Ig}{I_{g}}
\newcommand{\branchg}[1]{Q_{#1}}
\newcommand{\ling}[1]{B_{#1}}
\newcommand{\biasg}[1]{c_{#1}}
\newcommand{\latentZest}{Z'}
\newcommand{\numcompest}{K'}
\newcommand{\muest}[1]{\mu'_{#1}}
\newcommand{\Sigmaest}[1]{\Sigma'_{#1}}
\newcommand{\west}[1]{w'_{#1}}
\newcommand{\paramEst}[2]{\theta^{(\mapg)}_{#1,#2}}
\newcommand{\maph}{h}

\newcommand{\paramh}[2]{\theta^{(\maph)}_{#1,#2}}
\newcommand{\fineC}{C}
\newcommand{\fineCp}{C'}
\newcommand{\Ichamber}[2]{\mathcal{I}^{#1}(#2)}

\newcommand{\SymComp}[1]{\mathrm{Sym}_{#1}}
\newcommand{\orbit}[3]{[#1_{#2}]_{#3}}
\newcommand{\Transporter}[2]{\mathrm{Aff}({#1}{\to}{#2})}
\newcommand{\Gmix}[1]{G_{\mathrm{mix}}(#1)}
\newcommand{\IndexTrue}{[\numcomptrue]}
\newcommand{\IndexEst}{[\numcompest]}
\newcommand{\Df}{D_{\mapf}}
\newcommand{\Dg}{D_{\mapg}}
\newcommand{\Ef}{E_{\mapf}}
\newcommand{\Eg}{E_{\mapg}}
\newcommand{\preimage}[2]{#1^{-1}(#2)}

\newcommand{\condlabel}[1]{\label{cond:#1}}
\newcommand{\condref}[1]{\hyperref[cond:#1]{(#1)}}

\newcommand{\ADname}{Atom distinctness}
\newcommand{\MSTname}{Mixture symmetry triviality}
\newcommand{\SBname}{Simple boundary}
\newcommand{\USBname}{Universal simple boundary}
\newcommand{\DDname}{Distinct diagonals}
\newcommand{\ODname}{Cross-component orbit disjointness}
\newcommand{\DWname}{Distinct weights}
\newcommand{\JSTname}{Joint symmetry triviality}
\newcommand{\TODname}{Transporter-orbit disjointness}
\newcommand{\CrossCompParamInj}{\textsc{CrossCompAD}}
\newcommand{\SameCompParamInj}{\textsc{SameCompAD}}

\usepackage{rotating}
\newcommand{\YES}{\ensuremath{\checkmark}}
\newcommand{\NO}{\ensuremath{\times}}
\newcommand{\PARTIAL}{\ensuremath{\sim}}
\newcommand{\AD}{\condref{AD}}
\newcommand{\MST}{\condref{MST}}
\newcommand{\SB}{\condref{SB}}
\newcommand{\USB}{\condref{USB}}
\newcommand{\DD}{\condref{DD}}

\newif\ifmarkedits \markeditsfalse
\newif\ifmarkold \markoldfalse
\definecolor{editcolor}{rgb}{0.0,0.30,0.75}
\newcommand{\NEW}[1]{\ifmarkedits\ifmarkold{\color{editcolor}#1}\else#1\fi\else#1\fi}
\newenvironment{newedit}{\ifmarkedits\ifmarkold\color{editcolor}\fi\fi}{}
\newcommand{\edtag}[1]{\ifmarkedits\ifmarkold\textsuperscript{\color{red!70!black}\tiny\textsf{[#1]}}\fi\fi}
\newcommand{\NEWB}[1]{\ifmarkedits\ifmarkold{\color{editcolor}#1}\else#1\fi\else#1\fi}

\newcommand{\btag}[1]{\ifmarkedits\ifmarkold\textsuperscript{\color{red!70!black}\tiny\textsf{[#1]}}\fi\fi}
\newcommand{\NEWC}[1]{\ifmarkedits\ifmarkold{\leavevmode\color{editcolor}#1}\else#1\fi\else#1\fi}

\newcommand{\ctag}[1]{\ifmarkedits\ifmarkold\rlap{\smash{\textsuperscript{\color{red!70!black}\tiny\textsf{[#1]}}}}\fi\fi}
\newcommand{\NEWD}[1]{\ifmarkedits{\leavevmode\color{editcolor}#1}\else#1\fi}
\newcommand{\dtag}[1]{\ifmarkedits\rlap{\smash{\textsuperscript{\color{red!70!black}\tiny\textsf{[#1]}}}}\fi}

\begin{document}

\title{Beyond ICA: Identifiability by Symmetry Breaking}
\author{
  Pengzhou Wu\\
  \NEWB{SANKEN, The University of Osaka / RIKEN AIP}\btag{B5-1}\\
  \texttt{pengzhou.sk@osaka-u.ac.jp}
}

\maketitle
\vspace{-10pt}
\begin{abstract}
We prove the \NEW{population-level}\edtag{E2} identifiability of deep generative models (DGMs)
with piecewise-affine (PWA) decoders and Gaussian
mixture model (GMM) priors, in a purely unsupervised
setting. We introduce three algebraic contrast
principles for symmetry breaking: \emph{domain contrast}, which
trivializes the mixture symmetry group\NEWC{, \NEWD{a label-free counterpart of the sufficient-variability assumption of nonlinear ICA}\dtag{B12-1}}\ctag{B10-1}; \emph{mechanism
contrast}, which ensures every decoder branch
\NEWB{gets an observational boundary across which it toggles alone}\btag{B5-2}; and \emph{interaction
contrast}, which forbids parameter conspiracies between
latent components and decoder branches. Together they exploit the interplay
between the discrete combinatorics of the PWA map and the
continuous symmetry structure of the latent GMM.
Continuity is replaced by algebraic symmetry conditions; injectivity is decoupled from structural identification and required only for pointwise inversion.
Our results form a hierarchy: from \emph{law
identifiability} (LID; latent distribution up to a global affine
map) through \emph{map identifiability} (MID; decoder
up to the same map) to \emph{posterior} and
\emph{pointwise identifiability}. 
The ICA-form ambiguity emerges under conditions on diagonal component covariances.
\NEW{The interaction contrast is the only one of our contrasts required on the estimator side.}\edtag{E45}

To our knowledge this is the first to develop algebraic symmetry-breaking into an identifiability engine for a concrete nonlinear model class, the first to admit discontinuous decoders, and the first to handle fully non-injective decoders, where every observation admits multiple latent codes.

\end{abstract}

\vspace{-7pt}
\section{Introduction}
\label{sec:intro}

Identifiability is the foundation of trustworthy DGM.
\NEW{It is a population notion: the observable law (probability distribution)}\edtag{E3}\footnote{A mapping between our terminology (\emph{law}, \emph{map}, \emph{domain}, \emph{contrast}, etc.) and formal objects is given in Table~\ref{tab:terminology} (App.~\ref{app:terminology}).} \NEW{is taken as known exactly.}\edtag{E46}
A deep latent variable model $\RVobs = \mapf(\latentZ)$ is
\emph{identifiable} if the observable law of $\RVobs$ pins down both
the \NEW{map (decoder/generator)} $\mapf$ and the latent law of $\latentZ$, up to a
well-defined ambiguity class. Without it, two trainings on the same
data can converge to entirely different mechanisms, and downstream
tasks lose any guarantee of being well-posed. Identifiability is the
precondition for robust, interpretable, and disentangled
representations \citep{locatello2019challenging}, with deep classical
roots in finite mixtures
\citep{teicher1963identifiability,yakowitz1968identifiability,allman2009identifiability},
tensor decompositions \citep{kruskal1977three}, and linear ICA
\citep{comon1994independent,darmois1953analyse,skitovich1953property}.
The nonlinear, high-capacity setting of modern generators makes the
ambiguity class far larger than these classical settings could admit
\citep{hyvarinen1999nonlinear}.

\NEW{Domain contrast in nonlinear ICA}\edtag{E5} breaks ambiguity---but only under \emph{observed} domain labels.
The dominant remedy treats data as drawn from several
\emph{domains}---variations in the problem setting (gender,
temperature, time, location)---each shaping the latent law inside a
parametric family, and lets the diversity of domain parameters break
the symmetries that otherwise plague the model: in time-contrastive learning (TCL)
\citep{hyvarinen2016unsupervised} domains are time segments; in identifiable variational autoencoder (iVAE)
\citep{khemakhem2020variational}, observed auxiliaries. In both, the
domain label is observed. Real data often arrives without labels, so
we model the latent law as a GMM---each component
being one domain. \NEWC{Our domain contrast distils the algebraic essence of this diversity and \NEWD{weakens}\dtag{B12-2} it to unlabelled domains.}\ctag{B10-2} Going from observed to unobserved domains is
qualitatively harder: domain contrast alone no longer pins down
spurious solutions. \NEW{Unsupervised identifiability is nonetheless available under continuity and smoothness, via structural sparsity \citep{zheng2022identifiability} or restricted function classes \citep{buchholz2022function}. These are assumptions we do not require.}

Mechanism contrast turns \NEWB{observational}\btag{B5-3} boundaries into a resource.
We add a complementary symmetry-breaking lever from the generator
itself. ReLU networks are
continuous PWA maps \citep{montufar2014number} and achieve almost-optimal convergence rates on discontinuous target
functions \citep{imaizumi2019deep}. A PWA generator
partitions the latent space into branches with distinct affine
rules---discrete \emph{mechanisms}---and \emph{each domain is acted on
by every mechanism whose branch domain it intersects}. The algebraic
resource is the boundary structure: when branches do not all
share their \NEWB{\emph{observational} boundaries}, each toggle localises distinct constraints
that no smooth global generator could ever produce. Discontinuity of
$\mapf$ is fully allowed by our theory but is not what does the work---a continuous
PWA already exhibits the boundary structure we exploit.

Algebraic symmetry collapse replaces continuity and smoothness.
Existing nonlinear-ICA proofs are differential at heart---log-density
and Jacobian comparisons \citep{hyvarinen2016unsupervised,
khemakhem2020variational,gresele2021independent,gresele2019incomplete}---silent
on discontinuous maps and boundary structure\btag{B5-4}, and structurally
requiring smoothness. We replace differential techniques with an
\emph{algebraic} engine: collisions between two candidate models are
read as nontrivial elements of an affine symmetry group of the latent
law, and assumptions that trivialise this group force the two models
to coincide. Three contrast principles---one each for map, law, and
their interaction---power the engine
(Table~\ref{tab:contrasts}; formal in \S\ref{sec:setup}). \NEW{The combination keeps the merits of both worlds}\edtag{E6}: the affine link is built from
map-side boundary information, so domain labels can remain latent\NEW{; it trades one assumption profile for another rather than dominating}.
Symmetry arguments have emerged in identifiability since
the 1950s \citep{hurwicz1950generalization}; classical
models had small, obvious symmetry groups \citep{anderson1956statistical}, so a general algebraic
engine is needed only in the nonlinear, high-capacity regime.

Generative modelling is not signal processing.
Classical ICA inherits its goals from signal processing: invert the
source pointwise. This requires injectivity, a constraint
both restrictive in theory and almost never enforced in
mainstream nonlinear-ICA learning procedures
\citep{hyvarinen2016unsupervised,khemakhem2020variational,
gresele2019incomplete}, even when their theorems demand it. Deep
generative modelling cares about a different object: the
\emph{mechanism} (the decoder) and the \emph{posterior} (the
representation). Neither needs pointwise inversion. Our hierarchy
(law $\to$ map $\to$ posterior $\to$ pointwise; Table~\ref{tab:ladder})
delivers the first three \emph{without} injectivity; pointwise
recovery is an optional upgrade. 
\emph{No} form of latent independence is needed for the affine link: Theorems \ref{thm:lid} and \ref{thm:mid} hold for any component covariances. Only the ICA-form refinement (Cor.~\ref{cor:pwa-ica}) requires them to be diagonal, i.e., conditional independence of coordinates given the component.


\begin{table}[h]
\caption{The three contrast principles powering the engine.\edtag{E7}
  Domain contrast \MST{} refines the classical observed-domain
  assumption; mechanism contrast \NEW{\SB{}/}\USB{} is new and exploits PWA branch boundaries; \NEW{interaction contrast \AD{} is the parameter-level analogue of Independent Causal
  Mechanisms.}\edtag{E48}}%
\label{tab:contrasts}
\centering
\small
\setlength{\tabcolsep}{3pt}
\begin{tabular}{@{}l c l p{7.9cm}@{}}
\toprule
\textbf{Contrast} & \textbf{Asn.}
  & \textbf{Constrains}
  & \textbf{Informal content} \\
\midrule
domain      & \MST  & latent law       & \leavevmode\NEW{no affine change of latent coordinates maps the mixture onto itself (\NEWC{\NEWD{replaces}\dtag{B12-3} TCL/iVAE's sufficient variability}\ctag{B10-9})} \\
\NEW{mech., existential} & \NEW{\SB} & \NEW{map} & \leavevmode\NEWB{an observational}\btag{B5-5} \NEW{boundary across which a solo branch toggles} \\
mech.\NEW{, universal} & \USB  & map              & \leavevmode\NEW{every branch gets such a boundary} \\
interaction & \AD   & map--law coupling & map and latent parameters do not conspire \\
\bottomrule
\end{tabular}
\vspace{-6pt}
\end{table}

\textbf{Contributions.}
{\setlength{\leftmargini}{1.5em}
\begin{enumerate}

\item \emph{Mechanism and interaction contrast} as principles
  complementary to domain contrast.
  \emph{Takeaway:} branch boundaries---and the discontinuities they
  may carry---are a resource, not an obstruction.
\item \emph{A group-theoretic identifiability engine.} Continuity and
  differentiability are replaced by triviality of an affine symmetry
  group, freeing pushforward identifiability from the smoothness
  restrictions of the differential proof tradition.
  \emph{Takeaway:} a single algebraic principle (symmetry collapse)
  carries the analysis from law identifiability through ICA-form
  recovery, \NEWC{and shows iVAE's sufficient variability as symmetry breaking}\ctag{B10-3}.
\item \emph{A hierarchy of identifiability} (law $\to$ map $\to$
  posterior $\to$ pointwise)\NEWC{, with conditions modular along the way}\ctag{B10-4}.
  \emph{Takeaway:} identifiable representation learning is achievable
  without injectivity and without any latent independence assumption.
\item \emph{Truth-sided assumptions.}
  \NEW{\AD{} is the only one of our contrasts required on the estimator side. The other contrasts live on the truth side.}\edtag{E47}
  \emph{Takeaway:} learning methods need not bake every theoretical
  condition into architecture, saving compute and admitting flexibility.
\end{enumerate}}

\section{Related Work}
\label{sec:related-work}

A comprehensive comparison along all assumption axes is in
Table~\ref{tab:comparison} (App.~\ref{app:comparison}); detailed
technical comparisons with iVAE \citep{khemakhem2020variational}, \citet{kivva2022identifiability}, and boundary-richness vs.\
injectivity appear in Appendix~\ref{app:tech-comparisons}. Here we
highlight the lineages most relevant to our positioning.

\textbf{Nonlinear ICA.}
TCL \citep{hyvarinen2016unsupervised} and iVAE
\citep{khemakhem2020variational} identify smooth invertible mixings of
independent latents whose marginals vary across observed segments or
auxiliaries. \NEWC{In effect, our \condref{MST} weakens TCL/iVAE's sufficient-variability assumption: it needs no labels, and generically as few as two components, where TCL/iVAE needs more than~$2\ambientdim$ \NEWD{for Gaussian latents of}\dtag{B12-5} dimension~$\ambientdim$ (Remarks~\ref{rem:ivae-mst} and~\ref{rem:mst-generic}).}\ctag{B10-5}
\NEWC{With a volume-preserving mixing, \citet{yang2022nonlinear} also need only two classes, but observed ones.}\ctag{B11-2}
Multi-view nonlinear ICA \citep{gresele2019incomplete} uses
several full mixings of a shared independent latent under a
sufficiently-distinct-views condition. Independent mechanism analysis (IMA)
\citep{gresele2021independent} postulates Jacobian-column orthogonality
of a smooth invertible mixing and is implemented through normalising
flows that are invertible by construction.
The mechanism-sparsity
line \citep{matthes2025mechanistic,lachapelle2023additive,
brady2023provably,brady2024interaction,zheng2023generalizing,zheng2022identifiability}\edtag{E9} places
interaction conditions on smooth invertible decoders. All are
differential at heart and assume smooth, typically invertible
generators. \NEW{Our \USB{} is a PWA-geometric analogue of their structural-sparsity conditions (App.~\ref{app:tech-comparisons}).}

\textbf{Relaxing injectivity and continuity.}
To our knowledge, only three works in this lineage (partially) relax
injectivity.
\citet{kivva2022identifiability} require injectivity \emph{on an
$\mapf$-image region} for LID and global injectivity for MID;
\citet{matthes2025mechanistic} focus on local disentanglement under local injectivity; \citet{roeder2021linear} obtain linear identifiability for \emph{discriminative} softmax-form models---a
different paradigm from generative pushforwards. Continuity of $\mapf$ is dropped only by us; \citet{kivva2022identifiability} rely on continuity for
density matching to obtain MID, though allow discontinuity for LID.

\textbf{Group symmetry in identifiability and disentanglement.}
\citet{higgins2018towards} define disentanglement via a \emph{given}
direct-product decomposition of a group on world states; the latent
law plays no role and there is no theorem for any concrete model
class. \NEWB{\citet{xi2023indeterminacy} taxonomise}\btag{B5-18} model indeterminacies as
$\mathcal{A}(\mathcal{F})\cap\mathcal{A}(\mathcal{P}_z)$; they require
injective generators and resolve multi-domain ambiguity by
\emph{observed} domain labels, and handle the GMM analogue only
by collapsing each domain to a single Gaussian (their \S4.3,
Thm.~4.5)---bypassing rather than trivialising the within-domain
symmetry that our \MST{} addresses. \NEW{Concurrently, \citet{squires2026unifying} reduce identifiability to whether an intersection of two transition sets lies in a target group, extending \citet{xi2023indeterminacy} to stochastic mixing.\edtag{E10} For PWA--GMM, \AD{} and \SB{}/\USB{} constrain the transitions between decoders, and \MST{} trivialises the latent law's own symmetry group, not the target group $\Aff(\R^{\ambientdim})$ (App.~\ref{app:tech-comparisons}).} \citet{ahuja2022properties} place
structure on \emph{known} latent dynamics: shared equivariances of the
dynamics define the ambiguity class, with smooth invertible
observation map throughout. Group ideas long appear in invariant
\textit{inference} \citep{eaton1989group,giri1996group}, starting from a
\emph{prescribed} group action.

\textbf{PWA${}+{}$GMM, and multi-mechanism modelling.}
\citet{kivva2022identifiability} opened the unsupervised PWA+GMM
regime; \citet{xu2026identifiability} extend to potentially-degenerate
GMMs via encoder sparsity (still injective continuous PWA);
\citet{lopez2024toward} specialise Kivva to comparative DGMs with an
\emph{observed} contrastive split. \NEW{Bayesian ICA with mixture priors \citep{choudrey2002variational} and VaDE, GMVAE and MFCVAE \citep{jiang2017variational,dilokthanakul2016deep,falck2021multi} pursued this class empirically.\edtag{E11}} \NEWB{For its equidimensional case, our conditions are sufficient for identifiability.}\btag{B5-19} Generalised PCA
\citep{vidal2005generalized} and clusterwise linear regression
\citep{hennig2000identifiability} are remote relatives---single-condition
projective settings without unobserved domain labels or branch
analysis. Modelling generating processes via multiple temporal or
spatial mechanisms is classically well-motivated: temporal exemplars
include switching/hybrid systems \citep{paoletti2007identification},
recurrent switching linear dynamics \citep{linderman2017bayesian}, and
PWA neuroscience dynamics \citep{durstewitz2017state}; spatial
exemplars include free-form-deformation and PWA image registration
\citep{rueckert1999nonrigid,sotiras2013deformable,
freifeld2017transformations,su2022nonrigid}. None addresses
identifiability; we bring algebraic identifiability to this PWA
modelling tradition.

\section{Setup and preparations}
\label{sec:setup}

A latent random variable~$\latentZ$ follows a GMM
and is observed through an unknown PWA map~$\mapf$;
that is, $\RVobs = \mapf(\latentZ)$.%
\footnote{Additive independent noise $\varepsilon$ with 
non-vanishing characteristic function can be
  handled via deconvolution; see e.g., \citep{khemakhem2020variational}.}
Both the map and the latent law are hidden: only the data
of~$\RVobs$ is available. We ask whether $\latentZ$ and~$\mapf$ can
be (partially) recovered from this data.

\subsection{PWA maps and Gaussian mixture models}
\label{subsec:pwa-gmm}

\begin{definition}[Irredundant GMM]
\label{def:gmm}
A random variable~$\GMMgeneric$ is an \emph{irredundant GMM} with
$\numcompgeneric$ components if its density is
$\sum_{k=1}^{\numcompgeneric}\wgeneric_k\,\NEW{\NormalDensityParam{\mugeneric_k,\Sigmageneric_k}}$\NEW{, with $\NormalDensityParam{\mugeneric_k,\Sigmageneric_k}$ the density of $\Gauss(\mugeneric_k,\Sigmageneric_k)$}\edtag{E12},
where all weights are positive, all covariances are positive definite,
and parameters $(\mugeneric_k,\Sigmageneric_k)\neq(\mugeneric_\ell,\Sigmageneric_\ell)$
for $k\neq\ell$.
\end{definition}

\begin{definition}[Reduced PWA map]
\label{def:reduced-pwa}
A map $\mapgeneric\colon\R^{\ambientdim}\to\R^{\ambientdim}$ is a
\emph{reduced PWA map} if it is specified by a finite index
set~$\indexgeneric$ and, for each $i\in\indexgeneric$, an invertible
affine rule $(\lingeneric_i,\biasgeneric_i)$ together with a branch
domain~$\branchgeneric_i\subset\R^{\ambientdim}$%
\footnote{The \emph{branch domain}~$\branchgeneric_i$ is a piece of the map's
  \emph{input} space~$\R^{\ambientdim}$; it should not be confused with a
  ``domain'' in the multi-domain sense (a mixture component / environment).},
such that
each~$\branchgeneric_i$ is a finite union of full-dimensional polyhedra
with pairwise disjoint interiors; the branch domains
$\{\branchgeneric_i\}_{i\in\indexgeneric}$ partition~$\R^{\ambientdim}$
up to measure-zero boundaries; and the rule map
$i\mapsto(\lingeneric_i,\biasgeneric_i)$ is injective.
\end{definition}

Reducedness prevents duplicate bookkeeping: knowing the local affine
rule determines the branch index. This is not a restriction in
practice---any PWA map can be converted to a reduced one by merging
branches with identical affine rules (their unions are always finite
unions of full-dimensional polyhedra). A single branch index may
therefore be active on several disconnected regions of the observational
space; this causes no issue in the analysis below.

\begin{definition}[PWA equality]
\label{def:pwa-equality}
Two reduced PWA maps $\mapf$ and~$\mapg$ are \emph{equal}, written
$\mapf=\mapg$, if there exists a bijection $\pi\colon\Ig\to\If$ such
that matched branches have identical affine rules
($\ling{j}=\linf{\pi(j)}$ and $\biasg{j}=\biasf{\pi(j)}$) and their
domains coincide up to measure-zero boundaries.
\end{definition}

\medskip\noindent\textbf{Convention.}
Throughout, \emph{PWA map} and \emph{GMM} refer to
reduced PWA maps and irredundant GMMs, respectively. \NEW{Equalities of branch domains and chambers} \NEWB{(\S\ref{subsec:chambers})}\btag{B5-16} \NEW{hold up to null sets.}\edtag{E13}

\subsection{Identifiability}
\label{subsec:identifiability}

We fix an arbitrary \emph{truth} pair $(\mapf,\latentZ)$ and an
\emph{estimator} pair $(\mapg,\latentZest)$, both PWA--GMM pairs
in~$\R^{\ambientdim}$. We observe $\RVobs = \mapf(\latentZ)$. The
central assumption throughout is the \emph{distributional equality}
\begin{equation}
  \label{eq:distributional-equality}
  \mapf(\latentZ) \deq \mapg(\latentZest),
\end{equation}
meaning the two models produce the same observable distribution. We
consider a hierarchy of identifiability notions, from weakest to
strongest, summarised in Table~\ref{tab:ladder}.


\begin{table}[h]
\caption{Identifiability hierarchy.\edtag{E14} \condref{AD} on both sides is
  always assumed; \NEW{\SB{}, \USB{} and \MST{} are on the truth side only. `+' means: in addition to all assumptions in the rows above.}
  Under \condref{DD} \NEW{on both sides} (conditional independence required), $\alpha$ takes the ICA form
  $P\Lambda z+c$ at every level (Cor.~\ref{cor:pwa-ica}).}\vspace{-0pt}
\label{tab:ladder}
\centering
\small
\setlength{\tabcolsep}{3pt}
\begin{tabular}{@{}l p{3.0cm} p{2.8cm} c p{5.5cm}@{}}
\toprule
\textbf{Level} & \textbf{Assns.} & \textbf{Conclusion} & \textbf{Ref.} & \textbf{Meaning} \\
\midrule
LID
  & \condref{AD} + \condref{SB}
  & $\alpha(\latentZest)\deq\latentZ$
  & Thm~\ref{thm:lid}
  & Latent distribution recovered up to a global affine reparametrization. \\
MID
  & $+$ \condref{USB} $+$ \condref{MST}
  & $\mapf = \mapg\circ\alpha^{-1}$
  & Thm~\ref{thm:mid}
  & Generative mechanism recovered up to~$\alpha$. \\
Post.
  & \leavevmode\NEW{(directly from MID; no additional assns)}
  & posteriors linked by $\alpha$
  & Prop~\ref{prop:posterior-id}
  & Posterior inference translates between the two models. \\
Point.
  & $+$ $\mapf$ injective
  & $z = \alpha(\hat{z})$
  & Prop~\ref{prop:pointwise-id}
  & True source recovered pointwise (classical ICA goal). \\
\bottomrule
\end{tabular}
\end{table}

Each level builds on the previous. Under \condref{DD} (conditional independence of coordinates given the component),
the affine link~$\alpha$ is further refined to the ICA form
$\alpha(z)=P\Lambda z+c$ (Cor.~\ref{cor:pwa-ica}); this is not a new
rung in the hierarchy, but a cross-cutting refinement that sharpens
$\alpha$ at every level. 

\subsection{Chambers and the pushforward density}
\label{subsec:chambers}

\NEWB{Identifiability relies on observational information.}\btag{B5-6} To analyze~\eqref{eq:distributional-equality}, we decompose the
\NEWB{\emph{observational space}} into regions\NEWB{, i.e., \emph{chambers},} where a map acts with constant branch
structure. This converts the global PWA pushforward equality into a
\emph{local} matching of finite Gaussian mixtures.

\begin{definition}[Chambers of a single map]
\label{def:chambers}
For a PWA map~$\mapgeneric$ with branch index set~$\indexgeneric$, the
\emph{active set} at $x\in\R^{\ambientdim}$ is
$\Ichamber{\mapgeneric}{x}
  := \bigl\{i\in\indexgeneric : x \in \mapgeneric(\branchgeneric_i)\bigr\}$.
A \emph{chamber} of~$\mapgeneric$ is a maximal open connected subset
of~$\R^{\ambientdim}$ on which $\Ichamber{\mapgeneric}{\cdot}$ is
constant.
\end{definition}

\NEW{Informally, the active set lists the branches that could have produced} \NEWB{the observation}\btag{B5-7}~\NEW{$x$, and a chamber keeps it constant. For example, $\mapgeneric(z)=|z|$ on~$\R$ has active set $\{1,2\}$ on $x>0$.}\edtag{E15} See Appendix Figure~\ref{fig:exC} for a 4-branch PWA map with 4 chambers. In general, the chamber structures of two different maps~$\mapf$
and~$\mapg$ may not coincide: one map's chamber boundary may fall in
the interior of another map's chamber. We therefore refine the
partition to the common one.

\begin{definition}[Fine chambers]
\label{def:fine-chambers}
Given two PWA maps $\mapf,\mapg$, the \emph{fine chambers} are the
maximal open connected subsets of~$\R^{\ambientdim}$ on which both
$\Ichamber{\mapf}{\cdot}$ and $\Ichamber{\mapg}{\cdot}$ are
simultaneously constant.
\end{definition}

\NEW{Informally, they overlay the chamber boundaries of both maps.}\edtag{E53} On each fine chamber~$\fineC$, we write $\Ichamber{\mapf}{\fineC}$ and
$\Ichamber{\mapg}{\fineC}$ for the constant active sets. The
\emph{pushforward} of the $k$-th latent component under truth
branch~$i$ is the Gaussian $\Gauss(\paramf{i}{k})$ with parameters
$\paramf{i}{k} :=
  \bigl(\linf{i}\mutrue{k}+\biasf{i},\;\linf{i}\Sigmatrue{k}\linf{i}^\top\bigr)$,
and analogously $\paramEst{j}{\ell}$ for~$(\mapg,\latentZest)$.
Inside~$\fineC$, the distributional equality~\eqref{eq:distributional-equality}
becomes a pointwise identity between two finite Gaussian mixtures:
\begin{equation}
  \label{eq:chamber-equality}
  \textstyle\sum_{i\in\Ichamber{\mapf}{\fineC}}
    \sum_{k=1}^{\numcomptrue}
      \wtrue{k}\,\NormalDensityParam{\paramf{i}{k}}(x)
  =
  \sum_{j\in\Ichamber{\mapg}{\fineC}}
    \sum_{\ell=1}^{\numcompest}
      \west{\ell}\,\NormalDensityParam{\paramEst{j}{\ell}}(x)
\end{equation}
for all $x\in\fineC$. Crucially, each Gaussian kernel
$\NormalDensityParam{\paramf{i}{k}}$ is defined globally
on~$\R^{\ambientdim}$ (its parameters come from the branch rule
$(\linf{i},\biasf{i})$ and the component parameters, not
from~$\fineC$). Extending both sides of~\eqref{eq:chamber-equality} by
analytic continuation from the chamber~$\fineC$ produces two
globally-defined analytic functions $F_{\fineC}$ and $G_{\fineC}$ that
agree everywhere on~$\R^{\ambientdim}$. We stress that this is an
analytic extension of the \emph{chamber-mixture function}, not of the
true pushforward density---the latter is given by different active sets
on different chambers. In particular, different fine chambers generally
yield different analytic extensions, and we never equate extensions
across chambers. This local--global interplay is a recurring theme of
our results.

We can already draw a nontrivial conclusion without any further
assumptions.

\begin{lemma}[Active-branch-count equality]
\label{lem:branch-count}
Under~\eqref{eq:distributional-equality}, the number of active
branches is the same on both sides of every fine chamber:
$|\Ichamber{\mapf}{\fineC}| = |\Ichamber{\mapg}{\fineC}|$.
\end{lemma}

In other words, the total ``mass budget'' on each fine chamber is
determined by the number of active branches, and this number must
match on both sides.

\section{Theoretical results}
\label{sec:results}

Full proofs of all results appear in Appendix~\ref{app:proofs};
proof sketches for the key results are given below. Illustrated examples are given in Appendix~\ref{app:examples} and \ref{app:sb-usb-winj}.

\textbf{Proof strategy.}
Our whole proof strategy has an \emph{algebraic} flavour and is built around \emph{bijections} between the truth and estimator, at progressively higher levels. \condref{AD} gives a \emph{branch-component-pair bijection} $\Psi$.
\condref{SB} restricts $\Psi$ to a single boundary and reads off a component-level affine link $\alpha$: it ``matches''--via the induced \emph{component bijection}--every estimator component to a truth component \emph{and} ``stitches'' the latent components together by this same $\alpha$, yielding LID.
\condref{USB} gives a \emph{branch bijection}: it ``matches'' every estimator branch to a truth branch via its own simple boundary.
\condref{MST} upgrades the branch bijection to a global branch-level affine $\alpha$: it ``stitches'' across all branches by trivialising the mixture symmetry group, completing MID.

\subsection{\NEW{Pushforward atom distinctness} and consequences}
\label{subsec:param-matching}

We now introduce the main genericity condition on map--latent pairs.

\textbf{Condition \condref{AD} (\ADname).}\NEW{\footnote{Called parameter injectivity, (PI), in the preprint version of this paper.}}
\condlabel{AD}
\NEW{Call the pushforward Gaussian $\Gauss(\paramImg{j}{\ell})$ of component~$\ell$ through branch~$j$ an \emph{atom}.}
A PWA-GMM pair $(\mapgeneric,\GMMgeneric)$ satisfies \condref{AD} if \NEW{its atoms are pairwise distinct, i.e., the map}
 from branch--component indices $(j,\ell)$ $\mapsto$ pushforward parameters $\paramImg{j}{\ell}$ is injective.\edtag{E1}

This is the \emph{interaction contrast} of Table~\ref{tab:contrasts}: it
formalises the principle that the map rules and the latent parameters
should not conspire to produce collisions in the \NEWB{observational}\btag{B5-14} space.
Collisions within the same branch are already excluded by invertibility
of the branch rule composed with irredundancy of the GMM; \condref{AD}
additionally forbids cross-branch collisions. It decomposes into two sub-conditions:

\begin{enumerate}[leftmargin=*]

\item Same-component \condref{AD}:
  no collision
  $\paramImg{j}{\ell}=\paramImg{m}{\ell}$.
  A violation means the \emph{branch-difference map}
  $\mapgeneric_m^{-1}\circ\mapgeneric_j$
  is a \emph{symmetry} of
  component~$\Gauss_\ell$.
  
\item Cross-component \condref{AD}:
  no collision
  $\paramImg{j}{\ell}=\paramImg{m}{n}$
  with $\ell\neq n$.
  Such a collision forces the affine difference
  map $\mapgeneric_m^{-1}\circ\mapgeneric_j$
  to \emph{transport} component~$\Gauss_\ell$
  exactly onto~$\Gauss_n$.

\end{enumerate}
Both are \emph{generic}, i.e., hold almost surely (dimension counting in Appendix~\ref{app:pi-genericity}).
\NEW{\AD{} is the parameter-level analogue of Independent Causal Mechanisms}~\citep{scholkopf2021toward}\edtag{E16}: GMM parameters are the mechanism for how $\latentZ$ is
generated, and affine branch rules are the mechanism for how
$\RVobs$ is generated from $\latentZ$; \AD{} forbids conspiracy
between them at the parameter level. 

Chamber-wise \condref{AD}---\NEW{atom distinctness} restricted to active pairs per chamber---is weaker than the stated global version; the glueing step in Lemma~\ref{lem:global-param-matching} requires the full global condition (see Appendix~\ref{app:global-pi} for the distinction and a counterexample).

We first record a standard but central fact: distinct non-degenerate
Gaussian densities are linearly independent as functions (see
Appendix~\ref{app:gaussian-indep} for the precise statement). We denote this below as
Lemma~\ref{lem:gaussian-independence} and refer to it as the
\emph{Gaussian-independence lemma}.

\begin{remark}[Equality on open sets implies equality everywhere]
\label{rem:open-set}
If two finite Gaussian mixtures agree on any open set
$U\subseteq\R^{\ambientdim}$, they agree on all of~$\R^{\ambientdim}$:
both are real-analytic, and $\R^{\ambientdim}$ is connected. The
Gaussian-independence lemma then applies to their difference:
\emph{parameters appearing on one side must also appear on the other,
and matched parameters carry equal coefficients.} We will use this
observation repeatedly when
analysing~\eqref{eq:chamber-equality} on each fine chamber.
\end{remark}

\begin{lemma}[Global parameter bijection under (AD)]
\label{lem:global-param-matching}
Assume~\eqref{eq:distributional-equality} and that both
$(\mapf,\latentZ)$ and $(\mapg,\latentZest)$ satisfy \condref{AD}.
Then there exists a unique bijection
\begin{equation}
  \label{eq:psi-global}
  \Psi\colon \Ig\times\IndexEst \;\longleftrightarrow\; \If\times\IndexTrue
\end{equation}
such that $\paramEst{j}{\ell}=\paramf{i}{k}$ and
$\west{\ell}=\wtrue{k}$ whenever $\Psi(j,\ell)=(i,k)$. On each fine
chamber~$\fineC$, $\Psi$ restricts to a bijection between the active
pairs on each side.
\end{lemma}
\vspace{-10pt}
\begin{proof}[Proof sketch]
On each fine chamber~$\fineC$, both sides
of~\eqref{eq:chamber-equality} are finite GMMs agreeing on an open
set; Remark~\ref{rem:open-set} and \condref{AD} yield a
bijection~$\Psi_{\fineC}$ matching active branch-component pairs by
parameter equality. Uniqueness of each match (by \condref{AD} on both
sides) forces all chamber-wise bijections to agree: if $(j,\ell)$ and
$(i,k)$ are matched on some~$\fineC$, then
$\paramEst{j}{\ell}=\paramf{i}{k}$ is a global identity, so
$\Psi_{\fineC'}(j,\ell)=(i,k)$ for any other chamber $\fineC'$
containing branch~$j$.
This is the step that requires \emph{global} rather than merely
chamber-wise \condref{AD} (Remark~\ref{rem:local-pi}).
\end{proof}
\vspace{-10pt}
Lemma~\ref{lem:branch-count} (no~\condref{AD}) and
Lemma~\ref{lem:global-param-matching} (with~\condref{AD}) form the
\emph{foundational machinery} reused throughout this paper.
The former controls the branch-count budget;
the latter turns density matching into a discrete combinatorial bijection.
Both play the role of \emph{matching} the estimator to the truth at
the most basic level; all higher-level identifiability results will
build on this matching. We assume double~\condref{AD} (on both sides) hereafter. 

For adjacent fine chambers $(\fineC,\fineCp)$ \NEW{(two whose closures share a $(\ambientdim{-}1)$-dimensional boundary)}\edtag{E49}, define the
\emph{leaving count}
$\Df(\fineC,\fineCp):=|\Ichamber{\mapf}{\fineC}\setminus\Ichamber{\mapf}{\fineCp}|$,
and symmetrically for~$\mapg$. By swapping $\fineC\leftrightarrow\fineCp$,
one obtains \emph{entering counts} $\Ef,\Eg$.

\noindent\emph{Matching identifies three combinatorial invariants---the
component count, the branch count, and the chamber structure}
(Lemmas~\ref{lem:comp-count}, \ref{lem:boundary-counting},
and~\ref{lem:chamber-coincidence} in Appendix~\ref{app:proofs}),
collected in Proposition~\ref{prop:matching-consequences}.

\begin{proposition}[Consequences of matching]
\label{prop:matching-consequences}
Under~\eqref{eq:distributional-equality} and double \condref{AD}:
the component counts agree, $\numcompest = \numcomptrue =: K$,
which immediately implies $|\If|=|\Ig|$;
at every pair of adjacent fine chambers the leaving and entering branch counts agree,
$\Df(\fineC,\fineCp) = \Dg(\fineC,\fineCp)$ and $\Ef(\fineC,\fineCp) = \Eg(\fineC,\fineCp)$;
and consequently
$\mapf$-chambers, $\mapg$-chambers, and fine chambers all coincide.
\end{proposition}

We write $K:=\numcomptrue = \numcompest$ throughout the rest of the
paper.
It is worth stressing that $|\If|=|\Ig|$ does \emph{not} follow
from active-count equality (Lemma~\ref{lem:branch-count}) alone:
chambers generally share branches, and before \condref{AD}, we have no
reason to expect the sharing patterns of the two maps to match.
It is the global bijection~$\Psi$ (Lemma~\ref{lem:global-param-matching})
that forces $|\If|=|\Ig|$, and then the boundary counting forces
the sharing patterns themselves to agree.

We begin with the simplest but already nontrivial identifiability
setting: the observable law is itself a GMM. This corresponds to the
true map being the identity, $\mapf=\Id$.

\begin{theorem}[From PWA to affine: the symmetry-rigidity]
\label{thm:global-affine}
Let $\mapg$ be a PWA map, $\latentZest$ a GMM, and suppose
$(\mapg,\latentZest)$ satisfies \condref{AD}. If
$\mapg(\latentZest)\deq\latentZ$, where $\latentZ$ is a GMM, then
$\mapg$ is globally affine: $\mapg(x) = Ax+b$ for some
$A\in\GL(\ambientdim),\; b\in\R^{\ambientdim}$, and for all~$x$.

In particular, if $\mapg(\latentZest)\deq\latentZest$, $\mapg$ is globally affine. That is, \emph{\textbf{no} strict PWA} self-map satisfies \condref{AD}.
\end{theorem}

\begin{remark}[Essential role of (AD)]
\label{rem:pi-essential}
Theorem~\ref{thm:global-affine} highlights why \condref{AD} is
essential for \emph{any} identifiability result: if a GMM~$p$
admitted a \emph{strict} PWA self-map~$S$ (at least two branches,
$S_\sharp p=p$), then for any truth map~$\mapf$ and $\latentZ\sim p$,
setting $\mapg:=\mapf\circ S$ and $\latentZest:=\latentZ$ yields
$\mapg(\latentZest)=\mapf(S(\latentZ))\deq\mapf(\latentZ)$. Here
$\mapg$ is genuinely different from any affine reparametrization
of~$\mapf$, so identifiability fails. \condref{AD} is precisely the
condition that rules out such pathologies: by
Theorem~\ref{thm:global-affine}, any PWA self-map of a GMM
respecting \condref{AD} must be globally affine.
\end{remark}

\begin{remark}[Hard cases]
\label{rem:hard-cases}
Theorem~\ref{thm:global-affine} also implies: under double
\condref{AD}, if either $\mapf$ or~$\mapg$ is globally affine, so is
the other. Indeed, if $\mapf$ is affine, then $\mapf(\latentZ)$ is a
GMM, and the theorem forces~$\mapg$ to be affine. The focus in the
sequel is the nontrivial case where \emph{both} maps are
\emph{strictly PWA} (at least two branches) and the pushforward is
\emph{not} a GMM.
\end{remark}

\subsection{Law identifiability}
\label{subsec:lid}

We now tackle the first nontrivial step in the hierarchy:
\emph{law identifiability (LID)}, where we show the two latent GMMs
are related by a global affine map. This requires a
boundary-genericity condition---the \emph{mechanism contrast} of
Table~\ref{tab:contrasts}.

\textbf{Condition \condref{SB} (\SBname).}
\condlabel{SB}
A PWA map~$\mapgeneric$ satisfies \condref{SB} if either
(a)~$\mapgeneric$ has a single chamber
(i.e., $\mapgeneric$ is globally affine; Lemma~\ref{lem:single-chamber-affine}), 
or
(b)~there exist adjacent chambers $(\fineC,\fineCp)$ of~$\mapgeneric$
and a branch~$i_\star$ that is the \emph{unique} branch leaving \emph{or} entering
($\Ichamber{\mapgeneric}{\fineC}\setminus\Ichamber{\mapgeneric}{\fineCp}=\{i_\star\}$
\emph{or}
$\Ichamber{\mapgeneric}{\fineCp}\setminus\Ichamber{\mapgeneric}{\fineC}=\{i_\star\}$).

\NEWB{\condref{SB} exposes \emph{observational} boundary information, which Theorem~\ref{thm:lid} employs for identification.}\btag{B5-8}
\condref{SB}
constrains only \emph{one direction} of the boundary crossing
(leaving \emph{or} entering); the opposite direction is unrestricted.
Thus a boundary with one branch leaving and three branches entering
qualifies as simple.
\textbf{Note:} \condref{SB} is satisfied even if multiple branches overlap across many boundaries, provided \emph{at least one} directional boundary crossing changes the active set by exactly one element.

\begin{remark}[Genericity of (SB)]
\label{rem:sb-generic}
The violation of \condref{SB} for a strict PWA map is a rigid
condition: \emph{every} boundary must have at least two branches
leaving together \emph{and} at least two branches entering together.
\NEW{Each} \NEWB{chamber}\btag{B5-9} \NEW{boundary of a generic PWA map} is typically traversed by a single branch, \NEW{so \condref{SB} is generic under unconstrained parametrizations (App.~\ref{app:winj-nongeneric}).}\edtag{E17}
\NEW{For continuous PWA maps, genericity is a conjecture: co-toggling can happen at shared breakpoints
(Appendix~\ref{app:winj-not-sb}), but a \textit{single} non-co-toggling boundary suffices.}\edtag{E50}
\end{remark}

\begin{theorem}[Law identifiability]
\label{thm:lid}
Assume~\eqref{eq:distributional-equality} with \condref{AD} on both
sides and \condref{SB} on~$\mapf$. Then $\numcompest=\numcomptrue=:K$,
and there exists an affine bijection $\alpha\in\Aff(\R^{\ambientdim})$
such that $\alpha(\latentZest)\deq\latentZ$; more precisely, there is a permutation~$\sigma$ of~$[K]$ with $\alpha_\sharp\Gauss(\muest{\ell},\Sigmaest{\ell})=\Gauss(\mutrue{\sigma(\ell)},\Sigmatrue{\sigma(\ell)})$ and $\west{\ell}=\wtrue{\sigma(\ell)}$ for every~$\ell$.
\end{theorem}
\vspace{-10pt}
\begin{proof}[Proof sketch]
Lemma~\ref{lem:comp-count} gives $K'=K$. The \condref{SB} boundary
yields unique toggling branches~$i_\star$ (truth, WLOG leaving) and
$j_\star$ (estimator, by Lemma~\ref{lem:boundary-counting}). The
bijection~$\Psi$ restricted to
$\{j_\star\}\times[K]\leftrightarrow\{i_\star\}\times[K]$ induces a
permutation~$\sigma$ and, via the parameter equality, determines the
affine map
$\alpha(z):=\linf{i_\star}^{-1}
  (\ling{j_\star}z+\biasg{j_\star}-\biasf{i_\star})$,
which transports each estimator component to the corresponding truth
component with matching weights.
\end{proof}
\vspace{-10pt}
The link~$\alpha$ is constructed from a specific toggling pair---it is not claimed canonical. LID says only that \emph{some} affine link exists. What will matter for MID is \emph{globalness}: forcing the same structural relationship across every branch, which is exactly what \condref{MST} delivers.

\subsection{Map identifiability}
\label{subsec:mid}

We now address the stronger question: must the maps $\mapf$
and~$\mapg$ themselves coincide (up to the latent reparametrization
$\alpha$)? The key challenge is that LID establishes the existence
of~$\alpha$ via a \emph{single} toggling boundary, but says nothing
about other branches. For MID, we must ensure that \emph{every branch}
of~$\mapg$ is related to the corresponding branch of~$\mapf$ by the
\emph{same} affine link. This requires two additional ingredients:
boundary richness (every branch is \NEW{exposed} at some simple boundary)
and symmetry triviality---the \emph{domain contrast} of
Table~\ref{tab:contrasts}.

\textbf{Condition \condref{USB} (\USBname).}
\condlabel{USB}
A PWA map~$\mapgeneric$ satisfies \condref{USB} if either
(a)~$\mapgeneric$ has a single chamber, or
(b)~for every branch $i\in\indexgeneric$, there exist adjacent
chambers $(\fineC,\fineCp)$ of~$\mapgeneric$ such that~$i$ is the
unique leaving branch or the unique entering branch.

\condref{USB} is \emph{universal for branches}: every branch is
\NEW{exposed} at some simple boundary. It does not require every boundary
to be simple. \condref{USB} trivially implies \condref{SB}. Similar to \condref{SB}, \condref{USB} is satisfied even if multiple branches overlap across many boundaries.
\condref{USB} is \emph{strictly weaker than global injectivity}; instead of weak injectivity \citep{kivva2022identifiability}, which demands a fragile, globally unopposed $1$-active chamber, \condref{USB} relies only on a robust, local combinatorial asymmetry (Appendix~\ref{app:sb-not-winj}\&\ref{app:winj-nongeneric}).

\begin{remark}[Genericity of (USB)]
\label{rem:usb-generic}
\NEW{A} violation \NEW{of \condref{USB}} requires that for some branch~$i$, every \NEWB{chamber}\btag{B5-10}
\NEW{boundary} of~$i$ is simultaneously shared by at least one other branch on
both the leaving and entering sides. \NEW{Like \condref{SB}, \condref{USB} is therefore generic under unconstrained parametrizations.}\edtag{E18}
\end{remark}

We next introduce the symmetry condition that ``locks'' the
identification. Given a GMM with density~$p$, its \emph{mixture
symmetry group} is $\Gmix{p} :=
\bigl\{T\in\Aff(\R^{\ambientdim}):\pushfwd{T}{p}=p\bigr\}$. This
group captures all affine bijections that preserve the mixture law as
a whole. It can include maps that \emph{permute components of equal
weight} (not just maps that fix each component individually; the
distinction is important and discussed in Appendix~\ref{app:dw-jst}).

\textbf{Condition \condref{MST} (\MSTname).}
\condlabel{MST}
A GMM~$p$ satisfies \condref{MST} if $\Gmix{p}=\{\Id\}$.

From the symmetry standpoint, \condref{MST} is
stronger than \condref{AD}: \condref{AD} rules out \emph{strict-PWA}
self-maps of the mixture (Remark~\ref{rem:pi-essential}), while
\condref{MST} further rules out \emph{affine} self-maps beyond the
identity.
\condref{MST} is required only for MID and above; likewise $K>1$ is needed only there.
Indeed, for $K=1$ (a single Gaussian), LID is trivial: any two non-degenerate Gaussians on $\R^\ambientdim$ differ by an affine bijection.

\condref{MST} is also \emph{necessary} for MID in a structural sense; see \hyperref[ex:101]{Example~3} for a prototype (``101'') construction. And this mechanism can be extended to \emph{any} GMM and beyond (Appendix~\ref{app:torsion}).

\begin{remark}[(MST) weakens iVAE's sufficient variability\ctag{B10-6}]
\label{rem:ivae-mst}
\NEWC{\NEWD{iVAE's assumption~(iv) is posed for exponential families; for Gaussian latents with observed labels, it}\dtag{B12-4} implies \condref{MST} (Proposition~\ref{prop:ivae-mst}). Labels force every symmetry to fix each component; the $2\ambientdim+1$ components that (iv) requires then have diagonal precisions spanning all diagonal directions, which leaves only coordinate sign flips, and their varying means rule those out. \condref{MST} is \emph{strictly weaker} in two ways. With labels, $\ambientdim+1$ components with affinely independent means already give it (Lemma~\ref{lem:aff-indep-jst}), where (iv) needs $K>2\ambientdim$. Without labels, (iv) cannot be stated, while \condref{MST} is a property of the unlabelled mixture alone.}
\end{remark}

\begin{remark}[Genericity of (MST)]
\label{rem:mst-generic}
For $K\ge 2$ components with generically chosen means, covariances,
and weights, \condref{MST} holds. 
A violation requires either (a)~a
nontrivial affine map fixing every component individually, which is ruled out by, e.g., affinely independent means (Lemma~\ref{lem:aff-indep-jst} in Appendix), or (b)~a nontrivial affine map
permuting two or more components---i.e., the components are
partitioned into cycles and the map acts one-to-one within each
cycle---which is ruled out by, e.g., distinct weights
\emph{or} distinct covariance determinants
(Lemma~\ref{lem:cycle-det} in Appendix).
Other generic and sufficient conditions for \condref{MST} exist; see Appendix~\ref{app:sum-suff-cond} for a summary. All the conditions are algebraically
rigid and occur on a measure-zero set. See Figure~\ref{fig:ex5} in Appendix~\ref{ex:jst-mst} for an illustration of why case~(b) matters.
\end{remark}

\textbf{Truth-sided assumptions.} Before proceeding to the proofs, we record two invariance results that
clarify on which side(s) each assumption needs to live. The first says
that \condref{USB} transfers to the estimator side already from double
\condref{AD}---without constructing any affine link.

\begin{proposition}[(\condref{SB}) and (\condref{USB}) inheritance]
\label{prop:usb-inheritance}
If $\mapf(\latentZ)\deq\mapg(\latentZest)$, $(\mapf,\latentZ)$ and
$(\mapg,\latentZest)$ satisfy \condref{AD}, and~$\mapf$ satisfies
\condref{SB} (resp.\ \condref{USB}), then~$\mapg$ satisfies \condref{SB} (resp.\ \condref{USB}).
\end{proposition}

Moreover, \condref{MST} transfers via affine conjugacy once an LID link is established:
$\Gmix{p_{\latentZest}} = \alpha^{-1}\Gmix{p_{\latentZ}}\alpha$,
so $p_{\latentZ}$ satisfies \condref{MST} iff $p_{\latentZest}$ does (Proposition~\ref{prop:mst-conjugacy} in Appendix~\ref{app:proofs}).
\NEWB{Observational} chamber structure is also preserved by any affine latent reparametrization, so \condref{SB}/\condref{USB} are invariant under~$\alpha$.
In summary, \NEW{\condref{AD} is the only contrast required on the estimator side}\edtag{E19}; (\condref{SB}, \condref{USB}, \condref{MST}) are imposed on the truth side, and their propagation to the estimator is automatic.

We first establish MID in the simpler case where both maps push
forward the \emph{same} GMM law.

\begin{lemma}[MID with shared latent]
\label{lem:mid-shared}
Let $\mapf,\mapg$ be PWA maps, and let $\latentZ,\latentZest$ both
follow the same GMM law~$p$ with $K$ components.
Assume~\eqref{eq:distributional-equality}, \condref{AD} on both sides,
\condref{USB} on~$\mapf$, and \condref{MST} on~$p$. Then $\mapf=\mapg$.
\end{lemma}
\vspace{-10pt}
\begin{proof}[Proof sketch]
By \condref{USB}, each truth branch~$i$ has a simple boundary, and
Lemma~\ref{lem:boundary-counting} forces a unique estimator branch~$j(i)$;
the composite $S_i:=\mapf_i^{-1}\circ\mapg_{j(i)}$ then lies in~$\Gmix{p}$ by
the parameter match. The domain-contrast condition $\Gmix{p}=\{\Id\}$
(\condref{MST}) forces $S_i=\Id$, so $\mapf_i=\mapg_{j(i)}$---simultaneously
killing all branch-wise ambiguities and delivering the global affine link
required for MID. Finally, injectivity of $i\mapsto j(i)$
(Proposition~\ref{prop:usb-inheritance}) promotes it to a bijection, and
domain coincidence follows from Lemma~\ref{lem:chamber-coincidence} and
invertibility of the branch maps.
\end{proof}
\vspace{-10pt}
The general MID result follows by reducing to the shared-latent case.

\begin{theorem}[Map identifiability]
\label{thm:mid}
Assume~\eqref{eq:distributional-equality} with \condref{AD} on both
sides, and \condref{USB} plus \condref{MST} on the truth side
$(\mapf,\latentZ)$. Then there exists an affine bijection
$\alpha\in\Aff(\R^{\ambientdim})$ with
$\alpha(\latentZest)\deq\latentZ$ and
$\mapf = \mapg\circ\alpha^{-1}$.
\end{theorem}

\vspace{-10pt}
\begin{proof}[Proof sketch]
Since \condref{USB} implies \condref{SB}, Theorem~\ref{thm:lid} first
yields an affine~$\alpha$ with $\alpha(\latentZest)\deq\latentZ$.
Setting $\maph:=\mapg\circ\alpha^{-1}$ preserves the reduced-PWA structure,
and the parameter map for $(\maph,\latentZ)$ inherits injectivity
from~$(\mapg,\latentZest)$, so \condref{AD} holds for~$(\maph,\latentZ)$.
Lemma~\ref{lem:mid-shared} applied to $\mapf$ and~$\maph$ (same
latent~$\latentZ$, both satisfying the conditions) then gives $\maph=\mapf$,
i.e.\ $\mapf=\mapg\circ\alpha^{-1}$.
\end{proof}
\vspace{-10pt}

Alternative purely algebraic routes replacing \condref{USB} with orbit conditions appear in Appendix~\ref{app:alt-mid}.

\subsection{Posterior, pointwise, and ICA-form identifiability}

MID immediately yields posterior and pointwise identifiability.

\begin{proposition}[Posterior identifiability]
\label{prop:posterior-id}
Under the conclusions of Theorem~\ref{thm:mid}, for $P_X$-almost
every observation~$x$,
$P_{\mathrm{est}}(\cdot\mid x)
  = \pushfwd{(\alpha^{-1})}{P_{\mathrm{true}}(\cdot\mid x)}$.
\end{proposition}

\begin{remark}[Practical significance]
\label{rem:posterior-practical}
In Bayesian inference or downstream tasks that integrate over latent
variables, the posterior $P(z\mid x)$ is the fundamental object.
Proposition~\ref{prop:posterior-id} guarantees that this posterior is
identified up to~$\alpha$, even when the map is non-injective and the
posterior has support on multiple latent points.
\end{remark}

\begin{proposition}[Pointwise identifiability]
\label{prop:pointwise-id}
If additionally $\mapf$ is globally injective, then for every~$x$ in
its image, $z = \alpha(\hat z)$ where $\hat z := \mapg^{-1}(x)$.
\end{proposition}

Throughout the paper, Proposition~\ref{prop:pointwise-id} is \emph{the only result that uses injectivity} of the map.




\textbf{ICA-form ambiguity under conditional independence.} Under \condref{DD} \NEW{on both sides} (definition in Appendix~\ref{app:ica-conditions}), 
the affine link~$\alpha$ shared by all the preceding results (Theorems~\ref{thm:lid} and~\ref{thm:mid}, Propositions~\ref{prop:posterior-id}--\ref{prop:pointwise-id}) has the ICA form. Alternative sufficient conditions, including the route via iVAE
assumption~(iv), are in Appendix~\ref{app:ica-conditions}. 

\begin{corollary}[PWA-ICA identifiability]
\label{cor:pwa-ica}
Under the assumptions of Theorem~\ref{thm:lid} or \ref{thm:mid}, if both $\latentZ$
and~$\latentZest$ additionally satisfy \condref{DD}, then $\alpha(z) =
P\Lambda z + c$ for a permutation matrix~$P$, a diagonal invertible
matrix~$\Lambda$, and $c\in\R^{\ambientdim}$. 

\end{corollary}

This is \emph{the only result that uses conditional independence}; nowhere else throughout the paper latent independence---conditional or not---appears. \NEWC{iVAE's conditionally factorial prior is \condref{DD} in our setting: it buys the permutation-scaling form of~$\alpha$, not identifiability itself.}\ctag{B10-7}

\subsection{\NEW{Stability and a numerical illustration}}
\label{subsec:stability}
\begin{newedit}
\edtag{E20}Theorem~\ref{thm:mid} assumes exactly equal laws. Its conditions being open \NEWB{in the parameter space}\btag{B5-11}, App.~\ref{app:stability} asks them to hold with a margin $c>0$ within a complexity budget~$\kappa$, which defines a class~$\mathcal{F}_\kappa$ of pairs with whitened latent laws and a ball~$B_R$. Here $W_1$ is the Wasserstein-1 distance.
\begin{proposition}[Stability]
\label{prop:stability}
There is a nondecreasing $\omega=\omega_\kappa$, with $\omega(\varepsilon)\to0$ as $\varepsilon\to0$, such that for all $(\mapf,\latentZ),(\mapg,\latentZest)\in\mathcal{F}_\kappa$: if $W_1(\pushfwd{\mapf}{p_{\latentZ}},\pushfwd{\mapg}{p_{\latentZest}})\le\varepsilon$, then a single orthogonal~$O$ satisfies $\|\mapf-\mapg\circ O^{-1}\|_{L^1(B_R)}\le\omega(\varepsilon)$ and $W_1(p_{\latentZ},\pushfwd{O}{p_{\latentZest}})\le\omega(\varepsilon)$.
\end{proposition}
\NEW{Identifiability thus holds approximately: within~$\mathcal{F}_\kappa$, close observed laws force close decoders and latent laws, up to one orthogonal map.}\edtag{E51} \label{sec:experiments}\edtag{E21}On a one-dimensional family varying only the \MST{} margin, the sample size a likelihood-ratio test needs to tell the true decoder from a competitor grows as the inverse square of that margin, from 29 at unwhitened margin $0.64$ to $29\,211$ at $0.02$ (App.~\ref{app:experiments}).

\textbf{Checking the conditions in a trained model.}\edtag{E23}
No training objective can supply what the true pair lacks, so the data-generating process comes first, with domain knowledge as a guide. On a fitted model, \AD{} is the only contrast required of the estimator (Prop.~\ref{prop:usb-inheritance}), \MST{} has testable sufficient conditions (Lemmas~\ref{lem:cycle-det}, \ref{lem:aff-indep-jst}), \USB{} needs one exposing boundary per branch, probed by sampling boundary crossings, and each returns a margin (App.~\ref{app:stability}). \NEWB{The standard ELBO enforces none of them, and adding the margins as an inductive bias may help learning; we leave this to future work.}\btag{B5-12}
\end{newedit}

\section{Conclusion and prospects}
\label{sec:conclusion}

We have established an identifiability hierarchy for DGMs
with PWA decoders and GMM priors, proceeding from LID
through MID to posterior and pointwise identification,
and finally to ICA-form recovery. Three algebraic contrast
principles---domain (\MST), mechanism (\USB), and interaction
(\AD)---power an engine that replaces both the topological continuity
glue of prior work and \NEWC{the observed labels of the TCL/iVAE line, whose sufficient variability our \MST{} weakens}\ctag{B10-8}.
Structural identifiability (recovering the decoder and the posterior)
is separated from pointwise inversion, which requires injectivity.

\textbf{Genericity.}\edtag{E22}
\condref{AD} and \condref{MST} are \NEW{generic}: violated only on a measure-zero algebraic subset of
the relevant parameter space (Appendix~\ref{app:pi-genericity}, \ref{app:suff-cond}; 
Remark~\ref{rem:mst-generic}). \NEW{\condref{SB}/\condref{USB} are generic} under unconstrained PWA parametrizations (Remarks~\ref{rem:sb-generic},
\ref{rem:usb-generic})\NEW{, and conjectured so under a continuity constraint (App.~\ref{app:winj-nongeneric})}. Concretely: \NEW{under an unconstrained parametrization,} if the branch matrices and biases of PWA map~$\mapf$ and
parameters of GMM~$p_{\latentZ}$ are drawn randomly and independently,
identifiability holds almost surely.
To our knowledge, this is the
first setting in which \emph{all} identifiability assumptions \NEW{are
simultaneously generic}, without requiring observed domain labels,
injectivity, or continuity.

\textbf{Necessity and prospects.}
\edtag{E24}
\condref{SB} is necessary \NEW{for this class}: Appendix~\ref{app:fold} gives a fold example, which can be generalized to a family of examples on $\R^{\ambientdim}$, where \condref{AD} and \condref{MST} hold yet LID breaks.
\condref{MST} is necessary: Appendix~\ref{app:torsion} proves that if a GMM admits any nontrivial affine self-symmetry, it admits a finite-order one; such torsion symmetries can be localised into strict PWA self-maps breaking identifiability. Analogous torsion-based necessity phenomena hold for broad exponential-family and location-scale mixtures (Table~\ref{tab:torsion-beyond-gmm}).
We conjecture that the gap between chamber-wise and global \condref{AD} is structurally necessary in the non-injective setting, and that investigating it leads toward a full necessity result for \condref{AD}.
More broadly, Gaussian mixtures and PWA maps are the first testbed for a \emph{general algebraic identifiability theory}: one can replace the affine group with the natural symmetry group of a general component family~$F$ (e.g., any exponential family), study mixtures from~$F$ and decoder branches from that $F$-symmetry group, and ask whether analogues of \condref{AD}, \condref{MST}, and \condref{USB} again mark \NEWB{the limits of identifiability}\btag{B5-17}. 

\textbf{Causality prospects.}
For causal discovery, \NEW{recovering a structural equation model (SEM) on top of the identified $(\mapf,p_{\latentZ})$ is a further, nontrivial step: the analogue of the gap the Linear Non-Gaussian
Acyclic Model (LiNGAM) \citep{shimizu2006linear} closed between linear ICA and linear SEM.}\edtag{E25} For causal inference, in Intact-VAE \citep{wu2022betaintactvae},
which is based on domain contrast, a central issue is to ensure that
the \emph{mechanisms of two treatment arms} are identified up to the
\emph{same} affine ambiguity---precisely the problem solved by our MID
(Theorem~\ref{thm:mid}).

\begin{newedit}
\textbf{Limitations.}\edtag{E26}
Results are at the population level (exact equality~\eqref{eq:distributional-equality}, the global optimum), as in iVAE and \citet{kivva2022identifiability}. Estimation is not claimed. The results assume equal latent and observation dimensions (a lower latent dimension, with or without noise, is not covered) and invertible branch rules. \SB{}/\USB{} genericity is conjectural under continuity. \AD{} is sufficient, with necessity open. The causal reading is an analogy, identifying no causal graph, intervention or SEM. Proposition~\ref{prop:stability} has no rate and says nothing at margin zero.
\end{newedit}

\section*{Acknowledgements \NEW{and disclosures}}\edtag{E-disc}
I thank my advisor, Prof. Shohei Shimizu, for organizing a seminar on this work, for his support since I joined the lab\NEWD{, and for the freedom to write this paper the way I wanted}\dtag{B12-7}. I thank Thong Pham \NEWB{and Juraj Bodik} for valuable \NEWB{discussions}. \NEWB{I thank Robert A. Vandermeulen for bringing \citet{kivva2022identifiability} to my attention.}\btag{B5-13} \NEWD{I also thank the review committee for valuing purely theoretical work, and the program chairs for the theory-track reviewer guidelines, which let me argue the theory on its own terms, even in the rebuttal.}\dtag{B12-8} This work is currently supported by JST CREST, Grant Number JPMJCR22D2, via my present lab; earlier stages were supported by IBS lab, Kyoto University.
\NEW{The author declares no competing interests.}

\bibliographystyle{plainnat}
\bibliography{neurips_2026}
\newpage


\appendix
\section{Proofs of main results}
\label{app:proofs}

\subsection*{Lemma~\ref{lem:single-chamber-affine} (Single-chamber PWA is globally affine)}

\begin{lemma}
\label{lem:single-chamber-affine}
Let $h\colon\R^{\ambientdim}\to\R^{\ambientdim}$ be a reduced PWA map with invertible affine branch rules. If $h$ has exactly one chamber, then $h$ has exactly one branch, hence $h(x)=Ax+b$ globally, up to the harmless measure-zero boundary convention for PWA maps.
\end{lemma}

\begin{proof}
Write the branch rules as
$$
  h_i(z)=A_i z+b_i,
  \qquad A_i\in\GL(\ambientdim),
$$
and write
$$
  U_i:=h_i(\branchgeneric_i)
$$
for the image of branch~$i$. Since each~$\branchgeneric_i$ is a finite union of full-dimensional polyhedra and $h_i$ is an affine bijection, each~$U_i$ is again a finite union of full-dimensional polyhedra.

Let
$$
  \mathcal B:=\bigcup_i \partial U_i .
$$
The set~$\mathcal B$ is contained in a finite union of lower-dimensional polyhedral pieces, hence has Lebesgue measure zero. On every connected component of
$$
  \R^{\ambientdim}\setminus \mathcal B,
$$
membership in each~$U_i$ is constant, and therefore the active set
$$
  \Ichamber{h}{x}=\{i:x\in U_i\}
$$
is constant. Thus every such connected component is contained in a chamber.

By assumption there is exactly one chamber; call it~$C$. Hence all points outside~$C$ are contained in~$\mathcal B$. In particular,
$$
  \lambda_{\ambientdim}(\R^{\ambientdim}\setminus C)=0,
$$
so $C$ has full Lebesgue measure.

We now show that every branch is active on~$C$. Fix a branch~$i$. Since~$\branchgeneric_i$ contains a full-dimensional polyhedral piece and $h_i$ is invertible affine, $U_i$ has nonempty interior. Therefore $U_i\setminus\mathcal B$ is nonempty. Pick
$$
  y\in U_i\setminus\mathcal B .
$$
Then $y$ lies in some chamber, and since there is only one chamber, $y\in C$. Also $y\in U_i$, so branch~$i$ is active at~$y$. Since the active set is constant on~$C$, branch~$i$ is active at every point of~$C$. Equivalently,
$$
  C\subset U_i=h_i(\branchgeneric_i).
$$

Because $h_i$ is an affine bijection,
$$
  h_i^{-1}(C)\subset \branchgeneric_i .
$$
Since $C$ has full measure and affine bijections preserve null sets, $h_i^{-1}(C)$ also has full measure. Hence~$\branchgeneric_i$ has full measure.

Thus every branch domain~$\branchgeneric_i$ has full measure. If there were two distinct branches $i\neq j$, then $\branchgeneric_i\cap \branchgeneric_j$ would have full measure, contradicting the assumption that the branch domains form a partition up to measure-zero boundaries. Therefore there is exactly one branch.

With only one branch, the PWA map is given by a single invertible affine rule
$$
  h(x)=Ax+b
$$
on the whole domain, up to the standard measure-zero boundary convention. Enlarging the unique branch domain by the remaining boundary points if necessary gives the stated global affine representative.
\end{proof}

\subsection*{Lemma~\ref{lem:branch-count} (Active-branch-count equality)}

\begin{proof}
Fix a fine chamber~$\fineC$. On~$\fineC$,
equation~\eqref{eq:chamber-equality} gives
an identity between two finite Gaussian mixtures.
Each side is a finite sum of globally-defined
Gaussian densities, hence a real-analytic
function on~$\R^{\ambientdim}$.
Since the two analytic functions agree on the
nonempty open set~$\fineC$, they agree on all
of~$\R^{\ambientdim}$.

Now integrate both sides over~$\R^{\ambientdim}$.
Each Gaussian density integrates to~$1$, so
every active branch contributes exactly one full
mixture mass budget:
\begin{align*}
  \text{LHS}
  &= \sum_{i\in\Ichamber{\mapf}{\fineC}}
       \sum_{k=1}^{\numcomptrue}\wtrue{k}
   = \sum_{i\in\Ichamber{\mapf}{\fineC}} 1
   = |\Ichamber{\mapf}{\fineC}|, \\
  \text{RHS}
  &= \sum_{j\in\Ichamber{\mapg}{\fineC}}
       \sum_{\ell=1}^{\numcompest}\west{\ell}
   = \sum_{j\in\Ichamber{\mapg}{\fineC}} 1
   = |\Ichamber{\mapg}{\fineC}|.
\end{align*}
Therefore
\[
  |\Ichamber{\mapf}{\fineC}|
  =
  |\Ichamber{\mapg}{\fineC}|.
\]
\end{proof}

\subsection*{Lemma~\ref{lem:global-param-matching} (Global parameter bijection under (AD))}

\begin{proof}
\textbf{Step 1: Chamber-wise matching.}
Fix a fine chamber~$\fineC$.
Equation~\eqref{eq:chamber-equality}
equates two finite Gaussian mixtures on the
open set~$\fineC$.
By Remark~\ref{rem:open-set}, this equality
extends to all of~$\R^{\ambientdim}$.

Under \condref{AD} on each side, the parameter
families
\[
  \{\paramf{i}{k}\}_{i\in\Ichamber{\mapf}{\fineC},\,k}
  \qquad\text{and}\qquad
  \{\paramEst{j}{\ell}\}_{j\in\Ichamber{\mapg}{\fineC},\,\ell}
\]
are pairwise distinct within their respective
sides.
Hence the Gaussian-independence lemma applies to
the difference of the two mixtures:
every Gaussian parameter appearing on one side
must also appear on the other, and matched terms
must carry the same coefficient.
Because all coefficients are strictly positive
(by irredundancy of the GMMs), this gives a
bijection
\[
  \Psi_{\fineC}\colon
  \Ichamber{\mapg}{\fineC}\times\IndexEst
  \;\longleftrightarrow\;
  \Ichamber{\mapf}{\fineC}\times\IndexTrue
\]
such that
\[
  \paramEst{j}{\ell}=\paramf{i}{k},
  \qquad
  \west{\ell}=\wtrue{k},
\]
whenever
$\Psi_{\fineC}(j,\ell)=(i,k)$.

\medskip
\textbf{Step 2: Consistency across chambers.}
We now show these chamber-wise bijections glue
to a single global bijection.

Take any estimator pair $(j,\ell)\in\Ig\times\IndexEst$.
Choose a fine chamber~$\fineC$ with
$j\in\Ichamber{\mapg}{\fineC}$.
Such a chamber exists because~$\branchg{j}$
contains a full-dimensional polyhedron, and the
invertible affine map~$\mapg_j$ sends it to a
set with nonempty interior; any interior point
lies in some fine chamber with branch~$j$
active.
Define
\[
  \Psi(j,\ell):=\Psi_{\fineC}(j,\ell).
\]

If~$\fineCp$ is another fine chamber with
$j\in\Ichamber{\mapg}{\fineCp}$, then
$\Psi_{\fineC}(j,\ell)$ and
$\Psi_{\fineCp}(j,\ell)$ are both truth-side
pairs having parameter~$\paramEst{j}{\ell}$.
By \condref{AD} on the truth side, there is
\emph{at most one} truth pair with that
parameter. Hence
\[
  \Psi_{\fineC}(j,\ell)=\Psi_{\fineCp}(j,\ell),
\]
so the definition of~$\Psi(j,\ell)$ is
independent of the chosen chamber.

The same uniqueness argument on the estimator
side shows that the inverse matching is also
well-defined. Therefore~$\Psi$ is a global
bijection
\[
  \Psi\colon
  \Ig\times\IndexEst
  \;\longleftrightarrow\;
  \If\times\IndexTrue
\]
with the stated parameter and weight matching
property, and whose restriction to each fine
chamber is exactly~$\Psi_{\fineC}$.
\end{proof}

\subsection*{Lemma~\ref{lem:comp-count} (Component-count equality)}

\begin{lemma}[Component-count equality]
\label{lem:comp-count}
$\numcompest = \numcomptrue$.
\end{lemma}

\begin{proof}
Choose a fine chamber~$\fineC$ with nonempty
active set; such a chamber exists because each
branch image has nonempty interior.
On~$\fineC$, the chamber-wise bijection
$\Psi_{\fineC}$ matches
\[
  |\Ichamber{\mapf}{\fineC}|\cdot\numcomptrue
\]
truth-side pairs to
\[
  |\Ichamber{\mapg}{\fineC}|\cdot\numcompest
\]
estimator-side pairs.
Hence
\[
  |\Ichamber{\mapf}{\fineC}|\cdot\numcomptrue
  =
  |\Ichamber{\mapg}{\fineC}|\cdot\numcompest.
\]
By Lemma~\ref{lem:branch-count},
\[
  |\Ichamber{\mapf}{\fineC}|
  =
  |\Ichamber{\mapg}{\fineC}|
  \ge 1.
\]
Dividing by this common positive number gives
\[
  \numcomptrue=\numcompest.
\]
\end{proof}

\subsection*{Lemma~\ref{lem:boundary-counting} (Boundary branch-count equality)}

\begin{lemma}[Boundary branch-count equality]
\label{lem:boundary-counting}
At every pair of adjacent fine chambers,
$\Df(\fineC,\fineCp) = \Dg(\fineC,\fineCp)$
and $\Ef(\fineC,\fineCp) = \Eg(\fineC,\fineCp)$.
\end{lemma}

\begin{proof}
Let~$\fineC,\fineCp$ be adjacent fine chambers.
Define the leaving blocks
\[
  L_{\mapg}
  :=
  \bigl(\Ichamber{\mapg}{\fineC}
    \setminus\Ichamber{\mapg}{\fineCp}\bigr)
  \times[K],
  \qquad
  L_{\mapf}
  :=
  \bigl(\Ichamber{\mapf}{\fineC}
    \setminus\Ichamber{\mapf}{\fineCp}\bigr)
  \times[K].
\]

We claim that~$\Psi$ restricts to a bijection
$L_{\mapg}\longleftrightarrow L_{\mapf}$.
Take any $(j,\ell)\in L_{\mapg}$.
Then~$j$ is active on~$\fineC$ but not on~$\fineCp$,
so the Gaussian with parameter
$\paramEst{j}{\ell}$ appears in the chamber
identity for~$\fineC$ but not in the chamber
identity for~$\fineCp$.
Let
\[
  \Psi(j,\ell)=(i,k).
\]
Then
$\paramf{i}{k}=\paramEst{j}{\ell}$.
Since $(j,\ell)$ is active on~$\fineC$,
the pair~$(i,k)$ must be active on~$\fineC$ as
well, by the chamber-wise matching on~$\fineC$.

If $(i,k)$ were also active on~$\fineCp$, then
the same Gaussian parameter would appear on the
truth side for~$\fineCp$.
But on~$\fineCp$, chamber-wise matching would
then force the corresponding estimator pair with
that parameter to be active there too.
By global uniqueness under \condref{AD}, that
estimator pair must be~$(j,\ell)$, contradicting
$(j,\ell)\notin \Ichamber{\mapg}{\fineCp}\times[K]$.
Hence $(i,k)\in L_{\mapf}$.

Thus $\Psi(L_{\mapg})\subseteq L_{\mapf}$.
Applying the same argument with the roles of
$\mapf$ and~$\mapg$ reversed gives the reverse
inclusion, so $\Psi$ indeed restricts to a
bijection
\[
  L_{\mapg}\longleftrightarrow L_{\mapf}.
\]

Now count elements:
\[
  |L_{\mapg}|=\Dg\cdot K,
  \qquad
  |L_{\mapf}|=\Df\cdot K.
\]
Since the two sets are in bijection,
\[
  \Dg\cdot K=\Df\cdot K.
\]
Because $K\ge 1$, we obtain $\Dg=\Df$.
The entering-count identity
$\Eg=\Ef$ is proved in exactly the same way,
after swapping~$\fineC$ and~$\fineCp$.
\end{proof}

\subsection*{Lemma~\ref{lem:chamber-coincidence} (Chamber coincidence)}

\begin{lemma}[Chamber coincidence]
\label{lem:chamber-coincidence}
Under double \condref{AD}, \NEW{let $\fineC,\fineCp$ be adjacent fine chambers with common boundary~$F$. Then $\Ichamber{\mapf}{\cdot}$ changes across~$F$ if and only if $\Ichamber{\mapg}{\cdot}$ does. Hence} $\mapf$-chambers, $\mapg$-chambers, and
fine chambers all coincide \NEW{and have the same boundaries}.\edtag{E36}
\end{lemma}

\begin{proof}
\NEW{The boundary~$F$} is, by definition, a
locus where at least one of the active sets
changes.
Equivalently, across \NEW{$F$} either
$\Df+\Ef>0$ or $\Dg+\Eg>0$ (or both).
Lemma~\ref{lem:boundary-counting} gives
\[
  \Df=\Dg,
  \qquad
  \Ef=\Eg,
\]
hence
\[
  \Df+\Ef>0
  \iff
  \Dg+\Eg>0.
\]
So a fine-chamber \NEW{boundary} is a \NEW{boundary} for
$\mapf$ if and only if it is one
for~$\mapg$.

Now let~$B$ be any \NEW{boundary of an} $\mapf$-chamber.
By definition, crossing~$B$ changes
$\Ichamber{\mapf}{\cdot}$, so~$B$ is also a
fine-chamber \NEW{boundary}. By the equivalence above,
crossing~$B$ also changes~$\Ichamber{\mapg}{\cdot}$,
so~$B$ is a \NEW{boundary of a} $\mapg$\NEW{-chamber}.
The same argument with the roles reversed shows
that every \NEW{boundary of a} $\mapg$\NEW{-chamber} is also \NEW{one of an}
$\mapf$\NEW{-chamber}.

Therefore the $\mapf$-chambers, the
$\mapg$-chambers, and the fine chambers all
coincide. \NEW{Equalities of branch domains and of unions of chambers are up to Lebesgue-null sets (Definitions~\ref{def:reduced-pwa} and~\ref{def:pwa-equality}).}
\end{proof}

\subsection*{Theorem~\ref{thm:global-affine} (From PWA to affine: the symmetry-rigidity)}

\begin{proof}
Take the truth map to be $\mapf=\Id$.
Because $\Id$ has a single chamber, $(\Id,\latentZ)$ trivially satisfies
\condref{AD}, so double \condref{AD} holds.

The identity map has exactly one chamber, namely
all of~$\R^{\ambientdim}$, and its active set on
that chamber has size~$1$.
By Lemma~\ref{lem:chamber-coincidence},
$\mapg$ also has a single chamber.
By Lemma~\ref{lem:branch-count},
the active set of~$\mapg$ on that chamber also
has size~$1$.

Since every branch of a reduced PWA map is
active on some open subset of its image, a map
with only one chamber and active-set size~$1$
must in fact have only one branch in total
(Lemma~\ref{lem:single-chamber-affine}).
A one-branch PWA map is globally affine, so
\[
  \mapg(x)=Ax+b
\]
for some $A\in\GL(\ambientdim)$ and
$b\in\R^{\ambientdim}$.

The ``in particular'' clause (no strict PWA self-maps) follows by taking $\mapf=\Id$
and $\latentZ\sim p$, so that the observable law
$\mapg(\latentZ)\deq\latentZ$ is itself a GMM.
\end{proof}

\subsection*{Theorem~\ref{thm:lid} (Law identifiability)}

\begin{proof}
Lemma~\ref{lem:comp-count} gives
\[
  \numcompest=\numcomptrue=:K.
\]

\textbf{Case 1: $\mapf$ has a single chamber.}
Then $\mapf$ is globally affine.
Write
\[
  \mapf(x)=Ax+b,
  \qquad A\in\GL(\ambientdim).
\]
Hence $\mapf(\latentZ)$ is itself a GMM.
Since
\[
  \mapg(\latentZest)\deq \mapf(\latentZ)
\]
and $(\mapg,\latentZest)$ satisfies
\condref{AD}, Theorem~\ref{thm:global-affine}
implies that~$\mapg$ is also globally affine,
say
\[
  \mapg(x)=Bx+c.
\]
Define
\[
  \alpha:=\mapf^{-1}\circ\mapg.
\]
Then~$\alpha$ is an affine bijection, and
\[
  \alpha(\latentZest)
  = (\mapf^{-1}\circ\mapg)(\latentZest)
  = \mapf^{-1}\bigl(\mapg(\latentZest)\bigr)
  \deq
    \mapf^{-1}\bigl(\mapf(\latentZ)\bigr)
  = \latentZ.
\]
So LID holds in this case.

\medskip
\textbf{Case 2: $\mapf$ has more than one chamber.}
By Lemma~\ref{lem:chamber-coincidence},
the $\mapf$-chambers are exactly the fine
chambers.
Thus an \condref{SB} boundary for~$\mapf$ is a
fine-chamber boundary.

By \condref{SB}, there exist adjacent chambers
$\fineC,\fineCp$ and a unique toggling truth
branch.
Swapping $\fineC$ and~$\fineCp$ if necessary,
we may assume
\[
  \Ichamber{\mapf}{\fineC}
  \setminus
  \Ichamber{\mapf}{\fineCp}
  =
  \{i_\star\},
\]
so $\Df=1$.
Lemma~\ref{lem:boundary-counting} then gives
$\Dg=1$, so there is a unique estimator branch
$j_\star$ leaving across the same boundary.

The global bijection~$\Psi$ restricts to a
bijection between the leaving blocks
\[
  \{j_\star\}\times[K]
  \;\longleftrightarrow\;
  \{i_\star\}\times[K].
\]
Since both sets have size~$K$, this restricted
bijection is equivalent to a permutation
$\sigma\colon[K]\to[K]$ such that, for every
$\ell\in[K]$,
\[
  \paramEst{j_\star}{\ell}
  =
  \paramf{i_\star}{\sigma(\ell)},
  \qquad
  \west{\ell}
  =
  \wtrue{\sigma(\ell)}.
\]

Unpacking the equality of Gaussian parameters,
we get
\begin{align*}
  \ling{j_\star}\muest{\ell}+\biasg{j_\star}
  &=
  \linf{i_\star}\mutrue{\sigma(\ell)}
  +\biasf{i_\star}, \\
  \ling{j_\star}\Sigmaest{\ell}\ling{j_\star}^{\!\top}
  &=
  \linf{i_\star}\Sigmatrue{\sigma(\ell)}
  \linf{i_\star}^{\!\top}.
\end{align*}
Because both branch matrices are invertible,
the map
\[
  \alpha(z)
  :=
  \linf{i_\star}^{-1}
  \bigl(
    \ling{j_\star}z
    +\biasg{j_\star}
    -\biasf{i_\star}
  \bigr)
\]
is an affine bijection.

The two displayed identities say exactly that
$\alpha$ sends the estimator component
$\Gauss(\muest{\ell},\Sigmaest{\ell})$
to the truth component
$\Gauss(\mutrue{\sigma(\ell)},
       \Sigmatrue{\sigma(\ell)})$,
and the matched weights are equal.
Therefore
\[
  \alpha(\latentZest)
  =
  \sum_{\ell=1}^{K}
  \west{\ell}\,
  \Gauss(\mutrue{\sigma(\ell)},
         \Sigmatrue{\sigma(\ell)})
  \deq
  \latentZ.
\]
This proves LID.
\end{proof}

\subsection*{Proposition~\ref{prop:usb-inheritance} ((\condref{SB}) and (\condref{USB}) inheritance)}

\begin{proof}
\textbf{(\condref{SB}) case.}
If $\mapf$ has a single chamber, Lemma~\ref{lem:chamber-coincidence} gives the same for $\mapg$, and clause~(a) applies.
If $\mapf$ is strict PWA and satisfies \condref{SB}, there exist adjacent chambers with a unique truth toggling branch (say leaving); Lemma~\ref{lem:boundary-counting} gives $\Dg=1$, so $\mapg$ also has a simple boundary. Hence $\mapg$ satisfies \condref{SB}.

\medskip
\textbf{(\condref{USB}) case.}
If $\mapf$ has a single chamber, then by
Lemma~\ref{lem:chamber-coincidence} so does
$\mapg$, and clause~(a) of \condref{USB} holds
for~$\mapg$ automatically.

Assume now that $\mapf$ is strict PWA.
Fix any truth branch $i\in\If$.
Because~$\mapf$ satisfies \condref{USB}, there
exist adjacent chambers where~$i$ is the unique
leaving branch or the unique entering branch.
Swapping the chambers if needed, assume~$i$ is
the unique leaving branch.
Then $\Df=1$, and
Lemma~\ref{lem:boundary-counting} gives
$\Dg=1$.
So there is a unique estimator branch, call it
$j(i)$, that toggles in the same direction
across that boundary.

We claim that the map $i\mapsto j(i)$ is
injective.
Suppose, toward contradiction, that
\[
  j(i_1)=j(i_2)=j_\star
\]
for two distinct truth branches
$i_1\neq i_2$.
At the boundary \NEW{exposing}~$i_1$, the leaving
block argument shows that~$\Psi$ restricts to a
bijection
\[
  \{j_\star\}\times[K]
  \longleftrightarrow
  \{i_1\}\times[K].
\]
At the boundary \NEW{exposing}~$i_2$, the same
global~$\Psi$ restricts to
\[
  \{j_\star\}\times[K]
  \longleftrightarrow
  \{i_2\}\times[K].
\]
But for any fixed~$\ell$, the pair
$\Psi(j_\star,\ell)$ is single-valued.
Its first coordinate therefore cannot be both
$i_1$ and~$i_2$.
Contradiction.

Thus $i\mapsto j(i)$ is injective.
Since $|\If|=|\Ig|$, an injection between these
finite sets is automatically a bijection.
So every estimator branch appears as some
$j(i)$ and is \NEW{exposed} at its own simple
boundary. Therefore~$\mapg$ satisfies
\condref{USB}.
\end{proof}

\subsection*{Proposition~\ref{prop:mst-conjugacy} ((MST) conjugacy)}

\begin{proposition}[(MST) conjugacy]
\label{prop:mst-conjugacy}
For any affine bijection~$\alpha$ with
$\alpha(\latentZest)\deq\latentZ$,
$\Gmix{p_{\latentZest}} = \alpha^{-1}\,\Gmix{p_{\latentZ}}\,\alpha$.
In particular, $p_{\latentZ}$ satisfies \condref{MST} if and only
if~$p_{\latentZest}$ does.
\end{proposition}

\begin{proof}
Let $T\in\Aff(\R^{\ambientdim})$.
Then
\[
  T_\sharp p_{\latentZest}=p_{\latentZest}
\]
if and only if
\[
  \alpha_\sharp T_\sharp p_{\latentZest}
  =
  \alpha_\sharp p_{\latentZest}.
\]
Using $\alpha_\sharp p_{\latentZest}=p_{\latentZ}$,
this is equivalent to
\[
  (\alpha T\alpha^{-1})_\sharp p_{\latentZ}
  =
  p_{\latentZ}.
\]
So
\[
  T\in\Gmix{p_{\latentZest}}
  \iff
  \alpha T\alpha^{-1}\in\Gmix{p_{\latentZ}}.
\]
Hence
\[
  \Gmix{p_{\latentZest}}
  =
  \alpha^{-1}\,\Gmix{p_{\latentZ}}\,\alpha.
\]
In particular, one symmetry group is trivial if
and only if the other is.
\end{proof}

\subsection*{Lemma~\ref{lem:mid-shared} (MID with shared latent)}

\begin{proof}
From Section~\ref{subsec:param-matching} we
already know:
\[
  K=K',
  \qquad
  |\If|=|\Ig|,
\]
the leaving and entering counts agree at every
boundary (Lemma~\ref{lem:boundary-counting}),
and the chamber decompositions coincide
(Lemma~\ref{lem:chamber-coincidence}).

\medskip
\textbf{Step 1: Match each truth branch to one
estimator branch.}
Fix any truth branch $i\in\If$.

If $\mapf$ has a single chamber, then by chamber
coincidence so does~$\mapg$.
Hence both maps are globally affine.
Then
\[
  S:=\mapf^{-1}\circ\mapg
\]
is affine and satisfies
\[
  \pushfwd{S}{p}
  =
  \pushfwd{\mapf^{-1}}{\pushfwd{\mapg}{p}}
  \deq
  \pushfwd{\mapf^{-1}}{\pushfwd{\mapf}{p}}
  = p.
\]
So $S\in\Gmix{p}$, and \condref{MST} gives
$S=\Id$, hence $\mapf=\mapg$.
Thus the single-chamber case is done.

Assume from now on that $\mapf$ is strict PWA.
By \condref{USB}, there exist adjacent chambers
$\fineC,\fineCp$ such that~$i$ is the unique
leaving or entering truth branch.
Swapping the chambers if needed, we may assume
$i$ is the unique leaving branch, so $\Df=1$.
Lemma~\ref{lem:boundary-counting} gives
$\Dg=1$; let $j(i)\in\Ig$ be the unique
estimator branch that leaves across the same
boundary.

The restricted bijection between leaving blocks
\[
  \Psi\colon
  \{j(i)\}\times[K]
  \;\longleftrightarrow\;
  \{i\}\times[K]
\]
induces a permutation $\sigma_i\colon[K]\to[K]$
such that
\begin{equation}
  \label{eq:usb-match}
  \paramEst{j(i)}{\ell}
  =
  \paramf{i}{\sigma_i(\ell)},
  \qquad
  w_\ell = w_{\sigma_i(\ell)}
\end{equation}
for all $\ell\in[K]$.

\medskip
\textbf{Step 2: The branch difference is a
mixture symmetry.}
Define
\[
  S_i:=\mapf_i^{-1}\circ\mapg_{j(i)}.
\]
Because the latent law is the same on both
sides, the parameter identity
\eqref{eq:usb-match} says that~$S_i$ sends the
$\ell$-th Gaussian component to the
$\sigma_i(\ell)$-th one:
\[
  \pushfwd{S_i}{\Gauss_\ell}
  =
  \Gauss_{\sigma_i(\ell)}.
\]
The weight identity in~\eqref{eq:usb-match}
shows that this permutation preserves mixture
weights. Therefore
\[
  \pushfwd{S_i}{p}
  =
  \sum_{\ell=1}^K
    w_\ell\,\Gauss_{\sigma_i(\ell)}
  =
  \sum_{m=1}^K
    w_m\,\Gauss_m
  =
  p.
\]
Hence $S_i\in\Gmix{p}$.

\medskip
\textbf{Step 3: Kill the ambiguity with
\condref{MST}.}
By \condref{MST},
\[
  \Gmix{p}=\{\Id\}.
\]
So $S_i=\Id$, which means
\[
  \mapf_i=\mapg_{j(i)}
\]
for every truth branch~$i$.

\medskip
\textbf{Step 4: Convert this into equality of
reduced PWA maps.}
The correspondence $i\mapsto j(i)$ is
well-defined.
It is injective by exactly the same argument as
in Proposition~\ref{prop:usb-inheritance}: one
estimator branch cannot be the unique toggling
partner of two different truth branches without
making~$\Psi$ multivalued on some
$(j,\ell)$.
Since $|\If|=|\Ig|$, this injection is a
bijection; write it as
\[
  \pi^{-1}\colon\If\to\Ig,
  \qquad
  \pi(j(i))=i.
\]

It remains to compare branch domains.
Let $i\in\If$.
For any fine chamber~$\fineC$, active-set
agreement together with the bijection~$\pi$
gives
\[
  i\in\Ichamber{\mapf}{\fineC}
  \iff
  \pi^{-1}(i)\in\Ichamber{\mapg}{\fineC}.
\]
By definition of active set, this means
\[
  \fineC\subset \mapf_i(\branchf{i})
  \iff
  \fineC\subset \mapg_{\pi^{-1}(i)}
    (\branchg{\pi^{-1}(i)}).
\]
\NEW{Branch domains may be open, closed or mixed, and chambers are open (Definition~\ref{def:chambers}). Equalities of branch domains and of unions of chambers are up to Lebesgue-null sets (Definitions~\ref{def:reduced-pwa} and~\ref{def:pwa-equality}).}\edtag{E37}
Taking the union over all such fine chambers
recovers the interiors of the two branch images \NEW{up to a Lebesgue-null set}:
\[
  \mapf_i(\branchf{i})^\circ
  =
  \mapg_{\pi^{-1}(i)}
    (\branchg{\pi^{-1}(i)})^\circ.
\]
But Step~3 gave
$\mapf_i=\mapg_{\pi^{-1}(i)}$, and this common
map is an invertible affine bijection.
Applying its inverse, which preserves interiors,
yields
\[
  \branchf{i}^\circ
  =
  \branchg{\pi^{-1}(i)}^\circ.
\]

Finally, each branch domain is a finite union of
full-dimensional polyhedra, hence \NEW{differs from its interior by a Lebesgue-null set}.
Therefore
\[
  \branchf{i}
  =
  \branchg{\pi^{-1}(i)}
\]
up to measure-zero boundaries.
By Definition~\ref{def:pwa-equality},
$\mapf=\mapg$.
\end{proof}

\subsection*{Theorem~\ref{thm:mid} (Map identifiability)}

\begin{proof}
\textbf{Phase I: Obtain the latent affine link.}
Since \condref{USB} implies \condref{SB},
Theorem~\ref{thm:lid} yields an affine bijection
$\alpha$ and a permutation $\sigma$ of~$[K]$
such that
\[
  \alpha(\latentZest)\deq\latentZ
\]
and, componentwise,
\begin{equation}
\label{eq:lid-sigma}
  \alpha_\sharp\Gauss(\muest{\ell},\Sigmaest{\ell})
  =\Gauss(\mutrue{\sigma(\ell)},
         \Sigmatrue{\sigma(\ell)}),
  \qquad
  \west{\ell}=\wtrue{\sigma(\ell)}.
\end{equation}

\medskip
\textbf{Phase II: Reduce to a shared latent law.}
Define
\[
  \maph:=\mapg\circ\alpha^{-1}.
\]
Since $\alpha(\latentZest)\deq\latentZ$,
equivalently
$\alpha^{-1}(\latentZ)\deq\latentZest$,
we get
\[
  \maph(\latentZ)
  =
  (\mapg\circ\alpha^{-1})(\latentZ)
  =
  \mapg\bigl(\alpha^{-1}(\latentZ)\bigr)
  \deq
  \mapg(\latentZest)
  \deq
  \mapf(\latentZ).
\]
So $\mapf$ and~$\maph$ push forward the
\emph{same} latent GMM~$\latentZ$ to the same
observable law.

We next check that $(\maph,\latentZ)$ still
satisfies \condref{AD}.
From~\eqref{eq:lid-sigma},
\[
  (\alpha^{-1})_\sharp\Gauss(\mutrue{k},\Sigmatrue{k})
  =
  \Gauss(\muest{\sigma^{-1}(k)},
         \Sigmaest{\sigma^{-1}(k)}).
\]
Therefore, for branch~$j$ of~$\maph$ and
component~$k$ of~$\latentZ$,
the pushforward parameters satisfy
\begin{equation}
  \label{eq:h-param}
  \paramh{j}{k}
  =
  \paramEst{j}{\sigma^{-1}(k)}.
\end{equation}
Now suppose
$\paramh{j}{k}=\paramh{m}{n}$.
Using~\eqref{eq:h-param},
\[
  \paramEst{j}{\sigma^{-1}(k)}
  =
  \paramEst{m}{\sigma^{-1}(n)}.
\]
Since $(\mapg,\latentZest)$ satisfies
\condref{AD}, we get
\[
  (j,\sigma^{-1}(k))
  =
  (m,\sigma^{-1}(n)),
\]
hence $(j,k)=(m,n)$.
So $(\maph,\latentZ)$ satisfies \condref{AD}.

Also, composing a reduced PWA map with a global
affine bijection preserves the reduced PWA
structure: branch domains remain finite unions
of full-dimensional polyhedra, invertibility of
branch rules is unchanged, and distinct affine
rules remain distinct after composing on the
right with the same affine bijection.
Thus $(\maph,\latentZ)$ is again a reduced PWA
pair.

\medskip
\textbf{Phase III: Apply the shared-latent MID
lemma.}
We can now apply Lemma~\ref{lem:mid-shared} to
the pair $(\mapf,\maph)$ with common latent
law~$\latentZ$.
The truth side $(\mapf,\latentZ)$ still
satisfies \condref{USB} and \condref{MST},
and the estimator side $(\maph,\latentZ)$
satisfies \condref{AD}.
Therefore
\[
  \maph=\mapf.
\]
Recalling $\maph=\mapg\circ\alpha^{-1}$, we
obtain
\[
  \mapf=\mapg\circ\alpha^{-1}.
\]
\end{proof}

\subsection*{Proposition~\ref{prop:posterior-id} (Posterior identifiability)}

\begin{proof}
We argue at the level of regular conditional distributions (posterior
kernels), not Lebesgue densities: for a deterministic map the conditional
law of the latent given~$\RVobs$ is typically supported on a finite fibre
rather than on a full-dimensional set.

The subtlety is that LID/MID give only equality \emph{in law},
$p_{\latentZest}=\pushfwd{(\alpha^{-1})}{p_{\latentZ}}$, not a coupling of
the estimator latent~$\latentZest$ with the truth latent~$\latentZ$ on a
common space. We therefore build a \emph{surrogate} estimator latent from
the truth: on the truth space set
\[
  \widetilde{\latentZest}:=\alpha^{-1}(\latentZ),
  \qquad
  \RVobs:=\mapf(\latentZ).
\]
By the latent-law identity, $\widetilde{\latentZest}\deq\latentZest$, and by
Theorem~\ref{thm:mid} ($\mapf=\mapg\circ\alpha^{-1}$, up to the null-boundary
convention),
\[
  \RVobs=\mapf(\latentZ)=\mapg\bigl(\alpha^{-1}(\latentZ)\bigr)
  =\mapg(\widetilde{\latentZest})
  \qquad\text{almost surely.}
\]
Thus $(\widetilde{\latentZest},\RVobs)$ is generated by drawing
$\widetilde{\latentZest}$ from the law of~$\latentZest$ and outputting
$\RVobs=\mapg(\widetilde{\latentZest})$---which is \emph{exactly} how the
estimator model generates its latent--observation pair
$(\latentZest,\mapg(\latentZest))$. Hence the two pairs share the same joint
law, so the estimator posterior $P_{\mathrm{est}}(\cdot\mid x)$ (the
conditional law of~$\latentZest$ given its observation) equals the
conditional law of~$\widetilde{\latentZest}$ given $\RVobs=x$, for
$P_X$-almost every~$x$.

It remains to identify the latter. Since
$\widetilde{\latentZest}=\alpha^{-1}(\latentZ)$ is a deterministic function
of~$\latentZ$, for every Borel set $B\subseteq\R^{\ambientdim}$ we have
$\mathbf{1}_{\{\widetilde{\latentZest}\in B\}}
=\mathbf{1}_{\{\latentZ\in\alpha(B)\}}$ almost surely. Taking conditional
expectations given~$\RVobs$ (with measurable versions of the regular
conditional laws),
\[
  \underbrace{P(\widetilde{\latentZest}\in B\mid \RVobs=x)}_{=\;P_{\mathrm{est}}(\latentZest\in B\mid x)}
  =
  \underbrace{P(\latentZ\in\alpha(B)\mid \RVobs=x)}_{=\;P_{\mathrm{true}}(\latentZ\in\alpha(B)\mid x)}
\]
for $P_X$-almost every~$x$, which is exactly the conclusion of Proposition~\ref{prop:posterior-id}.
\end{proof}

\begin{remark}
The only measure-theoretic subtlety is that posteriors are conditional
probability kernels, hence defined only up to $P_X$-null sets. The proof
above avoids formal disintegration language, but the underlying reason it
works is simple: once two latent--observation pairs have the same joint
distribution, they have the same posterior kernels, up to $P_X$-almost-everywhere
equality.
\end{remark}

\subsection*{Proposition~\ref{prop:pointwise-id} (Pointwise identifiability)}

\begin{proof}
Injectivity of~$\mapf$ implies that for every
$x$ in its image, the fiber
$\preimage{\mapf}{x}$ is a singleton~$\{z\}$.
Since
\[
  \mapf=\mapg\circ\alpha^{-1}
  \qquad\Longleftrightarrow\qquad
  \mapg=\mapf\circ\alpha,
\]
and $\alpha$ is a bijection, the map~$\mapg$ is
also injective.
Hence
\[
  \preimage{\mapg}{x}
  =
  \{\alpha^{-1}(z)\}.
\]
So if $\hat z:=\mapg^{-1}(x)$, then
\[
  \hat z=\alpha^{-1}(z),
\]
equivalently
\[
  z=\alpha(\hat z).
\]
\end{proof}

\subsection*{Corollary~\ref{cor:pwa-ica} (PWA-ICA identifiability)}

See Appendix~\ref{app:ica-conditions} for the full proof.

\section{Terminology}
\label{app:terminology}

\begin{table}[h]
\caption{Compact terminology mapping (intuitive $\leftrightarrow$ formal).
  Used consistently throughout the paper.\edtag{E4}}
\label{tab:terminology}
\centering\small
\begin{tabular}{@{}l l p{5.1cm}@{}}
\toprule
\textbf{Intuitive term} & \textbf{Formal object} & \textbf{Notes} \\
\midrule
law       & probability \NEW{distribution}, e.g.\ $p_{\latentZ}$
          & \emph{latent law} for $p_{\latentZ}$;
            pushforward law for $\pushfwd{\mapf}{p_{\latentZ}}$ \\
map       & \NEW{decoder/generator} $\mapf$ (reduced PWA throughout)
          & also called mixing function \\
domain    & mixture component of latent GMM
          & not function domain; related to domain adaptation  \\
mechanism & branch (affine rule) of reduced PWA map
          & each domain is acted on by every mechanism
            whose region it intersects \\
boundary  & branch / chamber boundary
          & locus where the active branch set changes \\
contrast  & algebraic genericity / non-symmetry
          & instances: \MST{} (domain), \SB{}/\USB{} (mechanism),
            \AD{} (interaction) \\
\bottomrule
\end{tabular}
\end{table}

\section{Comparison of Related Work}
\label{app:comparison}

\begin{sidewaystable}
\caption{Comparison of related work. Legend: \YES{} allowed/supported;
  \NO{} ruled out / required not to happen; \PARTIAL{} partial / under
  restrictions; n/a not the object of study.
  Rows ordered most-recent to oldest; ours on top.
  \NEW{Types follow \citet{matthes2025mechanistic}: \textbf{D}, the two factors' Jacobian columns occupy disjoint observation coordinates; \textbf{M}, supports may overlap but neither contains the other; \textbf{S}, the factor-aligned basis is the strict $\ell_0$-sparsest; \textbf{H}$_n$, the $n$-th mixed-derivative block between the factors vanishes; \textbf{O}, Jacobian columns inner-product orthogonal. D is the strongest. Our own row carries none of these tags, so no result changes.
  \citet{kivva2022identifiability} and \citet{xu2026identifiability} also have multiple PWA branches; the branch-structure column records whether that structure is used to identify, not whether it exists.
  On the domain-contrast axis, \condref{MST} strictly weakens iVAE's observed labels plus assumption~(iv): it can hold with $d+1$ components and no labels, where (iv)\NEWD{, for Gaussian latents,}\dtag{B12-6} needs $K>2d$ (App.~\ref{app:tech-comparisons}).}\edtag{E32}}
\label{tab:comparison}
\centering\scriptsize
\setlength{\tabcolsep}{3pt}
\renewcommand{\arraystretch}{1.18}
\begin{tabular}{@{}p{2.0cm} p{2.4cm} p{2.6cm} p{2.6cm}
                  p{1.9cm} p{1.5cm} p{2.0cm} p{2.0cm} p{2.0cm}@{}}
\toprule
\textbf{Paper}
  & \textbf{Model class $(\mapf,\latentZ)$}
  & \textbf{Constraint location}
  & \textbf{Proof flavour}
  & \textbf{\NEW{Branch structure used as an identification resource}}
  & \textbf{Discont.\ $\mapf$}
  & \textbf{Non-inj.\ $\mapf$}
  & \textbf{Non-indep.\ latent}
  & \textbf{Unsup.\ domains} \\
\midrule
\textbf{Ours}
  & reduced PWA + irredundant GMM
  & Map (\USB{}); Law (\MST{}); Interaction (\AD{})
  & algebraic + analytic continuation on chambers; symmetry-group collapse
  & \YES{} (branches of one $\mapf$)
  & \YES{}
  & \YES{} (except pointwise)
  & \YES{} (affine; ICA form under \DD{})
  & \YES{} (mixture labels latent) \\
\addlinespace[2pt]
\citet{xu2026identifiability}
  & inj.\ continuous PWA + pdGMM
  & Map (inj.\ cont.\ PWA); Law (pdGMM genericity, shared basis); encoder sparsity
  & analytic + sparsity (extends Kivva past degeneracy)
  & \NO{} & \NO{} & \NO{}
  & \PARTIAL{} (low-rank / dependent latents via degeneracy) & \YES{} \\
\addlinespace[2pt]
\citet{matthes2025mechanistic}
  & generic nonlinear, possibly non-invertible
  & Interaction (support / sparsity / higher-order on $\mapf$); law free
  & algebraic / graph-theoretic taxonomy unifying Type O/D/S/H
  & \NO{} & \NO{} (smooth) & \PARTIAL{} (local inj.)
  & \YES{} & \YES{} \\
\addlinespace[2pt]
\citet{brady2024interaction} (Type H)
  & $C^{n+1}$ generator + block latents
  & Interaction (block-diag.\ on $(n{+}1)$-th derivatives)
  & differential (higher-order Jacobian; autoencoder regulariser)
  & \NO{} & \NO{} & \NO{} (smooth autoencoder)
  & \YES{} (within-block dep.) & \YES{} \\
\addlinespace[2pt]
\citet{lopez2024toward}
  & PWA inj.\ a.e.\ + structured GMM
  & Law (block structure across two conditions)
  & reduces to Kivva Thm.~D.4; block-wise affine ID
  & \NO{} (two conditions sharing background) & \NO{} & \NO{} (\PARTIAL{} a.e.\ inj.)
  & \PARTIAL{} (within-block dep.; indep.\ across split)
  & \NO{} (observed condition) \\
\addlinespace[2pt]
\citet{lachapelle2023additive} (Type H)
  & $C^2$-diffeo additive decoder + block latents
  & Map (additivity + sufficient nonlinearity)
  & differential (block-wise)
  & \NO{} & \NO{} & \NO{} ($C^2$ diffeo)
  & \YES{} (``almost arbitrary support'') & \YES{} \\
\addlinespace[2pt]
\citet{brady2023provably} (Type D/S)
  & diffeo compositional decoder + slot latents
  & Interaction (compositionality + irreducibility on $J_{\mapf}$)
  & differential (Jacobian sparsity)
  & \NO{} (slot blocks) & \NO{} & \NO{}
  & \YES{} & \YES{} \\
\addlinespace[2pt]
\citet{zheng2023generalizing} (Type M)
  & smooth, possibly non-bijective + grouped latents
  & Map (structural / partial sparsity); Law (partial dep.)
  & differential + sparsity
  & \NO{} & \NO{} & \NO{} (undercomplete, still inj.)
  & \YES{} (explicit) & \YES{} \\
\addlinespace[2pt]
\citet{xi2023indeterminacy}
  & injective $\mapf$ + unspecified prior family
  & Taxonomic: $\mathcal{A}(\mathcal{F})\cap\mathcal{A}(\mathcal{P}_z)$
  & algebraic taxonomy; measure-theoretic intersection
  & \NO{} & n/a & \NO{} (Asn.~2)
  & n/a & \NO{} (multi-env.\ needs labels) \\
\addlinespace[2pt]
\citet{buchholz2022function}
  & conformal / smooth $\mapf$ + indep.\ latents
  & Map (function-class restriction)
  & differential (function-class restriction)
  & \NO{} & \NO{} & \NO{} & \NO{} & \YES{} \\
\addlinespace[2pt]
\leavevmode\NEW{\citet{yang2022nonlinear}}\edtag{E38}
  & \leavevmode\NEW{volume-preserving $\mapf$ + latents cond.\ indep.\ given a class}
  & \leavevmode\NEW{Map (volume preservation); Law (two overlapping classes)}
  & \leavevmode\NEW{differential}
  & \leavevmode\NEW{\NO{}} & \leavevmode\NEW{\NO{}} & \leavevmode\NEW{\NO{} (invertible flow)}
  & \leavevmode\NEW{\NO{} (cond.\ indep.\ given class)} & \leavevmode\NEW{\NO{} (class observed)} \\
\addlinespace[2pt]
\citet{kivva2022identifiability}
  & continuous PWA + non-deg.\ GMM
  & Map (continuity + inj.\ on open $\mapf$-image set / global inj.); Law (GMM)
  & analytic (continuation, GMM matching)
  & \NO{}
  & \PARTIAL{} (cont.\ for MID; LID allows discont.)
  & \PARTIAL{} (open-set inj.\ for LID; global inj.\ for MID)
  & \PARTIAL{} (no separation of MID from ICA-amb.) & \YES{} \\
\addlinespace[2pt]
\citet{ahuja2022properties}
  & diffeo decoder + \emph{known} dynamics $m_t$
  & Map (diffeo); known $m_t$ / hypothesis class
  & group / equivariance (commutation)
  & \NO{} (iterated single $\mapg$) & \NO{} & \NO{}
  & \YES{} & \NO{} (time observed; mechanism known) \\
\addlinespace[2pt]
\citet{roeder2021linear}
  & broad family of discriminative softmax-form models
  & distributional matching of conditionals + diversity
  & algebraic (linear ID in function space)
  & \NO{} & n/a & \YES{} (no inj.) & \YES{} & n/a \\
\addlinespace[2pt]
\citet{gresele2021independent} (IMA, Type O)
  & smooth invertible $\mapf$ + indep.\ latents
  & Interaction (Jacobian columns $\perp$)
  & differential (Jacobian-based; not a full ID theorem)
  & \NO{} & \NO{} & \NO{} & \NO{} & \YES{} \\
\addlinespace[2pt]
\citet{khemakhem2020variational} (iVAE)
  & smooth invertible $\mapf$ + EF prior cond.\ on aux
  & Map (smooth/inv.); Law (EF + aux variability)
  & differential (Jacobian + sufficient stats)
  & \NO{} & \NO{} & \NO{}
  & \NO{} (cond.\ indep.\ given aux) & \NO{} (aux observed) \\
\addlinespace[2pt]
\citet{gresele2019incomplete} (Multi-View NICA)
  & smooth invertible views + indep.\ noise
  & Law (indep.\ + SDV on cond.\ log-density); Map (smooth/inv.)
  & differential (log-density factorisation, contrastive)
  & \PARTIAL{} (multi-view) & \NO{} & \NO{} & \NO{}
  & \PARTIAL{} (no env.\ labels but paired views) \\
\addlinespace[2pt]
\citet{higgins2018towards}
  & encoder $\mapf\colon W\to Z$; group acts on $W$
  & direct-product decomp.\ \emph{given}; latent law plays no role
  & definitional, group / representation theory
  & n/a & n/a & hand-waved
  & \PARTIAL{} (factor-as-subgroup) & n/a \\
\addlinespace[2pt]
\citet{hyvarinen2016unsupervised} (TCL)
  & smooth invertible $\mapf$ + non-stationary indep.\ latents
  & Map (smooth/inv.); Law (segment variability)
  & differential (Jacobian / log-density, contrastive)
  & \NO{} & \NO{} & \NO{}
  & \NO{} (indep.\ required) & \NO{} (segment labels) \\
\bottomrule
\end{tabular}
\end{sidewaystable}

\section{Linear independence of Gaussian densities}
\label{app:gaussian-indep}

\begin{lemma}[Linear independence of Gaussian densities \citep{teicher1963identifiability}]
\label{lem:gaussian-independence}
Let $\{\theta_r\}_{r=1}^M$ be pairwise distinct parameters
$\theta_r=(m_r,S_r)$ with $S_r$ positive definite. If
$\sum_{r=1}^M \alpha_r\,\NormalDensityParam{\theta_r}(x) = 0$ for all
$x\in\R^{\ambientdim}$, then $\alpha_r=0$ for all~$r$.
\end{lemma}

\begin{proof}[Proof sketch]
Take Fourier transforms. The Fourier transform of a non-degenerate
Gaussian is again a non-vanishing Gaussian; distinct mean--covariance
pairs yield distinct exponential-quadratic forms. These are linearly
independent as functions, so all coefficients vanish. A detailed proof
can be found in classical GMM identifiability references.
\end{proof}

\section{Global vs.\ chamber-wise (AD)}
\label{app:global-pi}

\begin{remark}[Global vs.\ chamber-wise]
\label{rem:local-pi}
Chamber-wise \condref{AD}---injectivity of the parameter map restricted to the active pairs on each individual chamber---is genuinely weaker than the global condition stated above.
The glueing step in Lemma~\ref{lem:global-param-matching} requires the \emph{global} version: uniqueness of the match $(i,k)$ for a given parameter value across \emph{all} branches, not just those active on a single chamber.
A counterexample showing that injectivity of the PWA map does not imply global \condref{AD} (even though it implies chamber-wise \condref{AD}) appears below.
\end{remark}

Chamber-wise \condref{AD} requires only that the parameter map $(i,k)\mapsto\paramf{i}{k}$ be injective when restricted to the active pairs $\Ichamber{\mapf}{\fineC}\times[K]$ on each individual chamber~$\fineC$. Global \condref{AD} requires injectivity over \emph{all} branch-component pairs. The following example shows these are genuinely different.

\begin{proposition}
\label{prop:global-pi-counterexample}
There exists an injective PWA map $\mapf$ and a GMM $\latentZ$ such that $(\mapf,\latentZ)$ satisfies chamber-wise \condref{AD} but not global \condref{AD}.
\end{proposition}

\begin{proof}
Take $d=1$, two branches, two mixture components.
Define
$$
f_1(z)=z \text{ on } P_1=(-\infty,0),\qquad f_2(z)=z+10 \text{ on } P_2=(0,\infty).
$$
This $f$ is injective and reduced (distinct affine rules). Let
$$
Z\sim w_1\,\mathcal{N}(\mu_1,\sigma^2)+w_2\,\mathcal{N}(\mu_2,\sigma^2),
$$
with $\mu_2=\mu_1-10$. Then
$$
\paramf{1}{1} = (\mu_1,\sigma^2) = (\mu_2+10,\sigma^2) = \paramf{2}{2},
$$
so global \condref{AD} fails.

On the other hand, since $f$ is injective, each nonempty chamber has exactly one active branch. Hence the active pairs on any chamber form $\{i\}\times[K]$ for a single~$i$, and the restriction of the parameter map to that set is injective by invertibility of the branch rule and irredundancy of the GMM. So chamber-wise \condref{AD} holds.
\end{proof}

\textbf{Where global (AD) enters.}
The glueing step in Lemma~\ref{lem:global-param-matching} uses the statement: for fixed $(j,\ell)$, there is a \emph{unique} $(i,k)$ with $\paramf{i}{k}=\paramEst{j}{\ell}$ across all branches, not just those active on a given chamber. Chamber-wise \condref{AD} alone allows a ``hidden duplicate parameter in a different branch'' phenomenon: on chamber~$\fineC$ with active set $\{i\}$, the match gives $\Psi_{\fineC}(j,\ell)=(i,k)$; but if $\paramf{i}{k}=\paramf{i'}{k'}$ for some $i'\neq i$ active on a different chamber~$\fineC'$, then $\Psi_{\fineC'}(j,\ell)$ could yield $(i',k')$, breaking the global consistency of~$\Psi$. This inconsistency propagates: the boundary-counting argument in Lemma~\ref{lem:boundary-counting} and the ``leaving blocks'' tracking in subsequent results all rely on the global bijection being well-defined.

\section{(AD) hierarchy, dimension counting, and information view}
\label{app:pi-genericity}

Let $D := \ambientdim(\ambientdim-1)/2$
denote the dimension of the orthogonal
symmetries of a covariance matrix
(the maximal dimension of the component
symmetry group; see
Appendix~\ref{app:sym-groups}).

\condref{AD} decomposes naturally into two sub-conditions,
which together with the orbit-based conditions
\condref{OD} and \condref{TOD}
form a hierarchy of increasingly strong
genericity requirements on the map--latent pair.

\paragraph{\CrossCompParamInj{} (cross-component).}
A violation
$\paramImg{j}{\ell}=\paramImg{m}{n}$
with $\ell\ne n$ forces
$\mapgeneric_m^{-1}\circ\mapgeneric_j$
to \emph{transport} $\Gauss_\ell$ onto
$\Gauss_n$. The corresponding bad set is a
coset of $O(\Sigmageneric_\ell)$ and
has dimension exactly~$D$.

\paragraph{\SameCompParamInj{} (same-component).}
A violation
$\paramImg{j}{\ell}=\paramImg{m}{\ell}$
with $j\neq m$ means the difference map
$\mapgeneric_m^{-1}\circ\mapgeneric_j$
preserves~$\Gauss_\ell$, i.e., it is a
\emph{symmetry} of that component.
The ``bad'' set of such affine maps has
dimension exactly~$D$.

\paragraph{\condref{OD} (\ODname).}
For distinct branches $j\neq m$ and distinct
components $\ell\neq n$,
$\orbit{\mapgeneric}{j}{\ell}
\cap\orbit{\mapgeneric}{m}{n}=\varnothing$.
A violation means
the difference map
$\mapgeneric_m^{-1}\circ\mapgeneric_j$
lies in the product
$\SymComp{n}\cdot\SymComp{\ell}^{-1}$,
a set of dimension at most~$2D$
inside~$\Aff(\R^{\ambientdim})$
(dimension $\ambientdim^2+\ambientdim$).
The ``bad set'' has codimension at
least~$2\ambientdim$, hence measure zero.

\paragraph{\condref{TOD} (\TODname).}
For all distinct branches $i\neq i'$ and all
component quadruples $(k,\ell,k',\ell')$,
$\bigl(\mapgeneric_i\circ
  \Transporter{k}{\ell}\bigr)
\;\cap\;
\bigl(\mapgeneric_{i'}\circ
  \Transporter{k'}{\ell'}\bigr)
= \varnothing$.
Setting $k=\ell$ and $k'=\ell'$
reduces
$\Transporter{k}{k}=\SymComp{k}$,
so \condref{TOD} implies \condref{OD}
(further restricted to $k\neq k'$)
and \SameCompParamInj{}
(further restricted to $k=k'$).
\condref{TOD} is thus strictly stronger
than both.

\paragraph{Dimension counting.}
Both \CrossCompParamInj{} and \SameCompParamInj{} bad sets sit
inside~$\Aff(\R^{\ambientdim})$
(dimension $\ambientdim^2+\ambientdim$) with
positive codimension, hence measure zero.
While their dimensions match,
\CrossCompParamInj{} and
\SameCompParamInj{} differ in what the
difference map must ``know'':
cross-component collisions require
the difference between two branches to encode
a relationship between \emph{two} independent
Gaussian components, while same-component
collisions require it to encode the
\emph{self-symmetry} of a single component.

\paragraph{Algorithmic-information viewpoint.}
For \condref{OD}, the dimension doubles because a joint orbit
intersection involves two independent degrees
of freedom, one for each component's symmetry
group.
Crucially, the difference map need \emph{not}
fix either component individually; it only
needs to realize the joint constraint coming
from both orbits. This makes \condref{OD}
logically independent of
\SameCompParamInj{}.
For \condref{TOD}, a violation requires
the difference map
$\mapgeneric_{i'}^{-1}\circ\mapgeneric_i$
to lie in
$\Transporter{k'}{\ell'}^{-1}
  \circ\Transporter{k}{\ell}$,
a set of dimension at most~$2D$
(two independent orthogonal degrees of
freedom, one from each covariance).
The full affine group has
dimension~$\ambientdim^2+\ambientdim$,
so the bad set has codimension at
least~$2\ambientdim$---the same
bound as for~\condref{OD}.

\section{Component symmetries and alternative conditions}
\label{app:symmetry-conditions}

\subsection{Component symmetry groups and orbits}
\label{app:sym-groups}

For each component~$\ell$ of a GMM, the \emph{component
symmetry group} is
\begin{equation}
  \label{eq:sym-comp}
  \SymComp{\ell}
  := \bigl\{T\in\Aff(\R^{\ambientdim}) :\; \pushfwd{T}{\Gauss_\ell}=\Gauss_\ell\bigr\}.
\end{equation}
Concretely, $T(x)=Ax+b\in\SymComp{\ell}$ if and only if
$A\Sigmageneric_\ell A^\top = \Sigmageneric_\ell$ and $b=(I-A)\mugeneric_\ell$. Thus
$\SymComp{\ell}$ is isomorphic to $O(\Sigmageneric_\ell)$, and has dimension $D:=\ambientdim(\ambientdim-1)/2$.

The \emph{orbit} of branch~$j$ under component~$\ell$ is
$\orbit{\mapgeneric}{j}{\ell}
  := \bigl\{\mapgeneric_j\circ T :\; T\in\SymComp{\ell}\bigr\}$.

Two branches have matching pushforward
parameters on component~$\ell$ if and only
if they belong to the same orbit:

\begin{proposition}[\SameCompParamInj{} $\iff$ orbit disjointness]
\label{prop:same-comp-equiv}
$(\mapgeneric,\GMMgeneric)$ satisfies \SameCompParamInj{} if and only
if $\orbit{\mapgeneric}{j}{\ell}\cap\orbit{\mapgeneric}{m}{\ell}=\varnothing$
for all distinct branches $j\neq m$ and any component~$\ell$.
\end{proposition}

\begin{proof}
The pushforward parameters match, $\paramImg{j}{\ell}=\paramImg{m}{\ell}$,
iff $\mapgeneric_m^{-1}\circ\mapgeneric_j\in\SymComp{\ell}$.
Rearranging: $\mapgeneric_j\in\mapgeneric_m\circ\SymComp{\ell}
= \orbit{\mapgeneric}{m}{\ell}$.
\end{proof}

\subsection{Orbit disjointness}
\label{app:OD}

Proposition~\ref{prop:same-comp-equiv}
handles collisions within a fixed component.
The following condition extends orbit
disjointness to \emph{cross-component}
collisions, where a single branch matches
different truth branches on different
components.

\textbf{Condition \condref{OD} (\ODname).}
\condlabel{OD}
A pair $(\mapgeneric,\GMMgeneric)$ satisfies \condref{OD} if for
distinct branches $j\neq m$ and distinct components $\ell\neq n$,
$\orbit{\mapgeneric}{j}{\ell}\cap\orbit{\mapgeneric}{m}{n}=\varnothing$.

A violation means the difference map
$\mapgeneric_m^{-1}\circ\mapgeneric_j$ lies in
$\SymComp{n}\cdot\SymComp{\ell}^{-1}$, a set of dimension at
most~$2D$ inside~$\Aff(\R^{\ambientdim})$, hence codimension at
least~$2\ambientdim$.

See the algorithmic-information viewpoint discussion in Appendix~\ref{app:pi-genericity} for a unified dimension-counting analysis of \condref{OD} and \condref{TOD}.

\subsection{Distinct weights and joint symmetry triviality}
\label{app:dw-jst}

\textbf{Condition \condref{DW} (\DWname).}
\condlabel{DW}
A GMM has \emph{distinct weights} if $\wgeneric_k\neq\wgeneric_\ell$
for all $k\neq\ell$.

Under \condref{DW}, each weight
$\wgeneric_k$ acts as a unique scalar ``tag''
that sticks to component~$k$ through
pushforward.
In the global bijection~$\Psi$ from
Lemma~\ref{lem:global-param-matching},
coefficient matching forces
$\wtrue{k}=\west{\ell}$;
\condref{DW} then forces $k=\ell$,
making the bijection
\emph{component-preserving}.

\textbf{Condition \condref{JST} (\JSTname).}
\condlabel{JST}
A GMM satisfies \condref{JST} if
$\bigcap_{k=1}^{K}\SymComp{k}=\{\Id\}$.

\condref{JST} forbids nontrivial affine maps that fix every component
individually. But crucially, it does \emph{not} forbid maps that
\emph{permute} components---a key distinction from~\condref{MST}.

\begin{proposition}
\label{prop:mst-jst}
The following implications hold:
(1)~$\condref{DW}+\condref{JST}\Longrightarrow\condref{MST}$;
(2)~$\condref{MST}\Longrightarrow\condref{JST}$;
(3)~under $\condref{DW}$, $\condref{JST}\iff\condref{MST}$.
\end{proposition}

\begin{proof}
\textbf{(1)} Let $T_\sharp p = p$. Since $T$ preserves the mixture
law, it induces a permutation $\sigma$ with $T_\sharp\Gauss_k =
\Gauss_{\sigma(k)}$ and $w_k = w_{\sigma(k)}$ for all~$k$. Under
\condref{DW}, distinct weights force $\sigma=\mathrm{id}$, so
$T\in\bigcap_k\SymComp{k}$, and \condref{JST} gives $T=\Id$.
\textbf{(2)} If $T\in\bigcap_k\SymComp{k}$, then $T_\sharp p = p$;
\condref{MST} gives $T=\Id$.
\textbf{(3)} Immediate from (1) and (2).
\end{proof}

\begin{remark}[The (JST)--(MST) gap]
\label{rem:jst-mst-gap}
Without \condref{DW}, \condref{JST} and \condref{MST} can diverge.
Consider $\Gauss_1=\Gauss(\mutrue{1},\sigma^2)$,
$\Gauss_2=\Gauss(\mutrue{2},\sigma^2)$, with $\mutrue{1}\ne\mutrue{2}$
and equal weights $w_1=w_2=\tfrac12$. Then
$\SymComp{1}\cap\SymComp{2}=\{\Id\}$: \condref{JST} holds. But the
reflection $T(x)=-x+(\mutrue{1}+\mutrue{2})$ swaps the two components;
with equal weights, $T_\sharp p=p$, so \condref{MST} fails.
If a decoder branch is post-composed with~$T$, the output distribution
is unchanged but the map has been modified. \condref{JST} cannot detect
this switcheroo; \condref{MST} can. This is why our main-text MID
result uses \condref{MST} rather than~\condref{JST}.
A two-dimensional extension of this example appears in Example~5 (Appendix~\ref{app:examples}).
\end{remark}

\section{Cycle determinant lemma}
\label{app:cycle-det}

\begin{lemma}[Cycle determinant]
\label{lem:cycle-det}
Let $p=\sum_{k=1}^K w_k\,\Gauss(\mu_k,\Sigma_k)$ be an irredundant GMM, and let $T(x)=Ax+b$ be an affine symmetry of $p$, inducing a permutation $\sigma\in S_K$ of the components. Then on each cycle of $\sigma$, the covariance determinants are constant.

In particular, if the numbers $\det\Sigma_1,\ldots,\det\Sigma_K$ are pairwise distinct, then $\sigma=\mathrm{id}$.
\end{lemma}

\begin{proof}
If $T$ sends component $k$ to component $\sigma(k)$, then $A\Sigma_k A^\top = \Sigma_{\sigma(k)}$. Hence
$$
\det\Sigma_{\sigma(k)} = (\det A)^2 \det\Sigma_k.
$$
Now follow a cycle $k\to\sigma(k)\to\cdots\to\sigma^{\ell-1}(k)\to k$. Multiplying the determinant relations gives
$$
\det\Sigma_k = (\det A)^{2\ell}\det\Sigma_k.
$$
Since $\det\Sigma_k>0$, this implies $(\det A)^{2\ell}=1$, hence $(\det A)^2=1$. Therefore
$$
\det\Sigma_{\sigma(k)} = \det\Sigma_k
$$
for every $k$, and determinants are constant along each cycle.
\end{proof}

\section{Sufficient conditions for symmetry triviality}
\label{app:suff-cond}

We present four sets of sufficient conditions for \condref{MST}, arising from the combinations $\{\condref{DW},\;\text{distinct }\det(\Sigma)\}\times\{\text{affinely independent means},\;\text{2-comp generic position}\}$.

\subsection{Reducing (MST) to (JST)}

Two sufficient conditions reduce \condref{MST} to \condref{JST}:

\begin{enumerate}[leftmargin=*]
\item \textbf{\condref{DW} (distinct weights).} Under \condref{DW}, any affine symmetry must fix every component individually (Proposition~\ref{prop:mst-jst}), so \condref{MST} reduces to \condref{JST}.

\item \textbf{Distinct covariance determinants.} By Lemma~\ref{lem:cycle-det}, if $\det\Sigma_1,\ldots,\det\Sigma_K$ are pairwise distinct, any affine symmetry induces $\sigma=\mathrm{id}$, again reducing \condref{MST} to \condref{JST}.
\end{enumerate}

\subsection{Sufficient conditions for (JST)}

Two sufficient conditions for \condref{JST}:

\begin{enumerate}[leftmargin=*]
\item \textbf{Affinely independent means.}

\begin{lemma}[Affinely independent means $\Rightarrow$ (JST)]
\label{lem:aff-indep-jst}
If the means $\{\mugeneric_k\}_{k=1}^{K}$ are in affinely
general position (i.e., the direction space
$U:=\mathrm{span}\{\mugeneric_k-\mugeneric_1 : k=2,\ldots,K\}$ has
dimension~$\ambientdim$), then the GMM satisfies \condref{JST}.
\end{lemma}

\begin{proof}
Let $T(x)=Ax+b\in\bigcap_k\SymComp{k}$. Then $T$ fixes each
component, so $A\Sigmageneric_k A^\top = \Sigmageneric_k$ for
every~$k$ and $b = (I-A)\mugeneric_k$ for every~$k$. The second
condition gives $(I-A)(\mugeneric_k-\mugeneric_j)=0$ for all~$j,k$,
implying $A|_U=I|_U$. If $\dim U = \ambientdim$, then $A=I$ and $b=0$.
\end{proof}

Note that this condition requires $K\ge d+1$ components.

\item \textbf{2-component generic position.}

\begin{theorem}[Generic affine MST for GMMs]
\label{thm:generic-affine-mst}
Let $M=\sum_{i=1}^K w_i\,\Gauss(\mu_i,\Sigma_i)$, $K\ge 2$, be an ordered irredundant GMM on $\R^d$. Suppose:
\begin{enumerate}
\item the matrix $H:=\Sigma_1^{-1/2}\Sigma_2\Sigma_1^{-1/2}$ has simple spectrum;
\item in an eigenbasis of $H$, the vector $v:=\Sigma_1^{-1/2}(\mu_2-\mu_1)$ has all coordinates nonzero.
\end{enumerate}
Then every element of $\bigcap_{k=1}^K\SymComp{k}$ is the identity. Therefore \condref{JST} holds on an open dense full-measure subset of GMM parameter space.
\end{theorem}

\begin{proof}
Let $x\mapsto Ax+b$ fix every component individually. In particular it fixes components $1$ and $2$. Set $O:=\Sigma_1^{-1/2}A\Sigma_1^{1/2}$. Then $O\in O(d)$, and $b=(I-A)\mu_1$. Fixing the second component gives $OHO^\top = H$ and $Ov=v$. Since $O^{-1}=O^\top$, the first identity implies $OH=HO$. Because $H$ has simple spectrum, every orthogonal matrix commuting with $H$ is diagonal with entries $\pm 1$ in an eigenbasis of $H$. In that basis, $v$ has no zero coordinate; therefore $Ov=v$ forces every sign to be $+1$. Hence $O=I$, so $A=I$ and $b=0$.

The two conditions define an open dense full-measure subset: simple spectrum excludes the discriminant-zero hypersurface; the nonzero-coordinate condition excludes finitely many hyperplanes in $v$-space on the simple-spectrum locus.
\end{proof}
\end{enumerate}

\subsection{Summary of four combinations}
\label{app:sum-suff-cond}
Combining the two reduction routes with the two \condref{JST} conditions gives four sets of sufficient conditions for \condref{MST}:

\begin{center}
\small
\begin{tabular}{@{}lll@{}}
\toprule
\textbf{Reduce to JST via} & \textbf{JST via} & \textbf{Practicality} \\
\midrule
\condref{DW} & affinely indep.\ means & requires $K>d$; DW often violated \\
\condref{DW} & 2-comp generic position & DW often violated \\
distinct $\det(\Sigma)$ & affinely indep.\ means & requires $K>d$ \\
distinct $\det(\Sigma)$ & 2-comp generic position & \textbf{most applicable} \\
\bottomrule
\end{tabular}
\end{center}
The first row is simple and practical: distinct weights are almost-sure under
continuous weight priors, and affinely independent means hold once
$K\ge\ambientdim+1$ with generic means.
The combination ``distinct $\det(\Sigma)$ + 2-component generic position'' is the most broadly applicable: \condref{DW} is sometimes violated in statistical practice (equal-weight mixtures), while affinely independent means requires $K>d$. All four sufficient conditions are \emph{generic}: each is violated only on a measure-zero algebraic subset of the parameter space. 

\begin{remark}[Sufficient conditions for (MST)]
\label{rem:mst-suff}
Beyond the four explicit combinations above, sufficient conditions for
\condref{MST} include: (a)~\condref{DW} plus \condref{JST}
(Proposition~\ref{prop:mst-jst}); 
(b)~general group-theoretic conditions extending to other exponential families; omitted here for scope.
Condition~(b) is the most general but least constructive.
\end{remark}

\section{ICA ambiguity conditions}
\label{app:ica-conditions}

\subsection{The (DD) condition}

\textbf{``Diagonal component covariances'' $\neq$ ``independent latent coordinates.''}
These are different notions because mixtures couple coordinates through the shared component index.
``Diagonal component covariances'' means: for each mixture component $C=k$,
$Z\mid (C=k)\sim \Gauss(\mu_k,\Sigma_k)$ with $\Sigma_k$ diagonal,
so coordinates are independent \emph{conditional on~$C$}.
``Independent latent coordinates'' means the \emph{unconditional} law factorizes:
$p_Z(z)=\prod_{t=1}^{\ambientdim}p_t(z_t)$.
A generic GMM with diagonal $\Sigma_k$ does \emph{not} factorize like this.
A clean diagnostic is the law of total covariance: for $s\neq t$,
$\mathrm{Cov}(Z_s,Z_t)
=\mathrm{Cov}\bigl(\mathbb{E}[Z_s\mid C],\mathbb{E}[Z_t\mid C]\bigr)
+\mathbb{E}\bigl[\mathrm{Cov}(Z_s,Z_t\mid C)\bigr]$.
If every $\Sigma_k$ is diagonal, then $\mathrm{Cov}(Z_s,Z_t\mid C)=0$, but the first term is generally nonzero.
For example, in $\R^2$,
$Z\sim \tfrac12\,\Gauss\bigl((0,0),I\bigr)+\tfrac12\,\Gauss\bigl((3,3),I\bigr)$
has diagonal $\Sigma_k$, yet $\mathrm{Cov}(Z_1,Z_2)=\tfrac{9}{4}>0$,
so $(Z_1,Z_2)$ are not independent.
Thus diagonal $\Sigma_k$ is a \emph{within-component} statement; coordinate-independence is a \emph{mixture-level} statement.
Corollary~\ref{cor:pwa-ica} proves an ICA-form affine ambiguity under diagonal component covariances; full classical ICA (independent coordinates) is a further specialization.

\medskip

The following condition is borrowed from \citet{kivva2022identifiability}.

\textbf{Condition \condref{DD} (\DDname).}
\condlabel{DD}
A GMM~$\GMMgeneric$ with parameters
$\{(\mugeneric_k,\Sigmageneric_k)\}_k$ satisfies \condref{DD} if
(i)~every covariance~$\Sigmageneric_k$ is diagonal, and
(ii)~there exist indices $k_1,k_2$ such that the spectral ratios
\[
  \Bigl(
    \frac{(\Sigmageneric_{k_1})_{tt}}{(\Sigmageneric_{k_2})_{tt}}
  \Bigr)_{t=1}^{\ambientdim}
\]
are pairwise distinct across dimensions $t\in\{1,\ldots,\ambientdim\}$.

\subsection{Proof of Corollary~\ref{cor:pwa-ica}}

\begin{proof}
By Theorem~\ref{thm:lid} (resp.\ Theorem~\ref{thm:mid}), there exists an affine bijection
\[
  \alpha(z)=Az+c,
  \qquad
  A\in \mathrm{GL}(d).
\]
By the component matching in Theorem~\ref{thm:lid}, there is a permutation
$\sigma$ of $[K]$ such that
\[
  A\Sigma'_\ell A^\top=\Sigma_{\sigma(\ell)}
\]
for all $\ell\in[K]$.

Choose truth indices $r,s$ given by condition~(DD), so that the diagonal
entries of
\[
  D:=\Sigma_r\Sigma_s^{-1}
\]
are pairwise distinct. Let
\[
  r':=\sigma^{-1}(r),
  \qquad
  s':=\sigma^{-1}(s),
  \qquad
  D':=\Sigma'_{r'}(\Sigma'_{s'})^{-1}.
\]
Since all covariance matrices are diagonal, both $D$ and $D'$ are diagonal, and
\[
  D
  =
  A D' A^{-1}.
\]

Let $e_j$ be the $j$-th coordinate vector. Since $D'$ is diagonal,
\[
  D'e_j=(D')_{jj}e_j,
\]
hence
\[
  D(Ae_j)=AD'e_j=(D')_{jj}Ae_j.
\]
So each column $Ae_j$ is an eigenvector of $D$. Because $D$ is diagonal with
pairwise distinct diagonal entries, its eigenspaces are exactly the coordinate
axes. Therefore each column of $A$ has exactly one nonzero entry. Since $A$ is
invertible, these nonzero entries occur in distinct rows, so
\[
  A=P\Lambda
\]
for some permutation matrix $P$ and some diagonal invertible matrix
$\Lambda$. Thus
\[
  \alpha(z)=P\Lambda z+c.
\]
\end{proof}

\begin{remark}[Comparison with linear ICA]
\label{rem:linear-ica}
In classical linear ICA
($\mapf(z)=Mz$ with mixing matrix~$M$),
identifiability states that~$M$
is recoverable up to right-multiplication
by a signed permutation matrix.
Corollary~\ref{cor:pwa-ica} extends this
to nonlinear PWA decoders and GMM priors:
the ``mixing map'' is a PWA function;
the latent distribution is a GMM;
and the identifiability takes the same
algebraic form
$\alpha(z) = P\Lambda\, z + c$.
The structural conditions
(\condref{AD}, \condref{USB}, \condref{MST})
replace the non-Gaussianity assumption
of classical ICA.
\end{remark}


\subsection{Alternative route: iVAE assumption~(iv) implies ICA ambiguity}

iVAE's assumption~(iv) with $k=2$ (Gaussian latents with modulated mean and variance) forces the diagonal \emph{precisions} of the truth components to span all diagonal directions: by the block-rank argument in the proof of Proposition~\ref{prop:ivae-mst}, the $2\ambientdim$ natural-parameter differences of~(iv) already make the variance block span $\R^{\ambientdim}$, with no independent generation of parameters. In particular $\ambientdim+1$ components have linearly independent diagonal precision differences, which we now show implies the ICA-form ambiguity. (A ``component'' in the ICA sense---a latent dimension---is distinct from a mixture component, i.e.\ a domain.)

\begin{proposition}[Affinely independent diagonal covariances or precisions $\Rightarrow$ ICA ambiguity]
\label{prop:aff-indep-diag-ica}
Let $\latentZ$ be an irredundant GMM with $K\ge 2$ components, all covariances diagonal. Suppose there exist $\ambientdim+1$ components $k_0,k_1,\ldots,k_\ambientdim$ such that \emph{either} the covariance differences $\mathrm{diag}(\Sigma_{k_j}) - \mathrm{diag}(\Sigma_{k_0})$ \emph{or} the precision differences $\mathrm{diag}(\Sigma_{k_j}^{-1}) - \mathrm{diag}(\Sigma_{k_0}^{-1})$ ($j=1,\dots,\ambientdim$) are linearly independent in $\R^\ambientdim$. Then any affine bijection $\alpha(z)=Az+c$ satisfying $\alpha(\latentZest)\deq\latentZ$ (with $\latentZest$ also having diagonal covariances) must have $A=P\Lambda$ for a permutation matrix~$P$ and diagonal invertible~$\Lambda$.
\end{proposition}

\begin{proof}
Because $\alpha$ is invertible affine and both mixtures are irredundant, standard mixture identifiability gives a permutation $\sigma\in S_K$ such that for each component index~$\ell$,
$$
A\Sigma'_{\ell}A^\top = \Sigma_{\sigma(\ell)}.
$$
Let $\ell_j:=\sigma^{-1}(k_j)$ for $j=0,1,\dots,\ambientdim$, and define diagonal matrices
$$
D'_j := \Sigma'_{\ell_j}-\Sigma'_{\ell_0},\qquad j=1,\dots,\ambientdim.
$$
Then
$$
AD'_jA^\top
= A(\Sigma'_{\ell_j}-\Sigma'_{\ell_0})A^\top
= \Sigma_{k_j}-\Sigma_{k_0},
$$
which is diagonal for each~$j$. The matrices $\Sigma_{k_j}-\Sigma_{k_0}$ are linearly independent in the $\ambientdim$-dimensional space of diagonal matrices (by assumption), hence the $D'_j$ are also linearly independent because $M\mapsto AMA^\top$ is an invertible linear map on matrices. Therefore $\{D'_j\}_{j=1}^\ambientdim$ is a basis of the diagonal subspace.

By linearity, for every diagonal matrix~$D$ we can write $D=\sum_{j=1}^\ambientdim a_j D'_j$, hence
$$
ADA^\top = \sum_{j=1}^\ambientdim a_j(AD'_jA^\top)
$$
is diagonal. In particular, for each standard basis diagonal matrix $E_t:=\mathrm{diag}(e_t)$,
$$
AE_tA^\top\ \text{is diagonal}.
$$
But $AE_tA^\top$ equals $a_t a_t^\top$, where $a_t$ is the $t$-th column of~$A$. A rank-one matrix $a_t a_t^\top$ is diagonal iff $a_t$ has at most one nonzero coordinate. Hence each column of~$A$ has at most one nonzero entry. Since $A$ is invertible, each column has exactly one nonzero entry, and these nonzero entries occur in distinct rows. Thus $A=P\Lambda$ for a permutation matrix~$P$ and an invertible diagonal~$\Lambda$.

This settles the covariance hypothesis. The precision hypothesis follows by the \emph{same argument with $\Sigma\mapsto\Sigma^{-1}$}: inverting $A\Sigma'_{\ell}A^\top=\Sigma_{\sigma(\ell)}$ gives $\Sigma'^{-1}_{\ell}=A^\top\Sigma_{\sigma(\ell)}^{-1}A$, so with $\widetilde D_j:=\Sigma_{k_j}^{-1}-\Sigma_{k_0}^{-1}$ one has $A^\top\widetilde D_j A=\Sigma'^{-1}_{\ell_j}-\Sigma'^{-1}_{\ell_0}$, again diagonal. Replacing $A(\cdot)A^\top$ by $A^\top(\cdot)A$ throughout, $A^\top E_t A=(A^\top e_t)(A^\top e_t)^\top$ is the rank-one outer product of the $t$-th \emph{row} of~$A$ with itself; its being diagonal forces each row of~$A$ to have at most one nonzero entry, so again $A=P\Lambda$.
\end{proof}

The chain of implications is: iVAE assumption~(iv) (Gaussian sources, modulated mean and variance, $K>2\ambientdim$ components) $\Rightarrow$ $\ambientdim+1$ components with affinely independent diagonal \emph{precisions} (by the block-rank argument of Proposition~\ref{prop:ivae-mst}) $\Rightarrow$ ICA ambiguity (Proposition~\ref{prop:aff-indep-diag-ica}).

\section{Worked examples}
\label{app:examples}

The five examples below, ordered by complexity, illustrate the distinct
ways in which \condref{AD}, \condref{SB}/\condref{USB}, and \condref{MST}
can fail, and the identifiability consequences that result.  Throughout,
$(\mapf,\latentZ)$ is the \emph{truth pair} and $(\mapg,\latentZest)$
the \emph{estimator pair}; we write $\latentZest\deq\latentZ$ when both
sides share the same latent law.  Example~4 recapitulates
Proposition~\ref{prop:fold-example} of Appendix~\ref{app:fold} in the
unified framework; Example~5 extends the one-dimensional
case of Remark~\ref{rem:jst-mst-gap} to two dimensions.

\paragraph{Example~1 (Fold-collapse; LID fails catastrophically).}
\label{ex:merge}

\begin{figure}[h]
\centering
\includegraphics[width=0.72\linewidth]{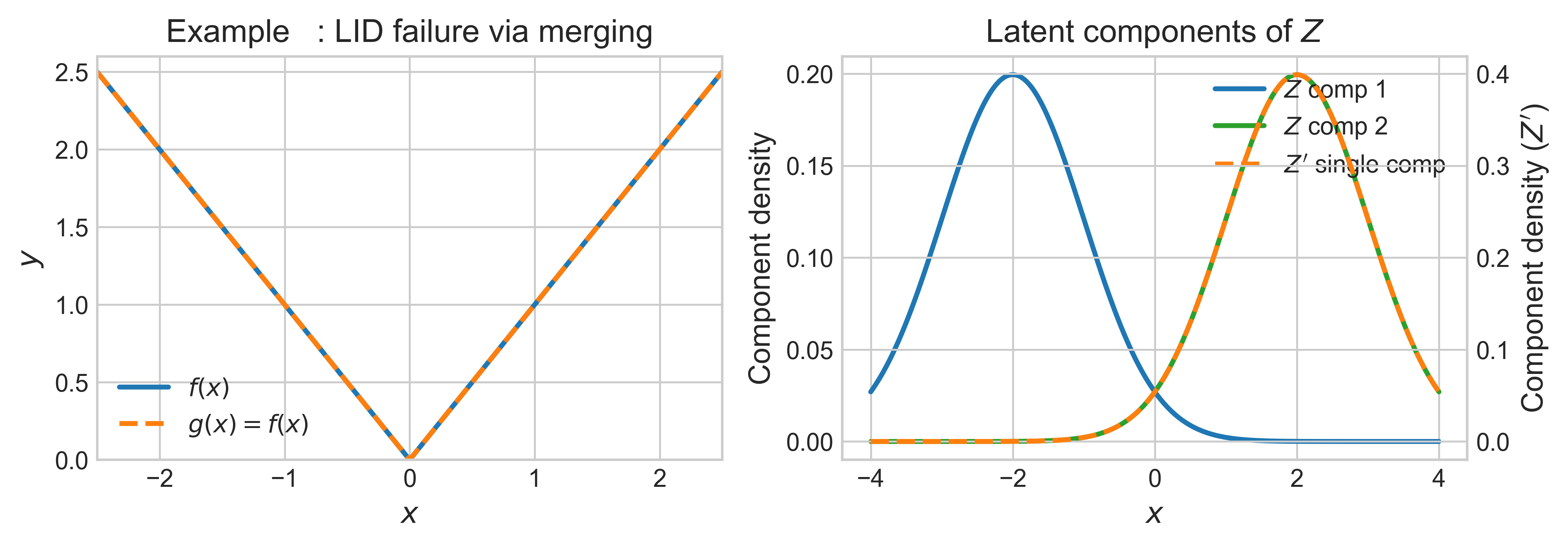}
\caption{Example~1: $\mapf=|x|$ collapses the symmetric two-component
  truth $\latentZ$ and the single Gaussian estimator $\latentZest$ to
  the same observed law.  LID fails; the two latents are not related by
  any affine map.}
\label{fig:ex1}
\end{figure}

\noindent\textbf{Setup.}
Take $\mapf=\mapg=|x|$ with branches $\branchf{1}=[0,\infty)$,
$x\mapsto x$ and $\branchf{2}=(-\infty,0)$, $x\mapsto -x$.
Truth latent: $\latentZ\deq\tfrac12\Gauss(-2,1)+\tfrac12\Gauss(2,1)$.
Estimator latent: $\latentZest\deq\Gauss(2,1)$ (single component).

\textbf{(SB)/(USB) for $\mapf$.}
Both branches map into $(0,\infty)$, so $\Ichamber{\mapf}{y}=\varnothing$
for $y<0$ and $\Ichamber{\mapf}{y}=\{1,2\}$ for $y>0$.  At the only
boundary $y=0$, \emph{both} branches enter simultaneously; no single
branch toggles alone.  Hence $\mapf$ violates \condref{SB} and \condref{USB}.

\textbf{(AD) for $(\mapf,\latentZ)$.}
With slopes $\linf{1}=1$ and $\linf{2}=-1$, the four branch-component
pushforward parameters are
\[
  \paramf{1}{1}=(-2,1),\quad
  \paramf{1}{2}=(2,1),\quad
  \paramf{2}{1}=(2,1),\quad
  \paramf{2}{2}=(-2,1).
\]
Pairs $(1,1)$ and $(2,2)$ collide at $(-2,1)$; pairs $(1,2)$ and $(2,1)$
collide at $(2,1)$.  Hence $(\mapf,\latentZ)$ violates \condref{AD}.
For $(\mapf,\latentZest)$ (single Gaussian), the two parameters
$(-2,1)$ and $(2,1)$ are distinct, so $(\mapf,\latentZest)$ satisfies
\condref{AD}.

\textbf{(MST).}
The reflection $T(x)=-x$ swaps the two components of $\latentZ$ with
equal weights, so $\Gmix{\latentZ}$ is nontrivial: $\latentZ$ fails
\condref{MST}.  Any single Gaussian $\Gauss(\mu,\sigma^2)$ is preserved
by $x\mapsto 2\mu-x$, so $\latentZest$ likewise fails \condref{MST}.

\textbf{Equal pushforwards.}
For any absolutely continuous $X$ with density $p_X$, the density of
$|X|$ satisfies $p_{|X|}(y)=p_X(y)+p_X(-y)$ for $y>0$.  Using
$\NormalDensity(y;\mu,\sigma^2)=\NormalDensity(-y;-\mu,\sigma^2)$
one computes $p_{|\latentZ|}(y)=p_{|\latentZest|}(y)$ for all $y>0$,
so $\mapf(\latentZ)\deq\mapf(\latentZest)$.

\textbf{LID fails.}
Any affine image of the Gaussian $\latentZest$ is again Gaussian, hence
cannot equal the genuine two-component mixture $\latentZ$: no affine
bijection $\alpha$ satisfies $\alpha(\latentZest)\deq\latentZ$.

\smallskip
\noindent\emph{Condition summary for Example~1:}
\begin{center}\small
\begin{tabular}{@{}lcc@{}}
\toprule
& Truth $(\mapf,\latentZ)$ & Estimator $(\mapf,\latentZest)$\\
\midrule
\condref{AD}  & \NO  & \YES\\
\condref{SB}  & \NO  & \NO \\
\condref{USB} & \NO  & \NO \\
\condref{MST} & \NO  & \NO \\
\bottomrule
\end{tabular}
\end{center}

\paragraph{Example~2 (Kivva D.2; LID holds, MID fails).}
\label{ex:kivva}

\begin{figure}[h]
\centering
\includegraphics[width=0.72\linewidth]{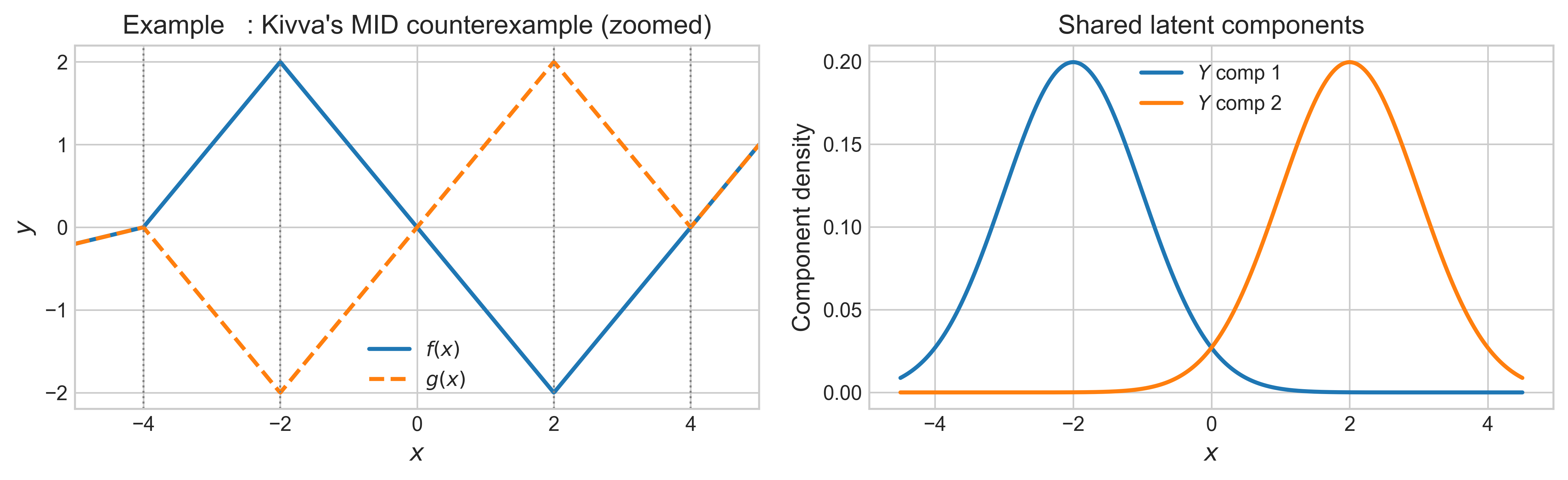}
\caption{Example~2 (Kivva D.2): two different PWA maps $\mapf$ (top)
  and $\mapg$ (bottom) push forward the same shared latent $\latentZ$
  to the same observed distribution.  LID is trivial ($\alpha=\Id$),
  but MID fails because all truth-side conditions are violated.}
\label{fig:ex2}
\end{figure}

\noindent\textbf{Setup.}
This is Example~D.2 of \citet{kivva2022identifiability}.  Let
$\latentZ\deq\tfrac12\Gauss(-2,1)+\tfrac12\Gauss(2,1)$ and
$\latentZest\deq\latentZ$ (shared law).  The PWA maps $\mapf,\mapg\colon\R\to\R$
are defined by
\[
  \mapf(x)=\begin{cases}
    x-4,     & x\ge 2,\\
    -x,      & -2\le x<2,\\
    x+4,     & -4\le x<-2,\\
    (x+4)/5, & x<-4,
  \end{cases}
  \qquad
  \mapg(x)=\begin{cases}
    x-4,    & x\ge 4,\\
    -x+4,   & 2\le x<4,\\
    x,      & -2\le x<2,\\
    -x-4,   & -4\le x<-2,\\
    (x+4)/5,& x<-4.
  \end{cases}
\]

\textbf{(AD).}
Computing branch-component parameters $(A_i\mutrue{k}+b_i,\,A_i^2)$
for the two components ($\mutrue{1}=-2$, $\mutrue{2}=2$, unit variances):
the branch $x\mapsto x+4$ on $[-4,-2)$ gives $(2,1)$ and $(6,1)$, while
the branch $x\mapsto -x$ on $[-2,2)$ gives $(2,1)$ and $(-2,1)$.  The
parameter $(2,1)$ is achieved by two branch-component pairs in $\mapf$:
$(\mapf,\latentZ)$ violates \condref{AD}.  For $(\mapg,\latentZest)$,
the parameter $(-2,1)$ is achieved by three separate branch-component
pairs: \condref{AD} fails even more severely.

\textbf{(SB)/(USB).}
The chambers of both maps partition the \NEWB{observational} axis at
$y\in\{-2,0,2\}$.  For $\mapf$: at $y=0$ exactly one branch enters
(the branch mapping $[-4,-2)$, whose image is $(0,2)$); this is a
simple boundary.  At $y=-2$ and $y=2$, sets of two or three branches
toggle together.  Hence $\mapf$ satisfies \condref{SB} but violates
\condref{USB}.  For $\mapg$: no boundary has a unique toggling branch;
$\mapg$ violates both \condref{SB} and \condref{USB}.

\textbf{(MST).}
The same reflection $T(x)=-x$ as in Example~1 is a nontrivial mixture
symmetry of $\latentZ$: \condref{MST} fails.

\textbf{Identifiability.}
LID holds trivially with $\alpha=\Id$ (shared latent law).  MID fails:
no affine $\alpha$ satisfies $\mapf=\mapg\circ\alpha^{-1}$.

\textbf{Relation to our framework.}
Kivva concluded that \emph{global injectivity} of $\mapf$ is necessary
for MID; our conditions show instead that \condref{AD}$+$\condref{USB}$+$\condref{MST}
(latter two on the truth side) suffice---conditions this example violates entirely.
Concretely, the difference map $S:=\mapf^{-1}\circ\mapg$ (defined
branch-wise using local invertibility) satisfies $S(\latentZ)
\deq\latentZ$ but is a genuinely non-affine PWA map that violates
\condref{AD}: the same latent-side non-affine symmetry pattern as in
Example~1, and foreshadowing Example~3.

\smallskip
\noindent\emph{Condition summary for Example~2:}
\begin{center}\small
\begin{tabular}{@{}lcc@{}}
\toprule
& Truth $(\mapf,\latentZ)$ & Estimator $(\mapg,\latentZest)$\\
\midrule
\condref{AD}  & \NO  & \NO \\
\condref{SB}  & \YES & \NO \\
\condref{USB} & \NO  & \NO \\
\condref{MST} & \NO  & \NO \\
\bottomrule
\end{tabular}
\end{center}

\paragraph{Example~3 (``Example~101''; injective maps, MID fails).}
\label{ex:101}

\begin{figure}[h]
\centering
\includegraphics[width=0.72\linewidth]{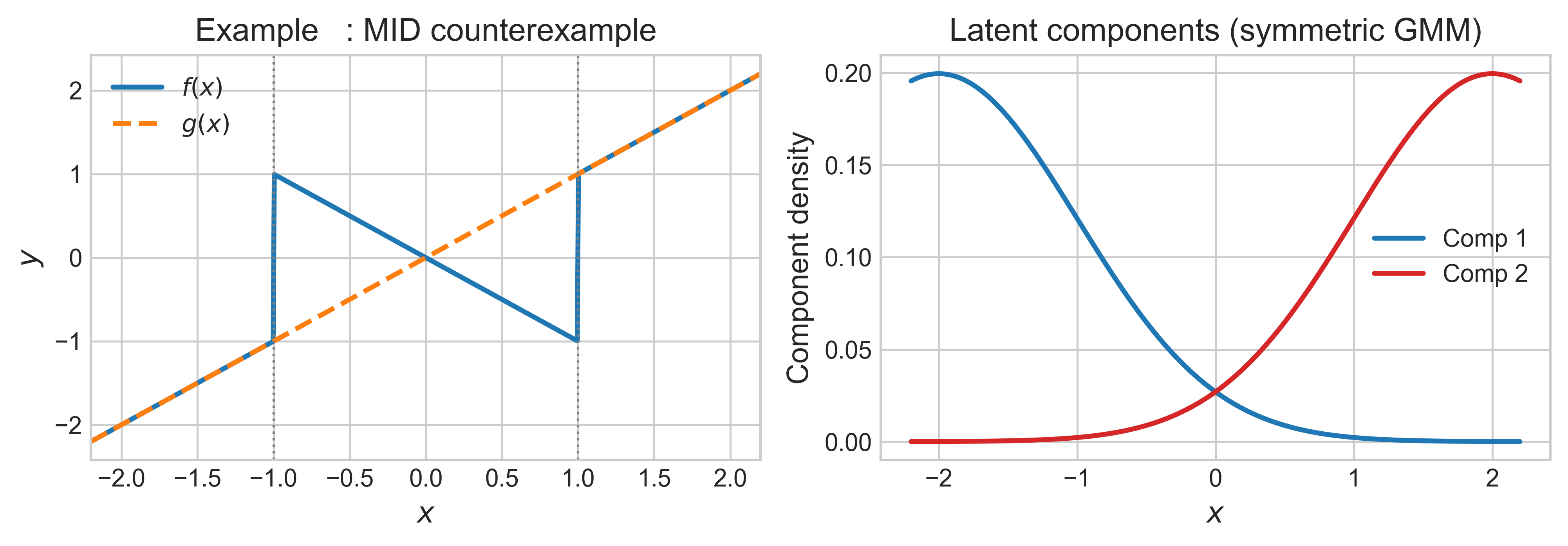}
\caption{Example~3 (``Example~101''): both $\mapf$ and $\mapg=\Id$ are
  injective, and $\mapf$ satisfies \condref{USB}, yet MID fails.
  The symmetric latent $\latentZ$ and the \condref{AD} failure of
  $(\mapf,\latentZ)$ conspire to hide the difference between the two maps.}
\label{fig:ex3}
\end{figure}

\noindent\textbf{Setup.}
Let
\[
  \mapf(x)=\begin{cases}-x,&|x|\le 1,\\ x,&|x|>1,\end{cases}
  \qquad\mapg(x)=x,
  \qquad\latentZ\deq\tfrac12\Gauss(-2,1)+\tfrac12\Gauss(2,1),
\]
with $\latentZest\deq\latentZ$ (shared law).  Under the reduced-PWA
definition (branch domains are finite unions of polyhedra), $\mapf$ has
a \emph{middle branch} $M\colon x\mapsto -x$ on $[-1,1]$ and an
\emph{outer branch} $O\colon x\mapsto x$ on $(-\infty,-1)\cup(1,\infty)$.
Both $\mapf$ and $\mapg=\Id$ are injective.

\textbf{Equal pushforwards.}
Since $\latentZ$ is symmetric under $x\mapsto -x$, the measure of $\latentZ$
on $[-1,1]$ is invariant under the reflection performed by branch $M$.
Combined with the identity action of branch $O$ on the tails,
$\mapf(\latentZ)\deq\latentZ=\mapg(\latentZest)$.

\textbf{(USB) for $\mapf$.}
Branch images: $M\to[-1,1]$, $O\to(-\infty,-1)\cup(1,\infty)$;
three chambers $(-\infty,-1)$, $(-1,1)$, $(1,\infty)$.  At $y=-1$,
$M$ enters and $O$ leaves (one branch each); at $y=1$, $M$ leaves and
$O$ enters.  Both branches enjoy a simple boundary, so $\mapf$ satisfies
\condref{USB}.

\textbf{(AD) for $(\mapf,\latentZ)$.}
Slopes: $\linf{M}=-1$, $\linf{O}=1$.  Branch-component parameters
$(A_i\mutrue{k}+b_i,\,A_i^2)$:
\[
  \paramf{M}{1}=(2,1),\quad
  \paramf{M}{2}=(-2,1),\quad
  \paramf{O}{1}=(-2,1),\quad
  \paramf{O}{2}=(2,1).
\]
Pairs $(M,1)$ and $(O,2)$ both yield $(2,1)$; pairs $(M,2)$ and $(O,1)$
both yield $(-2,1)$.  Hence $(\mapf,\latentZ)$ violates \condref{AD}.
For $(\mapg,\latentZest)=(\Id,\latentZ)$, the parameters $(-2,1)$ and
$(2,1)$ are distinct: \condref{AD} holds.

\textbf{(MST) for $\latentZ$.}
Same as Examples~1--2: $T(x)=-x$ is a nontrivial mixture symmetry,
so $\latentZ$ fails \condref{MST}.

\textbf{MID fails.}
$\mapf\ne\mapg$ yet $\mapf(\latentZ)\deq\mapg(\latentZest)$:
no affine $\alpha$ satisfies $\mapf=\mapg\circ\alpha^{-1}$.  This is a
genuinely injective MID counterexample---injectivity of $\mapf$ does not
rescue identifiability when \condref{AD} and \condref{MST} both fail.

\smallskip
\noindent\emph{Condition summary for Example~3:}
\begin{center}\small
\begin{tabular}{@{}lcc@{}}
\toprule
& Truth $(\mapf,\latentZ)$ & Estimator $(\mapg,\latentZest)$\\
\midrule
\condref{AD}  & \NO  & \YES\\
\condref{SB}  & \YES & \YES\\
\condref{USB} & \YES & \YES\\
\condref{MST} & \NO  & \NO \\
\bottomrule
\end{tabular}
\end{center}

\paragraph{Example~4 (Single-flip fold; \condref{AD} and \condref{MST}
  hold, \condref{SB} fails, LID fails).}
\label{ex:fold-sb}

\begin{figure}[h]
\centering
\includegraphics[width=0.72\linewidth]{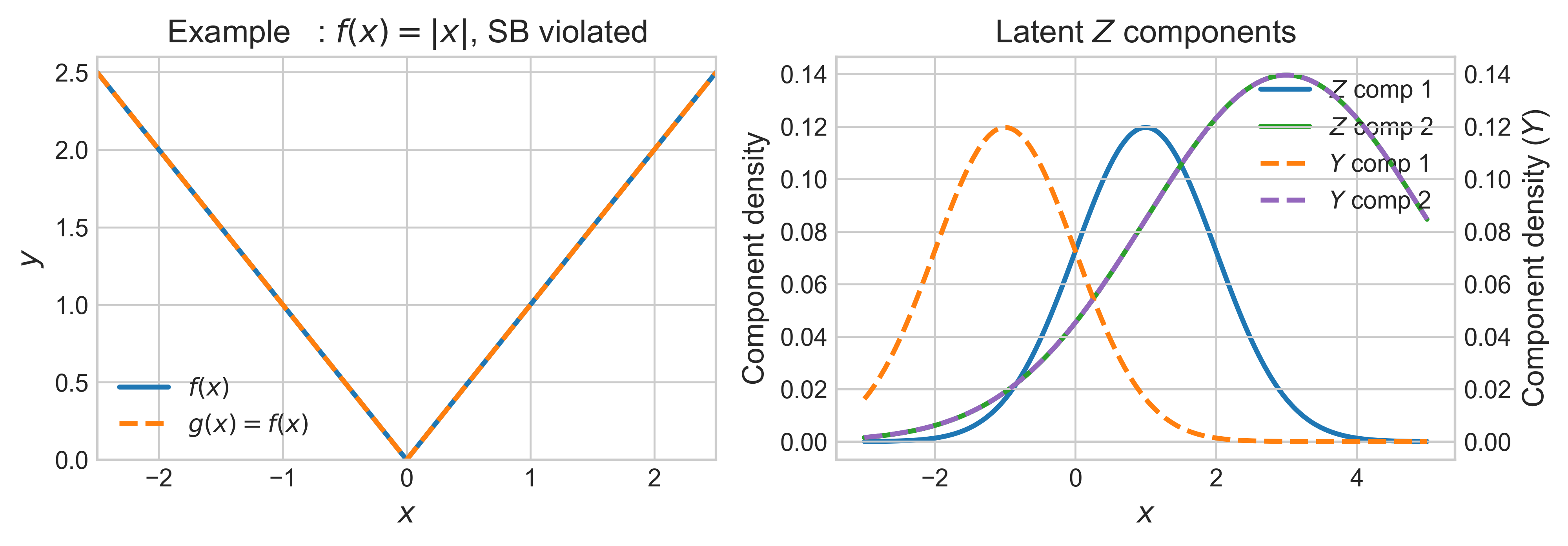}
\caption{Example~4: the fold $\mapf=|x|$ maps the truth $Z$ and the
  one-component-reflected $Y$ to the same observed law.  \condref{AD}
  and \condref{MST} hold for both latents, but $\mapf$ fails \condref{SB}
  and LID fails---showing \condref{SB} cannot be omitted even when the
  other conditions are in force.}
\label{fig:ex4}
\end{figure}

\noindent\textbf{Setup.}
Let $\mapf(x)=|x|$ (same \condref{SB}/\condref{USB} analysis as
Example~1), and define
\[
  Z\deq\tfrac12\Gauss(1,1)+\tfrac12\Gauss(3,4),
  \qquad
  Y\deq\tfrac12\Gauss(-1,1)+\tfrac12\Gauss(3,4).
\]
This is the setting of Proposition~\ref{prop:fold-example}
(Appendix~\ref{app:fold}); we recall the key points here.

\textbf{Results (see Appendix~\ref{app:fold} for proofs).}
\begin{enumerate}[leftmargin=*]
\item $\mapf(Z)\deq\mapf(Y)$:
  the fold identifies $\Gauss(1,1)$ with $\Gauss(-1,1)$ in the
  pushforward, while the shared $\Gauss(3,4)$ component is unchanged.
\item Both $Z$ and $Y$ satisfy \condref{MST}: the distinct variances
  ($1$ vs.\ $4$) forbid component-swapping, and no single reflection
  fixes both components individually.
\item Both pairs $(\mapf,Z)$ and $(\mapf,Y)$ satisfy \condref{AD}:
  the four pushforward parameters $\Gauss(\pm 1,1)$, $\Gauss(\pm 3,4)$
  are pairwise distinct.
\item No affine bijection $\alpha$ satisfies $\alpha(Y)\deq Z$:
  distinct variances ($1$ vs.\ $4$) forbid component-swapping,
  and neither $\alpha(x)=x+2$ nor $\alpha(x)=-x+2$ is consistent
  across both components simultaneously.
\end{enumerate}
Hence LID fails despite \condref{AD} and \condref{MST} holding on both
sides.  The obstruction is entirely on the map side: $\mapf=|x|$ has no
simple boundary (point~(1) of Appendix~\ref{app:fold}'s proof).  This
example belongs to a general \emph{reflection-fold family} described in
the remark following Proposition~\ref{prop:fold-example}.

\smallskip
\noindent\emph{Condition summary for Example~4:}
\begin{center}\small
\begin{tabular}{@{}lcc@{}}
\toprule
& Truth $(\mapf,Z)$ & Estimator $(\mapf,Y)$\\
\midrule
\condref{AD}  & \YES & \YES\\
\condref{SB}  & \NO  & \NO \\
\condref{USB} & \NO  & \NO \\
\condref{MST} & \YES & \YES\\
\bottomrule
\end{tabular}
\end{center}

\paragraph{Example~5 (\condref{JST} holds, \condref{MST} fails;
  1D and 2D).}
\label{ex:jst-mst}

\begin{figure}[h]
\centering
\includegraphics[width=0.72\linewidth]{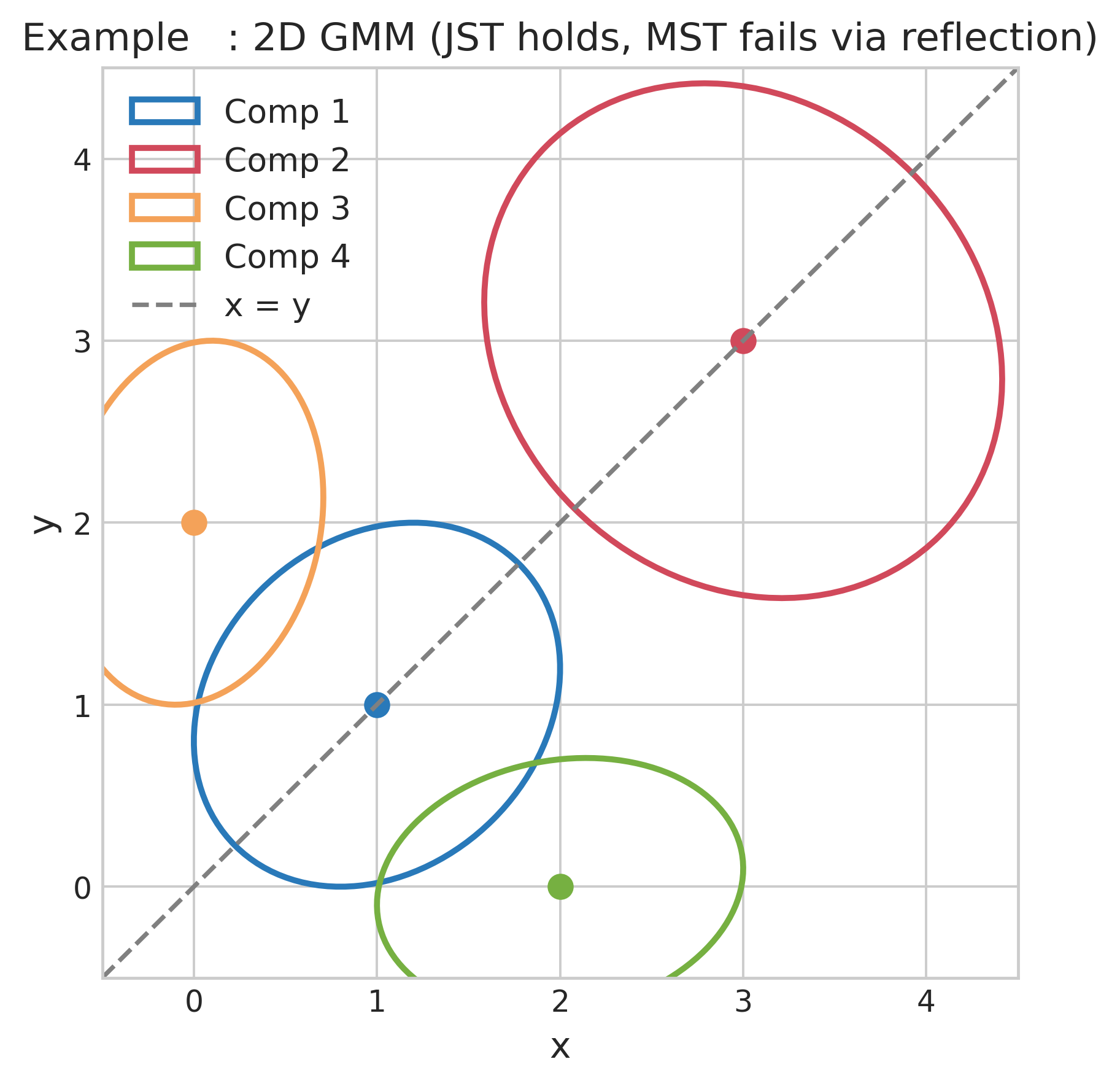}
\caption{Example~5: a 2D four-component GMM satisfying \condref{JST}
  but failing \condref{MST}.  Components~1 and~2 lie on the diagonal
  $x_1=x_2$, fixed individually by the swap $S(x_1,x_2)=(x_2,x_1)$.
  Components~3 and~4 are reflections of each other across this
  diagonal, swapped by $S$ with equal weights $w_3=w_4$.  Equal
  weights cause the \condref{JST}--\condref{MST} gap; under
  \condref{DW} the two conditions coincide
  (Proposition~\ref{prop:mst-jst}).}
\label{fig:ex5}
\end{figure}

\noindent
This example illustrates why the main-text MID results use \condref{MST}
rather than the strictly weaker \condref{JST}
(cf.\ Remark~\ref{rem:jst-mst-gap}).

\textbf{One-dimensional case (Remark~\ref{rem:jst-mst-gap}).}
Let $p=\tfrac12\Gauss(\mu_1,\sigma^2)+\tfrac12\Gauss(\mu_2,\sigma^2)$
with $\mu_1\ne\mu_2$.  No single affine map can simultaneously reflect
about both $\mu_1$ and $\mu_2$, so
$\SymComp{1}\cap\SymComp{2}=\{\Id\}$ and \condref{JST} holds.  However,
the reflection $T(x)=-x+(\mu_1+\mu_2)$ swaps the two components with
equal weights: $T_\sharp p=p$, $\Gmix{p}\ne\{\Id\}$, and \condref{MST}
fails.  A decoder branch post-composed with $T$ is undetectable from
the observed law; \condref{JST} cannot detect this switch, \condref{MST}
can.

\textbf{Two-dimensional extension.}
Work in $\R^2$.  Let $S(x_1,x_2)=(x_2,x_1)$ (the reflection across the
diagonal $x_1=x_2$, with matrix $P=\bigl(\begin{smallmatrix}0&1\\1&0
\end{smallmatrix}\bigr)$).  Define four components:
\begin{align*}
  \mu_1&=\begin{pmatrix}1\\1\end{pmatrix},\;\;
  \Sigma_1=\begin{pmatrix}1&\phantom{-}0.2\\0.2&1\end{pmatrix},
  &
  \mu_2&=\begin{pmatrix}3\\3\end{pmatrix},\;\;
  \Sigma_2=\begin{pmatrix}2&{-0.3}\\{-0.3}&2\end{pmatrix},\\[4pt]
  \mu_3&=\begin{pmatrix}0\\2\end{pmatrix},\;\;
  \Sigma_3=\begin{pmatrix}0.5&0.1\\0.1&1\end{pmatrix},
  &
  \mu_4&=P\mu_3=\begin{pmatrix}2\\0\end{pmatrix},\;\;
  \Sigma_4=P\Sigma_3 P^\top=\begin{pmatrix}1&0.1\\0.1&0.5\end{pmatrix},
\end{align*}
with weights $w_1=0.2$, $w_2=0.3$, $w_3=w_4=0.25$.

\emph{\condref{MST} fails.}
For $k=1,2$: both means satisfy $\mu_{k,1}=\mu_{k,2}$ and both
covariances satisfy $(\Sigma_k)_{11}=(\Sigma_k)_{22}$, so
$S_\sharp\Gauss(\mu_k,\Sigma_k)=\Gauss(\mu_k,\Sigma_k)$.  For $k=3,4$:
by construction $S_\sharp\Gauss(\mu_3,\Sigma_3)=\Gauss(\mu_4,\Sigma_4)$
and vice versa.  Since $w_3=w_4$, we have $S_\sharp p=p$, so
$\Gmix{p}$ contains the nontrivial element $S$ and \condref{MST} fails.

\emph{\condref{JST} holds.}
Suppose $T(x)=Ax+b$ fixes all four components individually.
From components 1 and 2 (means on the diagonal):
$A(3,3)^\top+b=(3,3)^\top$ and $A(1,1)^\top+b=(1,1)^\top$, giving
$A(1,1)^\top=(1,1)^\top$.  From components 3 and 4:
$A(2,0)^\top+b=(2,0)^\top$ and $A(0,2)^\top+b=(0,2)^\top$, giving
$A(1,-1)^\top=(1,-1)^\top$.  Since $(1,1)^\top$ and $(1,-1)^\top$ span
$\R^2$, we obtain $A=I$ and $b=0$: $\SymComp{1}\cap\cdots\cap\SymComp{4}
=\{\Id\}$, so \condref{JST} holds.

\begin{center}\small
\begin{tabular}{@{}lc@{}}
\toprule
Condition & GMM $p$\\
\midrule
\condref{JST} & \YES\\
\condref{MST} & \NO \\
\condref{DW}  & \NO\;($w_3=w_4$)\\
\bottomrule
\end{tabular}
\end{center}

\subsection{A fold example showing necessity of boundary richness}
\label{app:fold}

The next example isolates a map-side obstruction. It shows that latent \condref{MST} and \condref{AD} do not suffice for law identifiability when the decoder has a fold and no \condref{SB}-type boundary.

\begin{proposition}
\label{prop:fold-example}
Let
$$
f(x) = |x| = \begin{cases} -x, & x<0,\\ x, & x\ge 0, \end{cases}
$$
and define the two latent laws
$$
Z\sim\tfrac12\,\Gauss(1,1)+\tfrac12\,\Gauss(3,4), \qquad Y\sim\tfrac12\,\Gauss(-1,1)+\tfrac12\,\Gauss(3,4).
$$
Then:
\begin{enumerate}
\item $f(Z)\deq f(Y)$.
\item Both latent mixtures satisfy \condref{MST}.
\item Both pairs $(f,Z)$ and $(f,Y)$ satisfy \condref{AD}.
\item There is no affine bijection $\alpha\in\Aff(\R)$ such that $\alpha(Y)\deq Z$.
\end{enumerate}
Consequently, law identifiability fails although latent \condref{MST} and \condref{AD} both hold.
\end{proposition}

\begin{proof}
\textbf{Step 1: equal pushforwards.}
Because $f(x)=|x|$, $\pushfwd{f}{\Gauss(1,1)} = \pushfwd{f}{\Gauss(-1,1)}$. The second component $\Gauss(3,4)$ is the same in both mixtures, so $f(Z)\deq f(Y)$.

\textbf{Step 2: (SB) fails for $f$.}
The two branch images are both $[0,\infty)$. Hence the active set is $\Ichamber{f}{x}=\varnothing$ for $x<0$ and $\Ichamber{f}{x}=\{-,+\}$ for $x>0$. Across the boundary at $0$, both branches enter simultaneously. So there is no boundary at which exactly one branch enters or leaves, and $f$ fails \condref{SB}.

\textbf{Step 3: latent (MST).}
We check $Z$; the argument for $Y$ is identical. Let $T(x)=ax+b$ satisfy $T(Z)\deq Z$. Because the two component variances are $1$ and $4$, $T$ cannot swap the two components: swapping would require simultaneously $a^2\cdot 1=4$ and $a^2\cdot 4=1$, which is impossible. So $T$ must fix both components individually. A one-dimensional Gaussian $\Gauss(\mu,\sigma^2)$ is fixed only by $x\mapsto x$ or $x\mapsto 2\mu-x$. To fix both $\Gauss(1,1)$ and $\Gauss(3,4)$ individually, $T$ would have to be a reflection about $1$ and about $3$ at the same time, which is impossible. Hence $T=\Id$, and $Z$ satisfies \condref{MST}.

\textbf{Step 4: (AD).}
For $(f,Z)$, the four branch-component pushforward Gaussians are $\Gauss(1,1)$, $\Gauss(-1,1)$, $\Gauss(3,4)$, $\Gauss(-3,4)$, all distinct. So $(f,Z)$ satisfies \condref{AD}. For $(f,Y)$, the four are $\Gauss(-1,1)$, $\Gauss(1,1)$, $\Gauss(3,4)$, $\Gauss(-3,4)$, also all distinct.

\textbf{Step 5: no affine latent reparametrization.}
Suppose $\alpha(x)=ax+b$ satisfies $\alpha(Y)\deq Z$. Because the component variances are $1$ and $4$, $\alpha$ cannot swap the two components. So it must send $\Gauss(-1,1)\mapsto\Gauss(1,1)$ and $\Gauss(3,4)\mapsto\Gauss(3,4)$. Variance preservation forces $a^2=1$.

If $a=1$: the first condition gives $b=2$, while the second gives $b=0$. Contradiction.

If $a=-1$: the first condition gives $b=0$, while the second gives $b=6$. Contradiction.

So no such affine $\alpha$ exists.
\end{proof}

\begin{remark}
This example is the one-dimensional member of a natural reflection-fold family. Let $R$ be reflection across a hyperplane $H$, and define a two-branch map by $f(x)=A_+x+b$ on $H^+$ and $f(x)=A_+Rx+b$ on $H^-$. Then $f\circ R=f$. Replacing any latent Gaussian component $G$ by its reflected image $R_\sharp G$ leaves its pushforward through $f$ unchanged. For generic latent Gaussian parameters, this fold symmetry need not force latent \condref{MST} or \condref{AD} to fail; the obstruction is genuinely on the map side.
\end{remark}

\section{Simple/universal boundaries vs.\ weak injectivity}
\label{app:sb-usb-winj}

This appendix clarifies how our boundary conditions
\condref{SB}/\condref{USB} relate to global injectivity and to
Kivva's \emph{weak injectivity}
\citep{kivva2022identifiability}. All notation follows the main text.

\subsection{Weak injectivity and 1-active chambers}
\label{app:winj-def}

\begin{definition}[Weak injectivity]
\label{def:winj}
A PWA map $\mapgeneric\colon\R^\ambientdim\to\R^\ambientdim$ is
\emph{weakly injective} if there exists a nonempty open set
$U\subset\R^\ambientdim$ on which every $y\in U$ has a unique
preimage.
\end{definition}

\begin{proposition}
\label{prop:winj-1active}
Let $\mapgeneric$ be a reduced PWA map with invertible branch rules.
Then $\mapgeneric$ is weakly injective if and only if some chamber
$\fineC$ satisfies $|\Ichamber{\mapgeneric}{\fineC}|=1$.
\end{proposition}

\begin{proof}
($\Rightarrow$)
Let $U$ be the \NEW{open set given by} weak injectivity.  For $y\in U$, unique
preimage forces $|\Ichamber{\mapgeneric}{y}|=1$.  Since active sets
are constant on each chamber, the chamber containing any $y\in U$ is
1-active.

($\Leftarrow$)
If $\Ichamber{\mapgeneric}{\fineC}=\{i\}$, every $y\in\fineC$ lies
in $\mapgeneric(\branchgeneric_i)$ and in no other branch image.
Invertibility of $\mapgeneric_i$ gives a unique preimage on
$U:=\fineC$.
\end{proof}

\subsection{Global injectivity implies both USB and weak injectivity}
\label{app:inj-implies-both}

Global injectivity forces every nonempty chamber to be 1-active,
simultaneously yielding weak injectivity (trivially) and \condref{USB}
(Proposition~\ref{prop:inj-usb}).  Global injectivity also implies
weak injectivity, but the converse of both implications fails, as the
remainder of this appendix shows.

\subsection{Weak injectivity does not imply (SB)}
\label{app:winj-not-sb}

We first identify the structural obstruction and then give a concrete
example.

\begin{remark}[Co-active neighbouring branches]
\label{rem:coactive}
In a continuous 1-D PWA map, two branches sharing a domain breakpoint
are forced to share the corresponding \NEWB{chamber}\btag{B5-15} boundary (by
continuity $f(x_0^-)=f(x_0^+)$).  If these two branches \emph{both}
cover an open interval on the same side of the boundary in
\NEWB{observational} space, they enter or leave together.  A continuous map in
which \textit{every} pair of domain-neighbouring branches is co-active across
\textit{every} breakpoint therefore violates \condref{SB}---no boundary has a
unique toggling branch.  Any such map can be made weakly injective by
attaching a tail branch with a single-branch 1-active chamber
(e.g.\ a branch whose image extends to~$\pm\infty$).  The result is
a map that is weakly injective yet violates~\condref{SB}.
\end{remark}

\begin{figure}[t]
\centering
\includegraphics[width=0.85\linewidth]{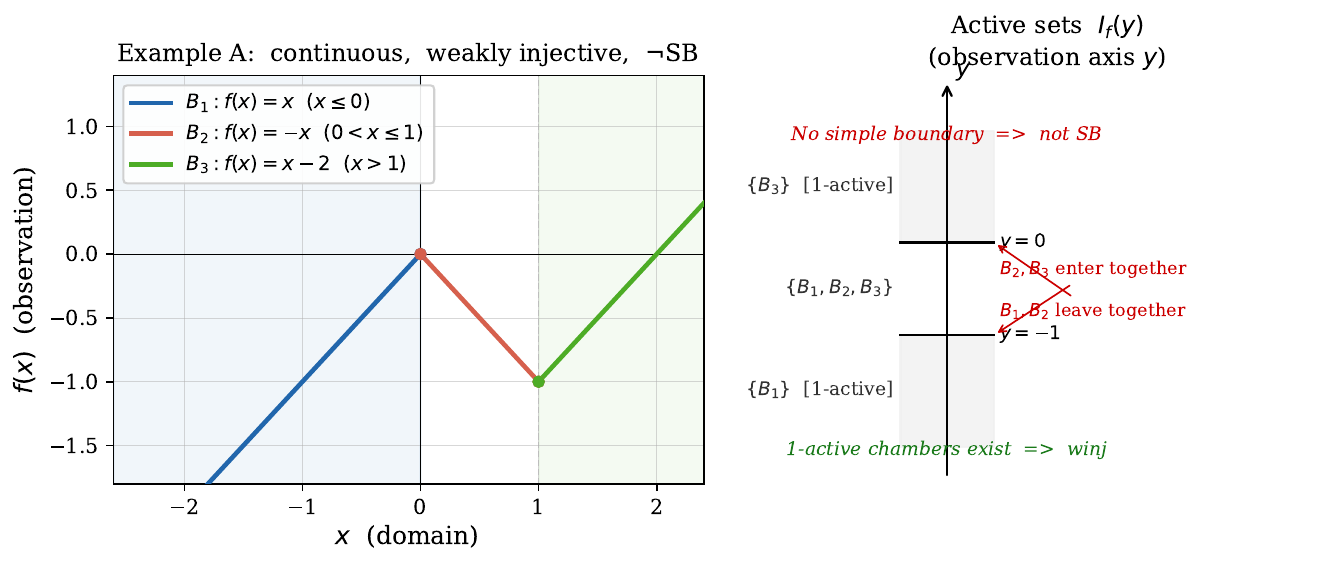}
\caption{Example~A\@: a continuous, weakly injective PWA map that
  violates \condref{SB}.  \emph{Left:} graph of~$\mapf$ with three
  branches.  The shaded regions indicate 1-active domains (\NEW{which make the map weakly injective}).  \emph{Right:} active-set diagram showing that at both
  boundaries, two branches toggle simultaneously.}
\label{fig:exA}
\end{figure}

\paragraph{Example~A (continuous, weakly injective, $\neg$SB).}
\label{ex:winj-no-sb}

Define $\mapf\colon\R\to\R$ by three branches on $(-\infty,0]$,
$(0,1]$, $(1,\infty)$:
\[
  \mapf(x)=\begin{cases}
    x,    & x\le 0,\\
    -x,   & 0<x\le 1,\\
    x-2,  & x>1.
  \end{cases}
\]
This map is continuous: $\mapf(0^-)=0=\mapf(0^+)$ and
$\mapf(1^-)=-1=\mapf(1^+)$.

\emph{Branch images and chambers.}
$B_1\to(-\infty,0]$, $B_2\to[-1,0)$, $B_3\to(-1,\infty)$.  The
\NEWB{observational} axis splits into three chambers:
\[
  (-\infty,-1):\;\Ichamber{\mapf}{\cdot}=\{B_1\},
  \qquad
  (-1,0):\;\{B_1,B_2,B_3\},
  \qquad
  (0,\infty):\;\{B_3\}.
\]
At $y=-1$, branches $B_2$ and $B_3$ enter together; at $y=0$, branches
$B_1$ and $B_2$ leave together.  No boundary has a unique toggling
branch, so $\mapf$ violates \condref{SB} (and hence \condref{USB}).
On the other hand, chambers $(-\infty,-1)$ and $(0,\infty)$ are
1-active, so $\mapf$ is weakly injective (Figure~\ref{fig:exA}).

\begin{remark}[Structural reading]
\label{rem:exA-structure}
Example~A is a ``weakly-injective tail attached to a co-toggling
core'': the middle portion ($B_2$, $B_3$) overlaps in image, forcing
coupled toggles, while the outer tails produce the 1-active chambers.
The same mechanism explains why $\mapg$ in
Example~2 (Appendix~\ref{app:examples}) violates \condref{SB}: every
breakpoint boundary of $\mapg$ has co-active domain neighbours.
We note that both $\mapf$ and~$\mapg$ in Example~2 are
weakly injective---the 1-active tail chambers $(-\infty,-2)$ and
$(2,\infty)$ \NEW{show} this---yet $\mapg$ violates both \condref{SB}
and~\condref{USB}.
\end{remark}

\begin{remark}[Discontinuous maps: weak injectivity generically implies SB]
\label{rem:discont-winj-sb}
The co-toggling mechanism of Remark~\ref{rem:coactive} requires
continuity: domain neighbours must share image boundaries.  Without
continuity, a 1-active chamber
$\fineC$ with $\Ichamber{\mapf}{\fineC}=\{i\}$ has a boundary
across which some branches enter.  For a discontinuous map, these
entering branches have image boundaries that are independent affine
images of non-adjacent domain \NEW{boundaries}; having two or more of them
coincide at the same $(d{-}1)$-dimensional \NEW{boundary} in \NEWB{observational} space
is a codimension-$(r{-}1)$ coincidence in the affine parameters
$(A_i,b_i)$.  Thus, for generic discontinuous PWA maps, a 1-active
chamber automatically yields a simple boundary, and weak injectivity
implies \condref{SB}.  This can be formalised
by parametrising reduced PWA maps by their affine data: the set of
maps where two non-neighbouring branch images share a boundary
is a proper algebraic subvariety of the parameter space, hence
measure-zero and nowhere-dense.
\end{remark}

\subsection{(SB) and (USB) do not imply weak injectivity}
\label{app:sb-not-winj}

We give a discontinuous example and a continuous example, both
satisfying \condref{USB} while having no 1-active chamber.

\begin{figure}[t]
\centering
\includegraphics[width=0.85\linewidth]{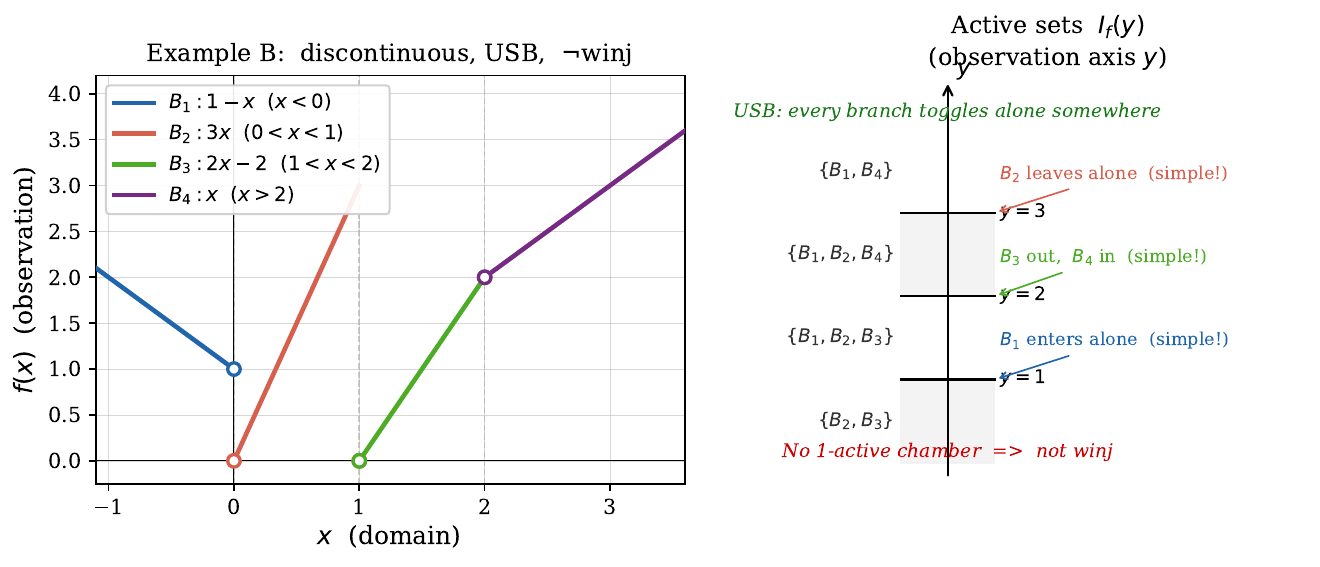}
\caption{Example~B\@: a discontinuous 4-branch PWA map satisfying
  \condref{USB} with no 1-active chamber.  \emph{Left:} graph
  of~$\mapf$; open circles mark jump discontinuities.  \emph{Right:}
  every branch toggles alone at some boundary (USB), yet all chambers
  are 2-or-more-active (not weakly injective).}
\label{fig:exB}
\end{figure}

\paragraph{Example~B (discontinuous, USB, $\neg$winj).}
\label{ex:discont-usb}

Let the branch domains be the open intervals $(-\infty,0)$, $(0,1)$,
$(1,2)$, $(2,\infty)$, with
\[
  \mapf(x)=\begin{cases}
    1-x,  & x<0,\\
    3x,   & 0<x<1,\\
    2x-2, & 1<x<2,\\
    x,    & x>2.
  \end{cases}
\]
Branch images: $B_1\to(1,\infty)$, $B_2\to(0,3)$, $B_3\to(0,2)$,
$B_4\to(2,\infty)$.  The \NEWB{observational} axis decomposes into four
chambers:
\[
  (0,1):\;\{B_2,B_3\},
  \quad
  (1,2):\;\{B_1,B_2,B_3\},
  \quad
  (2,3):\;\{B_1,B_2,B_4\},
  \quad
  (3,\infty):\;\{B_1,B_4\}.
\]
Every chamber has at least two active branches: $\mapf$ is \emph{not}
weakly injective.  Yet \condref{USB} holds:
at $y=1$, $B_1$ uniquely enters;
at $y=2$, $B_3$ uniquely leaves and $B_4$ uniquely enters;
at $y=3$, $B_2$ uniquely leaves.
Every branch toggles alone at some boundary
(Figure~\ref{fig:exB}).

\begin{figure}[t]
\centering
\includegraphics[width=0.85\linewidth]{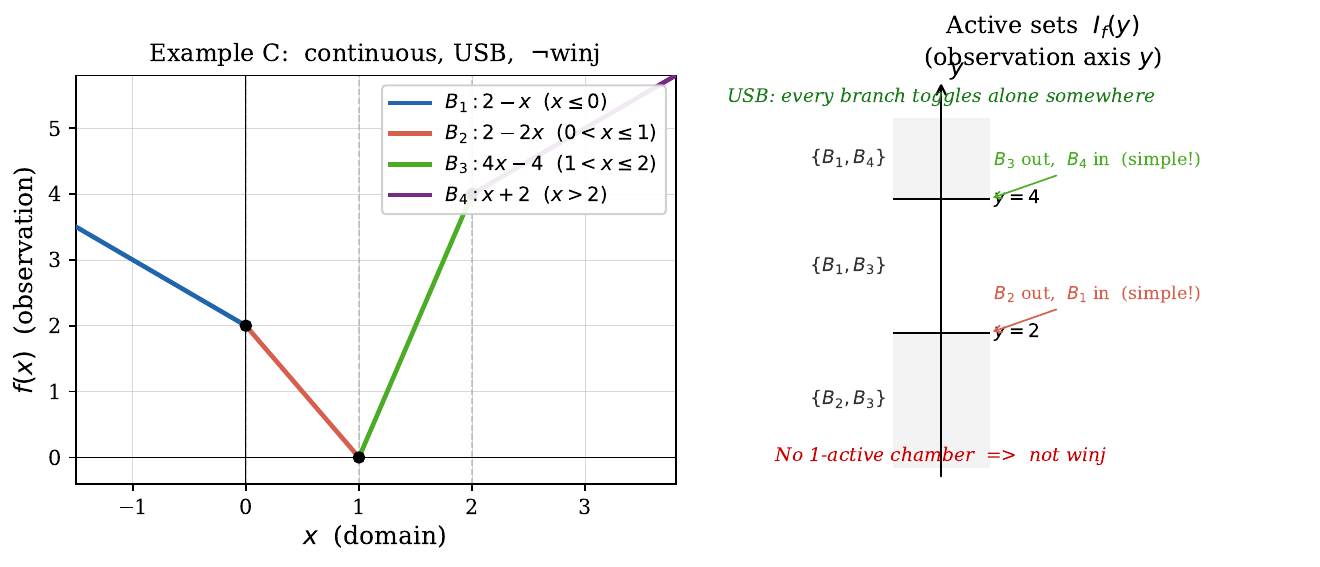}
\caption{Example~C\@: a continuous 4-branch PWA map satisfying
  \condref{USB} with no 1-active chamber.  \emph{Left:} graph
  of~$\mapf$; filled dots mark continuity joints.  \emph{Right:}
  every branch toggles alone (USB), yet all three chambers are
  exactly 2-active.}
\label{fig:exC}
\end{figure}

\paragraph{Example~C (continuous, USB, $\neg$winj).}
\label{ex:cont-usb}

Let the branch domains be $(-\infty,0]$, $(0,1]$, $(1,2]$,
$(2,\infty)$, with
\[
  \mapf(x)=\begin{cases}
    2-x,    & x\le 0,\\
    2-2x,   & 0<x\le 1,\\
    4x-4,   & 1<x\le 2,\\
    x+2,    & x>2.
  \end{cases}
\]
This map is continuous: $\mapf(0)=2$, $\mapf(1)=0$, $\mapf(2)=4$.
Branch images: $B_1\to[2,\infty)$, $B_2\to[0,2)$, $B_3\to(0,4]$,
$B_4\to(4,\infty)$.  Chambers:
\[
  (0,2):\;\{B_2,B_3\},
  \qquad
  (2,4):\;\{B_1,B_3\},
  \qquad
  (4,\infty):\;\{B_1,B_4\}.
\]
Every chamber is exactly 2-active; $\mapf$ is not weakly injective.
At $y=2$, $B_2$ uniquely leaves and $B_1$ uniquely enters; at $y=4$,
$B_3$ uniquely leaves and $B_4$ uniquely enters.  Each branch has a
simple boundary: \condref{USB} holds (Figure~\ref{fig:exC}).

\begin{remark}[Genericity of the USB$+\neg$winj pattern]
\label{rem:usb-nowinj-generic}
In both examples, the chamber structure is determined by the
combinatorial type of the branch partition and the signs of the
slopes.  Small perturbations of slopes and intercepts (preserving
continuity at breakpoints in Example~C) yield the same sign patterns
and hence the same chamber structure.  Thus ``USB~$\wedge\;\neg$winj''
holds on an open set in the space of affine parameters: these
examples are stable, not isolated.
\end{remark}

\subsection{(SB) does not imply global injectivity}
\label{app:sb-not-inj}

\condref{SB} is compatible with non-injective maps.  For a continuous
instance, the map~$\mapf$ of Example~2
(Appendix~\ref{app:examples}, originally
\citet[Example~D.2]{kivva2022identifiability}) is continuous, weakly
injective, and satisfies \condref{SB}, but is non-injective: in
\NEWB{observational} chamber $(-2,0)$, three branches
$\{A,C,D\}$ are simultaneously active.  For discontinuous instances,
Example~B above satisfies \condref{USB} (hence \condref{SB}) while
lacking even weak injectivity.

\subsection{Summary of logical relations}
\label{app:logical-summary}

\begin{center}\small
\begin{tabular}{@{}ll@{}}
\toprule
\textbf{Implication} & \textbf{Reference} \\
\midrule
global inj.\ $\Rightarrow$ \condref{USB} $\Rightarrow$ \condref{SB}
  & Prop.~\ref{prop:inj-usb}; definition \\
global inj.\ $\Rightarrow$ winj
  & trivial \\
winj $\not\Rightarrow$ \condref{SB}
  & Example~A (continuous); $\mapg$ in Ex.~2 \\
\condref{USB} $\not\Rightarrow$ winj
  & Examples~B (discont.) and~C (continuous)\\
\condref{SB} $\not\Rightarrow$ winj
  & (follows from \condref{USB} $\not\Rightarrow$ winj)\\
\condref{SB} $\not\Rightarrow$ global inj.
  & $\mapf$ in Ex.~2; Example~B \\
\bottomrule
\end{tabular}
\end{center}

In summary, \condref{USB} is strictly weaker than global injectivity,
while \condref{SB} and weak injectivity are logically incomparable.
These conclusions hold identically in the continuous and discontinuous
regimes.  There are no further implications:
\begin{equation}
\label{eq:no-implications}
\text{global inj.}
  \;\Longrightarrow\;
  \condref{USB}
  \;\Longrightarrow\;
  \condref{SB},
\qquad
\text{global inj.}
  \;\Longrightarrow\;
  \text{winj},
\qquad
\text{no other arrows.}
\end{equation}

\subsection{\NEW{Genericity:} \condref{SB}/\condref{USB} \NEW{versus} weak injectivity}
\label{app:winj-nongeneric}

The examples above prove logical independence between
\condref{SB}/\condref{USB} and weak injectivity. \NEW{We separate two regimes. Under unconstrained (including discontinuous) parametrizations, \condref{SB}/\condref{USB} is generic (Remark~\ref{rem:sb-usb-discont}). Under a continuity constraint we have no proof and state a conjecture (Conjecture~\ref{conj:sb-usb-cont}). The comparison with weak injectivity remains heuristic.}\edtag{E39} Throughout, ``non-generic'' here means \emph{failing on a set
that---while sometimes of positive measure---shrinks as the branch count
grows}, not the measure-zero notion used for \condref{AD}/\condref{MST} in the
main text.

\begin{remark}[\NEW{Discontinuous decoders}]
\label{rem:sb-usb-discont}
When discontinuities are allowed, \condref{SB} is generic: a single
non-co-toggling boundary suffices, and a generic PWA map has one. Its failure
set (every boundary co-toggling) is of higher codimension than that of weak
injectivity, so \condref{SB} is strictly harder to violate. \NEW{Likewise \condref{USB} is generic: it fails only if every boundary of some branch is co-toggling.} Nothing beyond the
examples above is needed here.
\end{remark}

\paragraph{Continuous decoders (open).}
When the decoder is constrained to be continuous we have no genericity proof
for \condref{SB}. In one dimension \condref{SB} can fail through an
every-boundary folding construction; as this is a condition on slope signs it
carves out a set of \emph{positive} measure, so \condref{SB} is not generic in
the strict sense. But the construction needs at least two folds, so
\condref{SB} holds with high probability and its failure probability decays
\emph{exponentially in the branch count~$M$}. The construction is
one-dimensional; whether it extends to higher dimensions is open---co-toggling
has more directions to exploit there---so we expect, but do not show, that a
comparable bound survives.

\begin{newedit}
\begin{conjecture}[Continuous decoders]
\label{conj:sb-usb-cont}
In every dimension~$\ambientdim$, the set of continuous PWA decoders that violate \condref{SB}, or \condref{USB}, shrinks as the branch count~$M$ grows.
\end{conjecture}
\noindent Weight-tied architectures are also open. They form a measure-zero slice of the space of continuous PWA decoders, so the conjecture would not cover them.
\end{newedit}

\paragraph{Weak injectivity is non-generic.}
Weak injectivity fails in both regimes, but the strength of the failure
depends on the chamber type.

\emph{Bounded chambers (exponentially unlikely).} Fix branch count~$M$ and
pretend that on each chamber~$\fineC$ the event ``branch~$i$ is active''
occurs independently with probability $p\in(0,1)$. Then
\[
  \Pr(|\Ichamber{\mapf}{\fineC}|=1)
  = M\,p\,(1-p)^{M-1}
  \;\xrightarrow{M\to\infty}\; 0,
\]
\emph{provided} $p$ stays bounded away from~$0$---equivalently, the average
branch coverage grows with~$M$, as when bounded branch images tile a fixed
region ever more densely (a covering/packing argument). Under that condition
1-active chambers, hence weak injectivity, become exponentially rare.

\emph{Unbounded chambers (non-generic, but not exponentially rare).} On
unbounded tails the coverage need not grow, and the failure is milder. In a
1-D continuous PWA map the two unbounded branches have slopes $a_L,a_R$: if
$\mathrm{sign}(a_L)=\mathrm{sign}(a_R)$ their images run to opposite infinities
(1-active tails), otherwise to the same infinity (2-active tails). Under a
symmetric random-sign model, tail weak injectivity already fails with
probability~$\approx\tfrac12$---a \emph{positive-measure} event, not a
measure-zero coincidence. This unbounded case is the most favourable one for
weak injectivity, and even here it is non-generic.

\medskip
The upshot: weak injectivity demands a \emph{global image-exclusivity}
event---one branch dominating an open region unopposed---which is fragile, and
exponentially so on bounded chambers. \condref{SB}/\condref{USB} demand only a
\emph{local combinatorial asymmetry} at a single boundary, which is robust
and, once discontinuities are allowed, generic.

\section{Injectivity implies (USB) and chamber-wise (AD)}
\label{app:inj-implications}

\begin{proposition}
\label{prop:inj-usb}
If $\mapf\colon\R^d\to\R^d$ is a globally injective reduced PWA map, then $\mapf$ satisfies \condref{USB}.
\end{proposition}

\begin{proof}
Under global injectivity, for any $x$ in the interior of $\mathrm{Im}(\mapf)$, there is exactly one $z$ with $\mapf(z)=x$. Hence (ignoring measure-zero boundary bookkeeping) every nonempty chamber has exactly one active branch: $|\Ichamber{\mapf}{\fineC}|\in\{0,1\}$.

Given any branch $i$, pick any chamber $\fineC\subset\mapf(\branchf{i})^\circ$. If $\mapf(\branchf{i})^\circ = \R^d$ (the single-chamber case), clause~(a) of \condref{USB} applies. Otherwise, $\partial\mapf(\branchf{i})\neq\varnothing$, so there is an adjacent chamber $\fineC'$ across \NEW{a boundary contained in} $\partial\mapf(\branchf{i})$. Since chambers have constant active set and $\fineC'$ cannot have active set $\{i\}$ (it is across a boundary of branch~$i$'s image), the active set of~$\fineC'$ is either $\varnothing$ or $\{i'\}$ with $i'\neq i$. In both cases, $i$ is the unique leaving branch across $(\fineC,\fineC')$. So \condref{USB} holds.

\textbf{Caveat.} A reduced branch domain $\branchf{i}$ may be a union of disconnected polyhedra, hence $\mapf(\branchf{i})^\circ$ can have multiple connected components (multiple chambers) with the same singleton active set $\{i\}$. This does not break the argument.
\end{proof}

\begin{proposition}
\label{prop:inj-chamber-pi}
If $\mapf$ is globally injective, then $(\mapf,\latentZ)$ satisfies chamber-wise \condref{AD} for any irredundant GMM $\latentZ$.
\end{proposition}

\begin{proof}
On any chamber $\fineC$ with $\Ichamber{\mapf}{\fineC}=\{i\}$, the active pairs are $\{i\}\times[K]$. Since $\mapf_i$ is an invertible affine map and the GMM is irredundant, the map $k\mapsto\paramf{i}{k}$ is injective. Hence chamber-wise \condref{AD} holds automatically.
\end{proof}

As shown in Appendix~\ref{app:global-pi}, injectivity does \emph{not} imply global \condref{AD}.

\section{Observed domain labels and iVAE assumption~(iv) imply (MST)}
\label{app:ivae-mst}

We show that, when the PWA+GMM model is viewed as a multi-domain model with \emph{observed} domain labels, iVAE's sufficient-variability condition (assumption~(iv) of \citet{khemakhem2020variational}) implies our \condref{MST}---or, once labels are observed, its labelled form \condref{JST}.

In iVAE, the latent distribution is a conditionally factorial exponential family with $n=\ambientdim$ latent dimensions and $k$ sufficient statistics per dimension. For Gaussian latents with modulated mean and variance, $k=2$, so the natural parameter vector $\lambda(u)\in\R^{2\ambientdim}$ has $nk=2\ambientdim$ entries. Assumption~(iv) requires $nk+1=2\ambientdim+1$ distinct domain values $u_0,\ldots,u_{2\ambientdim}$ such that the $2\ambientdim\times 2\ambientdim$ matrix $L=(\lambda(u_1)-\lambda(u_0),\ldots,\lambda(u_{2\ambientdim})-\lambda(u_0))$ is invertible. In our GMM framework, each domain value corresponds to a mixture component, so assumption~(iv) requires $K>2\ambientdim$ components. Each component is $\Gauss(\mu(u),\Sigma(u))$ with $\Sigma(u)$ diagonal, and for the sufficient statistic $T(z_i)=(z_i,z_i^2)$ the natural parameter splits as $\lambda(u)=(\eta(u),\gamma(u))\in\R^{2\ambientdim}$ with $\eta(u)=\Sigma(u)^{-1}\mu(u)$ and $\gamma(u)=-\tfrac12\,\mathrm{diag}(\Sigma(u)^{-1})$.

Observed domain labels turn the mixture into a \emph{labelled GMM} in practice: any symmetry $T$ with $T_\sharp p = p$ must preserve each label---because $T$ maps the data from labelled component $k$ to data that must still carry label $k$---so the permutation $\sigma$ induced by $T$ must be the identity. Consequently, admissible symmetries are precisely those that fix every component individually, i.e., $T\in\bigcap_k\SymComp{k}$. This reduces \condref{MST} to \condref{JST}.

\begin{proposition}
\label{prop:ivae-mst}
Suppose the GMM has $K > 2\ambientdim$ components with diagonal covariances, the domain labels are observed, and iVAE assumption~(iv) holds. Then the GMM satisfies \condref{MST}\NEW{ (precisely, its label-preserving form \condref{JST}, which coincides with \condref{MST} once labels are observed)}\edtag{E34}. \emph{No independent generation of the mean and covariance parameters is required.}
\end{proposition}

\begin{proof}
Since domain labels are observed, any affine self-symmetry fixes every
component individually, reducing \condref{MST} to \condref{JST}. It thus
suffices to show that the $2\ambientdim+1$ components singled out by~(iv)
have trivial common symmetry group,
$\bigcap_{r=0}^{2\ambientdim}\SymComp{u_r}=\{\Id\}$.

Let $T(x)=Ax+b$ fix each of these Gaussians, and write $\mu_r:=\mu(u_r)$,
$\Sigma_r:=\Sigma(u_r)$ (diagonal). Fixing a nondegenerate Gaussian is
equivalent to
\[
  A\Sigma_r A^\top=\Sigma_r,\qquad A\mu_r+b=\mu_r,
\]
so $b=(I-A)\mu_r$ and, inverting the covariance equation,
$A^\top\Sigma_r^{-1}A=\Sigma_r^{-1}$ for every~$r$.

\emph{Step 1: (iv) spans the diagonal precisions.} Split $L$ into
mean/variance blocks $L=\left(\begin{smallmatrix}X\\ Y\end{smallmatrix}\right)$
with $Y\in\R^{\ambientdim\times 2\ambientdim}$ collecting the rows
$\gamma(u_r)-\gamma(u_0)$. Every column of~$L$ lies in
$\R^{\ambientdim}\oplus\mathrm{col}(Y)$, so
$2\ambientdim=\mathrm{rank}(L)\le\ambientdim+\mathrm{rank}(Y)$, forcing
$\mathrm{rank}(Y)=\ambientdim$. Since
$\gamma(u)=-\tfrac12\,\mathrm{diag}(\Sigma(u)^{-1})$, the vectors
$\mathrm{diag}(\Sigma_r^{-1})-\mathrm{diag}(\Sigma_0^{-1})$ span
$\R^{\ambientdim}$; equivalently, $\{\Sigma_r^{-1}\}_{r=0}^{2\ambientdim}$
spans the space of diagonal $\ambientdim\times\ambientdim$ matrices.

\emph{Step 2: $A$ fixes every diagonal matrix.} The linear map $D\mapsto
A^\top D A$ fixes each $\Sigma_r^{-1}$, hence, by Step~1, every diagonal~$D$:
$A^\top D A=D$. Taking $D=E_j$ (the diagonal matrix with a single~$1$ in
entry $(j,j)$), $A^\top E_j A=(A^\top e_j)(A^\top e_j)^\top$ is the rank-one
outer product of the $j$-th row of~$A$ with itself; equating it to $E_j$
forces that row to be~$\pm e_j^\top$. Thus
$A=\mathrm{diag}(\varepsilon_1,\dots,\varepsilon_{\ambientdim})$ with
$\varepsilon_j\in\{\pm1\}$.

\emph{Step 3: (iv) rules out sign flips.} From $b=(I-A)\mu_r$, coordinate~$j$
reads $b_j=(1-\varepsilon_j)(\mu_r)_j$ for all~$r$. If $\varepsilon_j=-1$
then $(\mu_r)_j\equiv c$ is constant in~$r$, whence
$\eta(u_r)_j=(\Sigma_r^{-1})_{jj}\,(\mu_r)_j=-2c\,\gamma(u_r)_j$; subtracting
the $r=0$ relation makes row~$j$ of~$L$ equal to $-2c$ times
row~$\ambientdim+j$, contradicting $\mathrm{rank}(L)=2\ambientdim$. Hence
every $\varepsilon_j=+1$, so $A=I$ and $b=(I-A)\mu_0=0$, i.e.\ $T=\Id$.
Therefore $\bigcap_{r}\SymComp{u_r}=\{\Id\}$, giving \condref{JST} and, with
observed labels, \condref{MST}.
\end{proof}

\begin{remark}
Under iVAE assumption~(iv) with $K>2\ambientdim$ components, one can also extract $\ambientdim+1$ components with affinely independent diagonal \emph{precisions} (Proposition~\ref{prop:aff-indep-diag-ica}), yielding the ICA-form ambiguity---with no independent generation of parameters. Indeed the block-rank argument in the proof of Proposition~\ref{prop:ivae-mst} shows that the $2\ambientdim$ natural-parameter differences of~(iv) already force the variance block, i.e.\ the diagonal precisions, to span $\R^{\ambientdim}$.
\end{remark}

\section{Finite-order affine symmetries of GMMs}
\label{app:torsion}

\begin{proposition}
\label{prop:torsion}
Let $p=\sum_{k=1}^K w_k\,\Gauss(\mu_k,\Sigma_k)$ be an irredundant nondegenerate Gaussian mixture on $\R^d$. If $p$ admits a nontrivial affine self-symmetry, then it admits a nontrivial affine self-symmetry of finite order.
\end{proposition}

\begin{proof}
Let $T(x)=Ax+b$ be a nontrivial affine map with $T_\sharp p=p$. Since affine maps send Gaussians to Gaussians and the Gaussian components are irredundant, uniqueness of finite Gaussian-mixture decomposition implies that $T$ permutes the weighted components. Let $\sigma\in S_K$ be the induced permutation, and let $m$ be its order.

Then $T^m$ fixes every component individually. Write $T^m(x)=Cx+d$, $C=A^m$. If $T^m=\Id$, then $T$ itself has finite order, and we are done.

Assume now that $T^m\neq\Id$. Because $T^m$ fixes each component $\Gauss(\mu_k,\Sigma_k)$, we have
$$
C\Sigma_k C^\top = \Sigma_k, \qquad C\mu_k+d=\mu_k \qquad (k=1,\ldots,K).
$$
Subtracting the mean equations for $k$ and $1$ gives $C(\mu_k-\mu_1)=\mu_k-\mu_1$ for all~$k$.

Define
$$
H := \bigl\{L\in\GL(d) : L\Sigma_k L^\top = \Sigma_k,\; L(\mu_k-\mu_1)=\mu_k-\mu_1 \;\forall k\bigr\}.
$$
This is a closed subgroup of $O(\Sigma_1):=\{L\in\GL(d):L\Sigma_1 L^\top=\Sigma_1\}$, hence $H$ is a compact Lie group. Since $C\in H$ and $T^m\neq\Id$, we must have $C\neq I$, so $H$ is nontrivial.

Every nontrivial compact Lie group contains a nontrivial finite-order element. Choose $L\in H$, $L\neq I$, of finite order. Then the affine map $S(x):=L(x-\mu_1)+\mu_1$ has finite order and fixes every Gaussian component:
$$
S_\sharp\Gauss(\mu_k,\Sigma_k) = \Gauss(\mu_k,\Sigma_k) \qquad (k=1,\ldots,K).
$$
Hence $S_\sharp p = p$, and $S$ is a nontrivial finite-order affine self-symmetry of~$p$.
\end{proof}

\begin{remark}
\label{rem:localise-torsion}
This is the group-theoretic reason the \condref{MST} assumption is structurally necessary for Gaussian-mixture latent laws in broad piecewise decoder classes: once a nontrivial affine symmetry~$T$ exists, a finite-order one~$S$ exists, and $S$ can be localised into a two-piece strict PWA self-map $\{S,\Id\}$ (with appropriately chosen branch domains) that preserves the mixture law while genuinely modifying the decoder (detail and proof omitted for scope; see also Remark~\ref{rem:pi-essential}). 
The next section shows analogous
torsion-based necessity phenomena for broad classes of
exponential-family and location-scale mixtures.
\end{remark}

\section{Broad torsion and symmetry-triviality mechanisms}
\label{app:ef-symmetry}

The present paper uses only the Gaussian case. For orientation, Table~\ref{tab:torsion-beyond-gmm} records several broader mechanisms that also arise for exponential-family and location-scale mixtures. Only the first row is proved here (Proposition~\ref{prop:torsion}); the remaining rows summarise complementary omitted results and directions of ongoing work.

\begin{table}[h]
\caption{Broad torsion and symmetry-triviality mechanisms.}
\label{tab:torsion-beyond-gmm}
\centering\small
\begin{tabular}{@{}p{3.2cm} p{4.2cm} p{2.5cm} p{2.8cm}@{}}
\toprule
\textbf{Setting} & \textbf{What happens} & \textbf{Mechanism} & \textbf{Examples} \\
\midrule
Irredundant Gaussian mixtures & Nontrivial symmetry $\Rightarrow$ finite-order element & compact stabilizer & proved here \\
Nondegenerate location-scale mixtures & likewise, finite order somewhere & compact linear stabilizer & Gaussian, Laplace, Student-$t$ \\
Compact-group-preserving families & torsion already present & compact Lie group contains finite-order elements & Bingham-type \\
Countable-support families & involution already present & transposition inside nontrivial fiber & discrete EF \\
One-parameter dilation families & no nontrivial realised symmetry & $c^\ell=1$ with $c>0 \Rightarrow c=1$ & exponential, fixed-shape gamma \\
Trivial realised symmetry group & MST holds identically & no symmetry to begin with & Poisson \\
\bottomrule
\end{tabular}
\end{table}

\begin{remark}
The broad message is that, beyond the Gaussian case, two phenomena recur: (1)~\emph{torsion-guaranteed settings}, where any nontrivial symmetry automatically produces a finite-order obstruction; (2)~\emph{symmetry-trivial settings}, where irredundant mixtures have no nontrivial symmetry at all. Both support the same conclusion: \condref{MST} is not an artificial assumption, but the natural boundary between identifiable and non-identifiable regimes.
\end{remark}


\section{Alternative map identifiability theorems}
\label{app:alt-mid}

The main-text MID theorem (Lemma~\ref{lem:mid-shared} and
Theorem~\ref{thm:mid}) uses the boundary-richness condition \condref{USB}
to establish a branch bijection, and then \condref{MST} to collapse the
branch-wise ambiguities. We now present two alternative routes that
replace \condref{USB} with purely algebraic conditions.

\subsection{MID via (DW), (OD), and (JST)}
\label{app:mid-dw-od-jst}

\begin{theorem}[MID via DW/OD/JST]
\label{thm:mid-dw-od-jst}
Let $\mapf,\mapg$ be PWA maps and $\latentZ,\latentZest$ both follow
the same GMM $p$ with $K$ components.
Assume~\eqref{eq:distributional-equality}, \condref{AD} on both sides,
and on the truth side: \condref{OD} on~$(\mapf,\latentZ)$, \condref{DW}
and \condref{JST} on~$p$. Then $\mapf = \mapg$.
\end{theorem}

\begin{proof}
\textbf{Step~1.} Since $\latentZ$ and~$\latentZest$ share law~$p$ with
distinct weights (\condref{DW}), the permutation aligning their
components is unique; fix matched indexing.

\textbf{Step~2.} Coefficient matching requires $\wtrue{k}=\west{\ell}$;
\condref{DW} forces $k=\ell$. Hence for each $(j,k)$ active
on~$\fineC$, there is a unique $i_k(j)$ with
$\paramEst{j}{k}=\paramf{i_k(j)}{k}$.

\textbf{Step~3.} We show $i_k(j)$ is independent of~$k$. Suppose
$k\neq\ell$ give $i'=i_k(j)\neq i''=i_\ell(j)$. Then
$\mapg_j\in\orbit{\mapf}{i'}{k}\cap\orbit{\mapf}{i''}{\ell}$, which
contradicts \condref{OD}. Hence $i_k(j)=:\pi(j)$ is constant in~$k$.
Injectivity of~$\pi$ follows from \condref{AD}; since $|\If|=|\Ig|$,
$\pi$ is a bijection.

\textbf{Step~4.} Fix $j$, set $i=\pi(j)$. Then
$S_j:=\mapf_i^{-1}\circ\mapg_j$ satisfies $S_j\in\bigcap_k\SymComp{k}$.
By \condref{JST}, $S_j=\Id$, so $\mapf_i=\mapg_j$. Active-set
agreement gives $\mapf=\mapg$.
\end{proof}

\begin{remark}[Role of each condition]
\label{rem:dw-od-jst-roles}
\begin{itemize}
\item \condref{DW} ensures component-preserving matching (Step~2),
  enabling the orbit argument. Without it, $\Psi$ can swap components
  across branches, and \condref{OD} becomes inapplicable.
\item \condref{OD} forces the branch-wise component match to be
  consistent across components (Step~3). Without it, different
  components could match to different truth branches.
\item \condref{JST} kills the residual stabiliser ambiguity (Step~4).
  Under~\condref{DW}, this is equivalent to \condref{MST}
  (Proposition~\ref{prop:mst-jst}).
\end{itemize}
\end{remark}

\subsection{Transporter-orbit disjointness}
\label{app:TOD}

For Gaussians $\Gauss_k,\Gauss_\ell$ of a GMM, define
$\Transporter{k}{\ell} := \bigl\{T\in\Aff(\R^{\ambientdim}) :
  \pushfwd{T}{\Gauss_k}=\Gauss_\ell\bigr\}$.
This is the set of all affine maps
that transport component~$k$
exactly onto component~$\ell$.
It is a coset of the stabiliser:
$\Transporter{k}{\ell}= T_{k\to\ell}\circ\SymComp{k}$ for any fixed
$T_{k\to\ell}\in\Transporter{k}{\ell}$.
Consequently,
$\dim\Transporter{k}{\ell}=D$.
When $k=\ell$,
the transporter set reduces to
the stabiliser:
$\Transporter{k}{k}=\SymComp{k}$.

The \emph{transporter-orbit} of branch~$i$ under the pair $(k,\ell)$ is
\[
  \mapgeneric_i\circ\Transporter{k}{\ell}
  = \bigl\{\mapgeneric_i\circ T : T\in\Transporter{k}{\ell}\bigr\}.
\]
This is the set of all affine maps~$S$ such that branch~$i$ followed
by~$S^{-1}$ transports $\Gauss_k$ to~$\Gauss_\ell$.

\textbf{Condition \condref{TOD} (\TODname).}
\condlabel{TOD}
A pair $(\mapgeneric,\GMMgeneric)$ satisfies \condref{TOD} if for all
distinct branches $i\neq i'$ and all component quadruples
$(k,\ell,k',\ell')$,
$\bigl(\mapgeneric_i\circ\Transporter{k}{\ell}\bigr)
  \cap\bigl(\mapgeneric_{i'}\circ\Transporter{k'}{\ell'}\bigr)=\varnothing$.

\paragraph{Relationship to other conditions.}
Setting $k=\ell$ and $k'=\ell'$
reduces
$\Transporter{k}{k}=\SymComp{k}$,
so \condref{TOD} implies \condref{OD}
(further restricted to $k\neq k'$)
and \SameCompParamInj{}
(further restricted to $k=k'$).
\condref{TOD} is thus strictly stronger
than both.

\paragraph{Genericity.}
A \condref{TOD} violation requires the difference
map $\mapgeneric_{i'}^{-1}\circ\mapgeneric_i$ to lie in
$\Transporter{k'}{\ell'}^{-1}\circ\Transporter{k}{\ell}$, a set of
dimension at most~$2D$
(two independent orthogonal degrees of
freedom, one from each covariance).
The full affine group has
dimension~$\ambientdim^2+\ambientdim$,
so the bad set has codimension at
least~$2\ambientdim$---the same
bound as for~\condref{OD}.

\subsection{MID via (TOD) and (MST)}
\label{app:mid-tod-mst}

\begin{theorem}[MID via TOD/MST]
\label{thm:mid-tod-mst}
Let $\mapf,\mapg$ be PWA maps and $\latentZ,\latentZest$ both follow
the same GMM~$p$ with $K$ components.
Assume~\eqref{eq:distributional-equality}, \condref{AD} on both sides,
and on the truth side: \condref{TOD} on~$(\mapf,\latentZ)$ and
\condref{MST} on~$p$. Then $\mapf=\mapg$.
\end{theorem}

\begin{proof}
\textbf{Step~1: Branch bijection via (TOD).}
The global bijection
$\Psi(j,k) = (i(j,k),\,\ell(j,k))$
from Lemma~\ref{lem:global-param-matching}
satisfies
$\paramEst{j}{k}
  = \paramf{i(j,k)}{\ell(j,k)}$
and
$w_k = w_{\ell(j,k)}$.
This means
$\mapg_j^{-1}\circ\mapf_{i(j,k)}$
transports $\Gauss_k$
to~$\Gauss_{\ell(j,k)}$, i.e.,
\[
  \mapg_j
  \;\in\;
  \mapf_{i(j,k)}\circ
    \Transporter{k}{\ell(j,k)}.
\]

Fix any estimator branch~$j$.
Suppose there exist $k_1\neq k_2$
with $i_1:=i(j,k_1)\neq i_2:=i(j,k_2)$.
Then
\[
  \mapg_j
  \;\in\;
  \bigl(\mapf_{i_1}\circ
    \Transporter{k_1}{\ell_1}\bigr)
  \;\cap\;
  \bigl(\mapf_{i_2}\circ
    \Transporter{k_2}{\ell_2}\bigr),
\]
where $\ell_s:=\ell(j,k_s)$.
Since $i_1\neq i_2$, this contradicts
\condref{TOD} on~$(\mapf,\latentZ)$.
Hence $i(j,k)=:\pi(j)$ is independent
of~$k$.

\medskip
\textbf{Step~2: $\pi$ is a bijection.}
If $\pi(j_1)=\pi(j_2)$ for
$j_1\neq j_2$,
then $\Psi$ maps both blocks
$\{j_1\}\times[K]$ and
$\{j_2\}\times[K]$
into $\{\pi(j_1)\}\times[K]$,
requiring $2K\le K$: impossible.
Since $|\If|=|\Ig|$, $\pi$ is a bijection.

\medskip
\textbf{Step~3: Kill branch differences
via (MST).}
For each~$j$, define
$S_j := \mapf_{\pi(j)}^{-1}\circ\mapg_j$.
The matching gives
$\pushfwd{S_j}{\Gauss_k}
  = \Gauss_{\ell(j,k)}$
for every~$k$,
so $S_j$ induces a bijection
$\tau_j\colon k\mapsto\ell(j,k)$
on components.
The weight equality
$w_k = w_{\tau_j(k)}$
from the coefficient matching
ensures this permutation
preserves weights.
Hence
\begin{align*}
  \pushfwd{S_j}{p}
  &= \sum_k w_k\,\Gauss_{\tau_j(k)} \\
  &= \sum_m w_{\tau_j^{-1}(m)}\,\Gauss_m
  = \sum_m w_m\,\Gauss_m = p,
\end{align*}
so $S_j\in\Gmix{p}$.
By \condref{MST}, $S_j=\Id$; hence
$\mapf_{\pi(j)}=\mapg_j$.
Active-set agreement gives $\mapf=\mapg$.
\end{proof}

\subsection{Comparison of approaches}

\begin{table}[h]
\caption{Three alternative routes to map identifiability (shared
  latent setting). All assume (AD) on both sides.}
\label{tab:alt-mid}
\centering\small
\begin{tabular}{@{}lll@{}}
\toprule
\textbf{Route} & \textbf{Branch bij.} & \textbf{Kill} \\
\midrule
USB + MST & Boundary richness & MST \\
DW + OD + JST & DW (comp-pres.) + OD & JST (= MST) \\
TOD + MST & TOD & MST \\
\bottomrule
\end{tabular}
\end{table}

In all three cases, once the branch bijection is established, the
residual branch-wise ambiguity $S_j$ lies in~$\Gmix{p}$, and
\condref{MST} (or its equivalent under \condref{DW}) forces $S_j=\Id$.
This ``two-step'' structure---first a branch bijection, then a
symmetry collapse---is a recurring pattern across all our
identifiability results.

The three routes reflect different trade-offs:
\begin{itemize}
\item \textbf{USB + MST} (Lemma~\ref{lem:mid-shared}) uses geometric
  boundary structure to isolate each branch, and a single symmetry
  condition to kill the ambiguity. It requires no algebraic conditions
  on the map beyond \condref{AD}.
\item \textbf{DW + OD + JST} (Theorem~\ref{thm:mid-dw-od-jst}) uses
  weight distinctness as combinatorial tags and orbit geometry to
  separate branches. It requires no boundary conditions. Under DW,
  \condref{JST} and \condref{MST} coincide
  (Proposition~\ref{prop:mst-jst}), so the ``killing'' mechanism is
  equivalent.
\item \textbf{TOD + MST} (Theorem~\ref{thm:mid-tod-mst}) operates
  without boundary conditions or weight distinctness, using the
  generalised transporter-orbit structure to handle arbitrary
  component-branch entanglement. The cost is a stronger orbit condition
  (\condref{TOD} vs.\ \condref{OD}).
\end{itemize}

\section{\NEW{Stability of map identifiability}}
\label{app:stability}
\begin{newedit}
\edtag{E27}The conditions of Theorem~\ref{thm:mid} are open: two atoms can drift together, an exposing boundary can shrink to nothing, a mixture can become symmetric in the limit. With unboundedly many regions no stability statement holds across all PWA decoders. So we fix a complexity budget~$\kappa$ and ask each condition to hold with a margin $c>0$.

\textbf{Notation.} $W_1$ is the Wasserstein-1 distance, $\mathrm{vol}$ is Lebesgue measure, $\sigma(A)$ is the set of singular values of~$A$, $\preceq$ is the positive-semidefinite order, and $B(x,r)$ is the open ball of centre~$x$ and radius~$r$. Two windows of radius~$R$ are used: $B_R:=B(0,R)$ in latent space, where $\|h\|_{L^1(B_R)}:=\int_{B_R}\|h\|\,d\mathrm{vol}$ and the volume margin are taken, and $\Omega_R:=B(0,R)$ in \NEWB{observational} space, which holds the clean balls. A \emph{wall} is one of the hyperplanes $\{z:\langle\nu_a,z\rangle=t_a\}$ that cut a PWA map in latent space. The signs of $\langle\nu_a,z\rangle-t_a$ decide which branch domain contains~$z$.

\textbf{Normalization.} We state the result for whitened pairs: $p_{\latentZ}$ has mean~$0$ and covariance~$I$. Whitening changes neither the observed law nor \MST{} (Proposition~\ref{prop:mst-conjugacy}). It reduces the affine ambiguity of Theorem~\ref{thm:mid} to a compact gauge: an affine map carrying one whitened law onto another has zero translation and orthogonal linear part, so it lies in~$O(d)$.

\textbf{The class.} Fix $\kappa=(N,m,K,M,\pi_0,\lambda_-,\lambda_+,c,\rho,R)$. $\mathcal{F}_\kappa$ is the set of whitened pairs $(\mapf,\latentZ)$ such that (1)~$p_{\latentZ}$ is an irredundant $K$-component GMM with weights $w_\ell\ge\pi_0$, means $\|\mu_\ell\|\le M$ and covariances $\lambda_-I\preceq\Sigma_\ell\preceq\lambda_+I$; (2)~$\mapf$ is a reduced PWA map with $m$ branches $F_j(z)=A_jz+b_j$ on domains~$P_j$, cut by at most $N$ walls with unit normals~$\nu_a$ and offsets $|t_a|\le M$, with $\sigma(A_j)\subset[1/M,M]$ and $\|b_j\|\le M$; (3)~the four margins below hold.

\emph{\AD{} with margin~$c$.} Atoms have parameters at distance at least~$c$: for $(j,\ell)\neq(j',\ell')$,
\[
\|(A_j\mu_\ell+b_j)-(A_{j'}\mu_{\ell'}+b_{j'})\|+\|A_j\Sigma_\ell A_j^{\top}-A_{j'}\Sigma_{\ell'}A_{j'}^{\top}\|_F\ \ge\ c.
\]

\emph{\USB{} with margin~$c$.} Fix $\vartheta\in(0,1/2)$ and $\gamma>1$. A ball $B(x,r)\subseteq\Omega_R$ is a \emph{clean ball} for branch~$j$ if (toggle) the branch image $F_j(P_j)$ and its complement each fill at least a fraction~$\vartheta$ of the volume of $B(x,r)$, and (isolation) every other branch image covers $B(x,\gamma r)$ or misses it, up to null sets. Let $\beta_j$ be the supremum of the radii of clean balls for~$j$ ($0$ if there are none). The margin is $\beta_j\ge c$ for every~$j$. Toggling keeps the ball from being full, empty or merely tangent, so branch~$j$ switches inside it, and isolation says that no other branch does (Lemma~\ref{lem:stab-witness}). Near a boundary both volume fractions tend to~$1/2$ as $r\to0$, which supplies clean balls and caps~$\vartheta$ at~$1/2$. Any such $\vartheta$ and~$\gamma$ give the same proof.

\emph{\MST{} with margin~$c$.} $W_1(\pushfwd{O}{p_{\latentZ}},p_{\latentZ})\ge c$ for every $O\in O(d)$ with $\|O-I\|_{\mathrm{op}}\ge\sqrt2$. By the normalization, $O(d)$ is the whole gauge. The cutoff is needed, since $W_1(\pushfwd{O}{p_{\latentZ}},p_{\latentZ})\to0$ as $O\to I$. It suffices, since a nontrivial symmetry group of~$p_{\latentZ}$ would contain some $O$ with $\|O-I\|_{\mathrm{op}}\ge\sqrt2$ (Lemma~\ref{lem:stab-sqrt2}), so the margin implies \MST{}, against discrete and continuous symmetries alike.

\emph{Volume margin.} $\mathrm{vol}(P_j\cap B_R)\ge\rho$ for every branch~$j$.

Proposition~\ref{prop:stability} is stated for this class. Its $\omega$ is a modulus: a nondecreasing function with $\omega(\varepsilon)\to0$ as $\varepsilon\to0$, depending only on~$\kappa$, $\vartheta$ and~$\gamma$, whose existence is asserted without a rate.

\textbf{The class in coordinates.} Fix a map~$\Lambda$ from the $2^N$ sign patterns of the walls to branch indices, padding with repeated walls so that at most $N$ walls become exactly~$N$. This splits $\mathcal{F}_\kappa$ into finitely many pieces and adds no hypothesis. The two pairs compared may have different~$\Lambda$. With $\Lambda$ fixed, a pair is a coefficient vector~$\theta$ (wall normals and offsets, branch data, GMM data) in a compact product of spheres, closed intervals and closed bounded matrix sets. Write $\theta=(\mapf,\latentZ)$ and $\Phi(\theta):=\pushfwd{\mapf}{p_{\latentZ}}$.

\begin{lemma}[Continuity of the parametrization]
\label{lem:stab-cont}
\edtag{E28}Let $\theta_n\to\theta_\infty$ with $\Lambda$ fixed. Then (i)~$W_1(\Phi(\theta_n),\Phi(\theta_\infty))\to0$, (ii)~$\|\mapf_n-\mapf_\infty\|_{L^1(B_R)}\to0$, and (iii)~$W_1(p_{\latentZ_n},p_{\latentZ_\infty})\to0$.
\end{lemma}
\begin{proof}
\edtag{E41}Converging walls $(\nu_a,t_a)$ change half-space volume by $o(1)$ on each ball, and finite Boolean operations preserve symmetric-difference convergence, so $\mathrm{vol}((P_j^{(n)}\triangle P_j^{(\infty)})\cap B)\to0$ for every ball~$B$, where $P_j^{(n)}$ is the domain of branch~$j$ at~$\theta_n$. For~(ii), the affine rules converge uniformly on~$B_R$, and the branch-disagreement set has vanishing measure and bounded integrand. For~(iii), write $\mathcal{N}_{\ell,n}:=\Gauss(\mu_\ell^{(n)},\Sigma_\ell^{(n)})$. A common standard normal gives $W_1(\mathcal{N}_{\ell,n},\mathcal{N}_{\ell,\infty})\to0$. Couple matched component mass by these, and move the vanishing weight mismatch at a $W_1$-cost bounded by the box. For~(i), we use the Kantorovich--Rubinstein duality and
\[
\int u\,d\Phi(\theta_n)=\sum_{j,\ell}w_\ell^{(n)}\int_{P_j^{(n)}}u\bigl(F_j^{(n)}z\bigr)\,d\mathcal{N}_{\ell,n}(z)
\]
to reduce the difference to weight, rule, domain and Gaussian-law terms. The first two vanish by coefficient convergence and uniform first moments. Given $\delta>0$, take $S$ such that the first-moment tails of the Gaussians in the box off~$B_S:=B(0,S)$ are below~$\delta$ uniformly in~$n$. On~$B_S$ the other two terms vanish by the symmetric-difference convergence above, uniformly bounded Gaussian densities (from $\lambda_-I\preceq\Sigma_\ell$) and Gaussian total-variation convergence. The bounds are uniform over 1-Lipschitz~$u$ with $u(0)=0$, and letting $\delta\downarrow0$ proves~(i).
\end{proof}

\begin{lemma}[The \USB{} margin is closed]
\label{lem:stab-usb}
$\theta\mapsto\min_j\beta_j(\theta)$ is upper semicontinuous, so $\{\theta:\min_j\beta_j(\theta)\ge c\}$ is closed.
\end{lemma}
\begin{proof}
Let $\theta_n\to\theta_\infty$ with $\min_j\beta_j(\theta_n)\ge c$, fix~$j$, and let $\chi_k^{(n)}$ be the indicator of the branch image $F_k(P_k)$ at~$\theta_n$.
(1)~\emph{A converging clean ball.} The supremum defining~$\beta_j$ need not be attained, so pick at each~$n$ a clean ball $B(x_n,s_n)\subseteq\Omega_R$ with $s_n>c-1/n$, and for each $k\neq j$ a value $v_{k,n}\in\{0,1\}$ with $\mathrm{vol}(\{\chi_k^{(n)}=v_{k,n}\}\cap B(x_n,\gamma s_n))=0$. On a subsequence, $x_n\to x_\infty$, $s_n\to s_\infty\ge c$ and $v_{k,n}=v_k$. Since $\|x_n\|+s_n\le R$ passes to the limit, $B(x_\infty,s_\infty)\subseteq\Omega_R$.
(2)~\emph{Volumes pass to the limit.} The walls converge, and $\sigma(A_k)\subseteq[1/M,M]$ makes the inverse branch maps converge. So, off the null set of images of the limiting walls under the limiting branch maps, the signs deciding $\chi_k^{(n)}(x)$ are eventually those of the limit, and $\chi_k^{(n)}\to \chi_k^{(\infty)}$ a.e. Ball indicators converge a.e.\ when centres and radii converge, so by dominated convergence $\mathrm{vol}(\{\chi_k^{(n)}=v\}\cap B(x_n,r_n))\to\mathrm{vol}(\{\chi_k^{(\infty)}=v\}\cap B(x_\infty,r))$ whenever $r_n\to r$, for $v\in\{0,1\}$.
(3)~\emph{The limit ball is clean.} With $r_n=s_n$, both toggle inequalities pass to the limit. With $r_n=\gamma s_n$, each isolation equality stays~$0$. So $\beta_j(\theta_\infty)\ge s_\infty\ge c$.
\end{proof}

\begin{lemma}[A clean ball gives \USB{}]
\label{lem:stab-witness}
If branch~$j$ of~$\mapf$ has a clean ball, then there are adjacent chambers of~$\mapf$ across which~$j$ is the unique entering or leaving branch. In particular, $\beta_j>0$ for every~$j$ implies \USB{}.
\end{lemma}
\begin{proof}
Each branch image is a finite union of full-dimensional polyhedra, so its essential boundary lies in finitely many hyperplanes. Let $B(x,r)$ be a clean ball for~$j$. By toggle, $\{\chi_j=1\}$ and $\{\chi_j=0\}$ both have positive measure in the ball, which is connected (indecomposable), so the essential boundary of $F_j(P_j)$ contains a relatively open $(d-1)$-dimensional piece of one of those hyperplanes inside the ball, with full-dimensional regions on both sides. By isolation, each $\chi_k$, $k\neq j$, is a.e.\ constant on $B(x,\gamma r)$. Chambers are open\NEW{, hence of positive measure,} and carry a constant~$\chi_k$, so every chamber meeting the ball carries that value. The two chambers across that piece therefore differ exactly in~$j$\NEW{: they are adjacent, and $j$ is the unique branch entering or leaving. The argument uses chambers, not points, so it does not depend on whether the domains are open, closed or mixed.}\edtag{E43}
\end{proof}

\begin{lemma}[No subgroup in the $\sqrt2$ ball]
\label{lem:stab-sqrt2}
A subgroup $H\le O(d)$ with $\|O-I\|_{\mathrm{op}}<\sqrt2$ for every $O\in H$ is~$\{I\}$.
\end{lemma}
\begin{proof}
\NEW{In real orthogonal normal form, $\|O-I\|_{\mathrm{op}}=\max_k2|\sin(\psi_k/2)|$ over the block angles~$\psi_k$, a $-1$ block counting as $\psi=\pi$. So every angle in the $\sqrt2$ ball lies in $[0,\pi/2)$, and any $O\in H$ other than~$I$ has an angle $\psi\in(0,\pi/2)$. Take $n$ least with $n\psi\ge\pi/2$. Then $(n-1)\psi<\pi/2$, so $n\psi<\pi/2+\psi<\pi$, whence $\|O^n-I\|_{\mathrm{op}}\ge2\sin(n\psi/2)\ge\sqrt2$, contradicting $O^n\in H$.}\edtag{E42}
\end{proof}

\edtag{E29}\textbf{Proof of Proposition~\ref{prop:stability}.} Exact identifiability (Theorem~\ref{thm:mid}, one pair at a time) plus compactness gives one modulus for all pairs.

\emph{Step~1: the class is compact.} Whitening, the \AD{} margin and the volume margin are continuous in~$\theta$ (the volume by the \NEW{symmetric-difference convergence of branch domains} in the proof of Lemma~\ref{lem:stab-cont}). The \MST{} margin is continuous by Berge's maximum theorem, as a minimum over the fixed compact set $\{O:\|O-I\|_{\mathrm{op}}\ge\sqrt2\}$ of a jointly continuous function (Lemma~\ref{lem:stab-cont}(iii)). The \USB{} margin is closed by Lemma~\ref{lem:stab-usb}. So the coefficient set of each piece of~$\mathcal{F}_\kappa$ is closed in a compact set, hence compact.

\emph{Step~2: the modulus.} For pairs $\theta=(\mapf,\latentZ)$ and $\theta'=(\mapg,\latentZest)$ let
\[
\Delta(\theta,\theta'):=\min_{O\in O(d)}\max\bigl(\|\mapf-\mapg\circ O^{-1}\|_{L^1(B_R)},\ W_1(p_{\latentZ},\pushfwd{O}{p_{\latentZest}})\bigr),
\]
and let $\omega(\varepsilon)$ be the supremum of $\Delta(\theta,\theta')$ over $\theta,\theta'\in\mathcal{F}_\kappa$ with $W_1(\Phi(\theta),\Phi(\theta'))\le\varepsilon$. It is nondecreasing, and finite since the coefficient box bounds both quantities at $O=I$. The minimum is attained: $O(d)$ is compact, the coefficients of $\mapg\circ O^{-1}$ and $\pushfwd{O}{p_{\latentZest}}$ depend continuously on~$O$ and stay in the box, and Lemma~\ref{lem:stab-cont}(ii)--(iii) make the objective continuous. So the proposition holds once $\omega(\varepsilon)\to0$.

Suppose instead $\omega(\varepsilon)\to\delta_0>0$. Then there are $\theta_n=(\mapf_n,\latentZ_n)$ and $\theta_n'=(\mapg_n,\latentZest_n)$ in~$\mathcal{F}_\kappa$ with $W_1(\Phi(\theta_n),\Phi(\theta_n'))\le1/n$ and $\Delta(\theta_n,\theta_n')\ge\delta_0/2$.
(1)~\emph{Extract a limit.} There are finitely many~$\Lambda$, so infinitely many~$n$ share one pair of them, and a further subsequence converges, $\theta_n\to\theta_\infty=(\mapf_\infty,\latentZ_\infty)$ and $\theta_n'\to\theta_\infty'=(\mapg_\infty,\latentZest_\infty)$. By Step~1 both limits lie in~$\mathcal{F}_\kappa$. They are still reduced and irredundant, because merged branches or components would give two atoms at parameter distance~$0$, against the \AD{} margin.
(2)~\emph{Exact equality.} By Lemma~\ref{lem:stab-cont}(i) and the triangle inequality,
$W_1(\Phi(\theta_\infty),\Phi(\theta_\infty'))\le W_1(\Phi(\theta_\infty),\Phi(\theta_n))+1/n+W_1(\Phi(\theta_n'),\Phi(\theta_\infty'))\to0$, so the limits satisfy~\eqref{eq:distributional-equality}.
(3)~\emph{Exact identifiability.} The margins give \AD{} on both sides, and \USB{} (Lemma~\ref{lem:stab-witness}) and \MST{} for $(\mapf_\infty,\latentZ_\infty)$. Theorem~\ref{thm:mid} gives an affine $\bar\alpha$ with $\mapf_\infty=\mapg_\infty\circ\bar\alpha^{-1}$ (Definition~\ref{def:pwa-equality}) and $\bar\alpha(\latentZest_\infty)\deq\latentZ_\infty$. Both laws are whitened, so $\bar\alpha=\bar O\in O(d)$.
(4)~\emph{Travel back.} Hold $\bar O$ fixed. The coefficients of $\mapg_n\circ\bar O^{-1}$ and of $\pushfwd{\bar O}{p_{\latentZest_n}}$ converge to those of $\mapg_\infty\circ\bar O^{-1}$ and $\pushfwd{\bar O}{p_{\latentZest_\infty}}$. By Lemma~\ref{lem:stab-cont}(ii)--(iii) and~(3), $\|\mapf_n-\mapg_n\circ\bar O^{-1}\|_{L^1(B_R)}\to0$ and $W_1(p_{\latentZ_n},\pushfwd{\bar O}{p_{\latentZest_n}})\to0$. So $\Delta(\theta_n,\theta_n')<\delta_0/2$ for large~$n$, a contradiction. \qed

\textbf{What the margins buy.} The proposition changes no theorem. (i)~The margin is the conditioning parameter: our conditions are open, and a margin closes the class. The modulus~$\omega$ is qualitative (compactness gives existence, not a rate), and we claim nothing at margin zero. (ii)~On the family of App.~\ref{app:experiments}, the margin is a sample complexity. (iii)~The margins are computable on a fitted model, returning a number rather than a verdict (\S\ref{subsec:stability}).
\end{newedit}

\section{\NEW{Numerical illustration: full setting}}
\label{app:experiments}
\begin{newedit}
\edtag{E30}The family has $\ambientdim=1$ and its only moving part is the \MST{} margin. Let $\latentZ_\eta\sim(\tfrac12+\eta)\Gauss(-1,1)+(\tfrac12-\eta)\Gauss(1,1)$, let $\mapf$ have branch slopes $(1,3,2)$ and zero intercepts on $(-\infty,-1]$, $[-1,1]$, $[1,\infty)$, and let the competitor be $\maph=\mapf\circ S$, with $S(x)=-x$ off $(-1,1)$ and $S(x)=x$ inside. Measured before whitening, the \AD{}, \USB{} and volume margins do not depend on~$\eta$, and the \MST{} margin is $W_1(\pushfwd{(-1)}{p_{\latentZ_\eta}},p_{\latentZ_\eta})=4\eta$. After the whitening of App.~\ref{app:stability} it is $0.110$ at $\eta=0.16$ and $0.0038$ at $\eta=0.005$, also linear in~$\eta$ to leading order. At $\eta=0$ the two models induce the same law, yet $\maph\neq\mapf\circ\alpha$ for every affine~$\alpha$. Let $\mathrm{KL}$ be the Kullback--Leibler divergence between the two observed laws, and $n^\ast(\eta):=\chi^2_{1,0.95}/(2\,\mathrm{KL})$ the sample size at which the expected likelihood-ratio statistic $2n\,\mathrm{KL}$ reaches the 5\% critical value of~$\chi^2_1$. This Wilks calibration is heuristic for two simple models: the exponent below is exact, the constant is not. At~$n^\ast$, the likelihood-ratio test picks~$\mapf$ over~$\maph$ in about 84\% of runs (asymptotically $\Phi(\sqrt{\chi^2_{1,0.95}/4})=0.8365$; measured in the last row, 400 replicates each, binomial standard error at most $0.020$).
\begin{center}
\small
\begin{tabular}{@{}lcccc@{}}
\toprule
$\eta$ & 0.160 & 0.040 & 0.010 & 0.005 \\
\midrule
\NEW{Total variation (TV)} of the observed laws & 0.1527 & 0.0382 & 0.0095 & 0.0048 \\
$n^\ast(\eta)$ & 29 & 460 & 7\,313 & 29\,211 \\
Selection rate at $n^\ast$ & 0.810 & 0.850 & 0.840 & 0.845 \\
\bottomrule
\end{tabular}
\end{center}
Here $\mathrm{TV}=(2\Phi(2)-1)\eta=0.9545\eta$ exactly and $\mathrm{KL}\propto\eta^{2}$ to leading order, so on this family $n^\ast\propto\eta^{-2}$, the inverse square of the \MST{} margin, and power curves collapse when plotted against $n/n^\ast(\eta)$. Since $n^\ast$ is a population quantity, we compute it exactly rather than read it off a training curve. The supplementary material contains the code (\texttt{common.py}, \texttt{exp\_oracle.py}; NumPy only; \texttt{python exp\_oracle.py}, seed~0).
\end{newedit}

\section{Additional Discussion}
\label{app:discussion}


\textbf{Boundary richness vs.\ algebraic conditions.}
The main-text MID results use \condref{USB} (every branch has an
exposed boundary) to establish the branch bijection, illustrating how
geometric structure can substitute for algebraic conditions.  This is
not the only route: Appendix~\ref{app:alt-mid} presents two
alternative theorems that replace \condref{USB} with purely algebraic
conditions on the map--latent pair---orbit disjointness
(\condref{OD}; Theorem~\ref{thm:mid-dw-od-jst}) and
transporter-orbit disjointness (\condref{TOD};
Theorem~\ref{thm:mid-tod-mst}).  These approaches require weight
conditions or stronger orbit conditions instead of boundary richness,
illustrating the modularity of the framework.

\textbf{Globalness, not canonicity, of~$\alpha$.}
The link~$\alpha$ constructed in Theorem~\ref{thm:lid} comes from a specific toggling pair
$(i_\star,j_\star)$ and the permutation~$\sigma$. Neither the toggling
pair nor~$\sigma$ is claimed canonical: LID says only that \emph{some}
affine link exists, not that a single canonical one is attached to the
truth. What matters for MID is not canonicity but
\emph{globalness}---forcing the same structural relationship to hold
across \emph{every} branch, not just at the boundary that
produced~$\alpha$. The role of \condref{MST} is exactly to
deliver this globalness.
For a fixed truth--estimator pair, MID produces an affine~$\alpha$
with $\alpha(\latentZest)\deq\latentZ$ and
$\mapf=\mapg\circ\alpha^{-1}$. Without \condref{MST}, different
branches could in principle use different affine relabellings of the
mixture components; \condref{MST} forces all branchwise difference
maps $S_i=\mapf_i^{-1}\circ\mapg_{j(i)}$ into a trivial
mixture-symmetry group, so that they all coincide. This is a statement
of \emph{global consistency across branches}, not of a
``universal''~$\alpha$ independent of which estimator one considers.
Different estimators matching the same observable law can in general be
related to the truth by different affine links.

\textbf{Necessity of assumptions.}
For Gaussian mixtures,
Appendix~\ref{app:torsion} proves that if a GMM admits any nontrivial affine self-symmetry, then it admits a finite-order one (i.e., \emph{torsion}). In globally injective piecewise decoder
classes that are closed under precompositions, such torsion symmetries can be localised into law-preserving discontinuous
precompositions, producing observationally identical yet
mechanism-distinct models. Concretely, a torsion~$T$ can give a two-piece construction
$S=\{T,\Id\}$ (details omitted here for scope) that yields a
non-\condref{AD} strict PWA self-map~$S$ of the kind discussed in
Remark~\ref{rem:pi-essential}, which explains why such~$S$ breaks
identifiability.
This is the mechanism behind the necessity of \condref{MST}.
We also have complementary results (see
Table~\ref{tab:torsion-beyond-gmm}), showing analogous
torsion-based necessity phenomena for broad classes of
exponential-family and location-scale mixtures, and further examples
(omitted here) showing that the same obstruction can persist even in
continuous piecewise classes.
Boundary richness is likewise necessary: Appendix~\ref{app:fold} gives
a simple fold example in which latent \condref{MST} and \condref{AD}
both hold, yet \condref{SB} fails and law identifiability breaks.

\textbf{Towards a broader algebraic identifiability theory.}
Gaussian mixtures and piecewise-affine maps are the first concrete
testbed for a wider program. In the present paper, the relevant
ambiguity group is the affine group~$\Aff(\R^{\ambientdim})$, which is
exactly the natural symmetry group of the Gaussian family. More
generally, one can replace affine maps by the natural invertible
transformations that preserve a chosen component family~$F$, and then
study mixtures from~$F$ together with piecewise maps whose branch rules
come from that family symmetry group.
We also have a group-theoretic proof of
\condref{MST} genericity (omitted for scope) that uses weaker
conditions than any sufficient conditions shown in Appendix~\ref{app:suff-cond}
\emph{and} applies to many exponential-family
mixtures.
A key next step is to formulate family-level analogues of \NEW{atom distinctness} and to understand how far the branch class can be enlarged
before hidden symmetries reappear. Our expectation is that analogues of
\condref{AD}, \condref{MST}, and \condref{USB} will again mark the
boundary between identifiability and non-identifiability for a much
wider class of piecewise-parametric latent variable models.




\section{Technical comparisons with related work}
\label{app:tech-comparisons}

This section gathers technical comparisons that are too detailed for the main text.

\textbf{Two roles of (MST): comparison with prior work.}
The condition \condref{MST} plays a role analogous to two distinct
mechanisms in prior identifiability frameworks:
In \textbf{iVAE} \citep{khemakhem2020variational},
auxiliary labels (observed class indices) serve as tags that
allow local identifiability to extend across the latent space. In our
setting, \condref{MST} achieves a similar effect: the mixture
components serve as internal tags, and \condref{MST} ensures their
geometry is rich enough to pin down all branch-wise ambiguities.
In \textbf{Kivva et al.}~\citep{kivva2022identifiability},
MID is obtained by a different mechanism: continuity forces adjacent affine pieces to agree on their common \NEW{boundary}, so their relative map fixes a hyperplane. The remaining local alternatives are shear-type or reflection-type. Global density equality, analytic continuation, and GMM identifiability rule out the shear case, since a nontrivial shear cannot preserve a nondegenerate Gaussian mixture. Injectivity then rules out the reflection/fold case. By contrast, our 
\condref{MST} assumption bypasses this continuity-plus-injectivity argument by directly trivializing the affine self-symmetry group of the latent mixture.
Example~2 (Appendix~\ref{app:examples}), adapted from
\citet[Example~D.2]{kivva2022identifiability}, concretely illustrates how both
\condref{AD} and \condref{MST} fail simultaneously, yielding a
strict-PWA self-map.

\textbf{Comparison with iVAE.}
Under the PWA+GMM model class, the assumptions of iVAE
\citep{khemakhem2020variational} are strictly stronger than our
\condref{MST}, \condref{USB}, and \emph{chamber-wise} \condref{AD}:
injectivity of the map implies \condref{USB} and chamber-wise
\condref{AD} (Appendix~\ref{app:inj-implications}), while observed
domain labels together with iVAE's sufficient-variability condition
(assumption~(iv)) imply \condref{MST}
(Appendix~\ref{app:ivae-mst}).
Concretely: observed labels fix component identities so any admissible symmetry must be component-preserving, reducing \condref{MST} to \condref{JST}; \NEW{the mean/variance variation that~(iv) requires} then trivialises the remaining within-component stabiliser, giving \condref{JST} and hence \condref{MST}.
Since our main results use \emph{global} \condref{AD}, which is
strictly stronger than the chamber-wise version
(Appendix~\ref{app:global-pi}), we cannot say that our full assumption
set is strictly weaker than iVAE's. However, our framework strictly
weakens the \emph{domain-contrast assumption}: \condref{MST} replaces
iVAE's requirement of observed domain labels plus sufficient
variability. \NEW{The weakening is strict. \condref{MST} can hold with as few as $\ambientdim+1$ components (affinely independent means, Lemma~\ref{lem:aff-indep-jst}; with labels observed this is already \condref{MST}, App.~\ref{app:ivae-mst}), whereas~(iv) requires $K>2\ambientdim$ and does not apply below that. \condref{MST} also needs no labels.}\edtag{E40}

\begin{newedit}
\textbf{What our conditions trivialize, and what they do not.}\edtag{E35}
Were the target group itself trivialized, no ambiguity would remain, i.e., strong identifiability in the sense of \citet{xi2023indeterminacy}, which we do \emph{not} claim. Two groups must be kept apart. In Theorem~4.2 of \citet{squires2026unifying} the target group is a free parameter. Ours is the admissible choice $\mathrm{Aff}(\R^{\ambientdim})$, and our contrast conditions (domain, mechanism, interaction; Table~\ref{tab:contrasts}) do not shrink it. \emph{What \MST{} makes trivial} is the affine symmetry group of the latent law itself, $\Gmix{p}$ (Section~\ref{subsec:mid}).

Theorem~4.2 of \citet{squires2026unifying} takes the characterization of the intersection of two transition sets as its \emph{hypothesis}, and for PWA-GMM \emph{that characterization is our whole problem}. We draw the correspondence as framing, not implication. Their valid kernel transitions, in the deterministic case the maps $\mapf^{-1}\circ\mapg$, are what \AD{} with \SB{}/\USB{} constrains, and our conclusion is containment in the target group. Their valid measure transitions reduce to the stabilizer of~$p$, hence to $\Gmix{p}$ on its affine part, when the latent law is shared, as in Lemma~\ref{lem:mid-shared}.

The symmetry-breaking view is shared with \citet{squires2026unifying} and \citet{xi2023indeterminacy} and is much older \citep{hurwicz1950generalization,anderson1956statistical}. \emph{Our priority} is the conjunction of four features: algebraic symmetry breaking as the engine; the latent law's own symmetry group as the object trivialized; a concrete nonlinear class (PWA decoders, GMM latents); and no supervision. Table~\ref{tab:comparison}'s proof-flavour column places us against prior work.
\end{newedit}

\begin{newedit}
\textbf{\USB{} and structural sparsity.}\edtag{E33}
A structural-sparsity \NEW{condition} \citep{zheng2022identifiability} isolates a \emph{source} through the support/intersection pattern of the Jacobian. \USB{} isolates a \emph{branch} through a boundary at which exactly one branch toggles. Both are local structural \NEW{conditions} that pin one factor at a time. The difference is the geometry they live in: differential and support-theoretic for smooth decoders, versus boundary-combinatorial for PWA, where the \NEW{condition} survives discontinuity. This also explains why \USB{} is strictly weaker than global injectivity: \NEW{an exposing boundary} is needed once per branch, not everywhere.
\end{newedit}

\textbf{Boundary richness vs.\ weak injectivity.}
Appendix~\ref{app:sb-usb-winj} establishes that \condref{USB} is strictly weaker than global injectivity and logically incomparable with weak injectivity (in the sense of \citet{kivva2022identifiability}), with concrete examples in both the continuous and discontinuous regimes.


\section*{NeurIPS Paper Checklist}

\begin{enumerate}

\item {\bf Claims}
    \item[] Question: Do the main claims made in the abstract and introduction accurately reflect the paper's contributions and scope?
    \item[] Answer: \answerYes{}
    \item[] Justification: The abstract and introduction state exactly the hierarchy of identifiability results (LID, MID, posterior, pointwise) that are proved as formal theorems in \S\ref{sec:results} and Appendix~\ref{app:proofs}, with conditions precisely matched.

\item {\bf Limitations}
    \item[] Question: Does the paper discuss the limitations of the work performed by the authors?
    \item[] Answer: \answerYes{}
    \item[] Justification: \NEW{The Limitations paragraph of \S\ref{sec:conclusion} states the population-level scope, equal latent and observation dimensions, invertible branch rules, conjectural genericity of \condref{SB}/\condref{USB} under continuity, open necessity of \condref{AD}, the causal reading as an analogy, and that the stability result has no rate; Necessity and prospects lists what remains open.}

\item {\bf Theory assumptions and proofs}
    \item[] Question: For each theoretical result, does the paper provide the full set of assumptions and a complete (and correct) proof?
    \item[] Answer: \answerYes{}
    \item[] Justification: All conditions (\condref{AD}, \condref{SB}, \condref{USB}, \condref{MST}) are formally stated in \S\ref{sec:setup}--\ref{sec:results}; proof sketches appear in the main text and complete proofs in Appendix~\ref{app:proofs} (Proposition~\ref{prop:stability}: Appendix~\ref{app:stability}). All lemmas used in the proofs are stated and proved within the appendix; statements labelled as conjectures are used by no proof.

    \item {\bf Experimental result reproducibility}
    \item[] Question: Does the paper fully disclose all the information needed to reproduce the main experimental results of the paper to the extent that it affects the main claims and/or conclusions of the paper (regardless of whether the code and data are provided or not)?
    \item[] Answer: \NEW{\answerYes{}}
    \item[] Justification: \NEW{The one numerical illustration (\S\ref{subsec:stability}) is specified in full in App.~\ref{app:experiments}: the family, the competitor, the test and the exact formula for the sample size $n^\ast$.}

\item {\bf Open access to data and code}
    \item[] Question: Does the paper provide open access to the data and code, with sufficient instructions to faithfully reproduce the main experimental results, as described in supplemental material?
    \item[] Answer: \NEW{\answerYes{}}
    \item[] Justification: \NEW{The supplementary material contains the code of the numerical illustration (App.~\ref{app:experiments}; NumPy only). No data are used.}

\item {\bf Experimental setting/details}
    \item[] Question: Does the paper specify all the training and test details (e.g., data splits, hyperparameters, how they were chosen, type of optimizer) necessary to understand the results?
    \item[] Answer: \NEW{\answerYes{}}
    \item[] Justification: \NEW{App.~\ref{app:experiments} gives every setting of the numerical illustration; no model is trained.}

\item {\bf Experiment statistical significance}
    \item[] Question: Does the paper report error bars suitably and correctly defined or other appropriate information about the statistical significance of the experiments?
    \item[] Answer: \answerYes{}
    \item[] Justification: The sample size $n^\ast$ is computed exactly. The selection rates in App.~\ref{app:experiments} are Monte Carlo estimates over 400 replicates; their binomial standard error, at most $0.020$, is stated there.

\item {\bf Experiments compute resources}
    \item[] Question: For each experiment, does the paper provide sufficient information on the computer resources (type of compute workers, memory, time of execution) needed to reproduce the experiments?
    \item[] Answer: \NEW{\answerYes{}}
    \item[] Justification: \NEW{The numerical illustration (App.~\ref{app:experiments}) is a small NumPy computation that runs on a laptop CPU.}
    
\item {\bf Code of ethics}
    \item[] Question: Does the research conducted in the paper conform, in every respect, with the NeurIPS Code of Ethics \url{https://neurips.cc/public/EthicsGuidelines}?
    \item[] Answer: \answerYes{}
    \item[] Justification: The paper is foundational mathematical research on identifiability theory; it raises no ethical concerns.

\item {\bf Broader impacts}
    \item[] Question: Does the paper discuss both potential positive societal impacts and negative societal impacts of the work performed?
    \item[] Answer: \answerNA{}
    \item[] Justification: This is foundational theory with no direct path to specific applications or deployments; there is no immediate positive or negative societal impact beyond advancing the mathematical foundations of trustworthy generative modelling.
    
\item {\bf Safeguards}
    \item[] Question: Does the paper describe safeguards that have been put in place for responsible release of data or models that have a high risk for misuse (e.g., pre-trained language models, image generators, or scraped datasets)?
    \item[] Answer: \answerNA{}
    \item[] Justification: No data, models, or other assets with misuse risk are released.

\item {\bf Licenses for existing assets}
    \item[] Question: Are the creators or original owners of assets (e.g., code, data, models), used in the paper, properly credited and are the license and terms of use explicitly mentioned and properly respected?
    \item[] Answer: \answerYes{}
    \item[] Justification: The only external asset is NumPy (BSD-3-Clause license), used by the code of the numerical illustration (Appendix~\ref{app:experiments}); all cited work is referenced via standard academic citation.

\item {\bf New assets}
    \item[] Question: Are new assets introduced in the paper well documented and is the documentation provided alongside the assets?
    \item[] Answer: \answerYes{}
    \item[] Justification: The only new asset is the code of the numerical illustration (\texttt{common.py}, \texttt{exp\_oracle.py}) in the supplementary material, documented in Appendix~\ref{app:experiments} and in the scripts' docstrings (NumPy only; \texttt{python exp\_oracle.py}, seed~0). No datasets or models are released.

\item {\bf Crowdsourcing and research with human subjects}
    \item[] Question: For crowdsourcing experiments and research with human subjects, does the paper include the full text of instructions given to participants and screenshots, if applicable, as well as details about compensation (if any)? 
    \item[] Answer: \answerNA{}
    \item[] Justification: The paper involves no crowdsourcing or human subjects.

\item {\bf Institutional review board (IRB) approvals or equivalent for research with human subjects}
    \item[] Question: Does the paper describe potential risks incurred by study participants, whether such risks were disclosed to the subjects, and whether Institutional Review Board (IRB) approvals (or an equivalent approval/review based on the requirements of your country or institution) were obtained?
    \item[] Answer: \answerNA{}
    \item[] Justification: The paper involves no human subjects research.

\item {\bf Declaration of LLM usage}
    \item[] Question: Does the paper describe the usage of LLMs if it is an important, original, or non-standard component of the core methods in this research? Note that if the LLM is used only for writing, editing, or formatting purposes and does \emph{not} impact the core methodology, scientific rigor, or originality of the research, declaration is not required.
    \item[] Answer: \answerNA{}
    \item[] Justification: LLMs were used for writing, editing, and as a verification aid during proof development.

\end{enumerate}

\end{document}